\documentclass{article}

\usepackage[preprint]{neurips_2026}
\usepackage{amsmath, amssymb, graphicx, url}
\usepackage{booktabs}
\usepackage{longtable}
\usepackage[ruled,linesnumbered]{algorithm2e}
\usepackage{algpseudocode}
\usepackage{amsthm}
\usepackage{dsfont}
\usepackage{xcolor}
\setcitestyle{authoryear,open={(},close={)}}
\newtheorem{definition}{Definition}
\newtheorem{theorem}{Theorem}
\newtheorem{lemma}[theorem]{Lemma}
\newtheorem*{restatelemma}{Lemma}
\usepackage{mathtools}

\usepackage{thmtools}
\usepackage{thm-restate}

\usepackage[utf8]{inputenc} 
\usepackage[T1]{fontenc}    
\usepackage{hyperref}       
\usepackage{url}            
\usepackage{booktabs}       
\usepackage{amsfonts}       
\usepackage{nicefrac}       
\usepackage{microtype}      
\usepackage{xcolor}         

\title{Closing the Gap on the Sample Complexity of 1-Identification}

\author{%
  Zitian Li\\
  Department of Industrial Systems Engineering \& Management\\
  National University of Singapore\\
  Singapore, 117576 \\
  \texttt{lizitian@u.nus.edu} \\
  \And
  Wang Chi Cheung\\
  Department of Industrial Systems Engineering \& Management \\
  National University of Singapore\\
  Singapore, 117576 \\
  \texttt{isecwc@nus.edu.sg} \\
}

\begin{document}

\maketitle

\begin{abstract}
  The 1-identification problem is a fundamental pure-exploration problem in multi-armed bandits. An agent aims to determine whether there exists an arm whose mean reward exceeds a known threshold $\mu_0$, or to output \textsf{None} otherwise. The agent must guarantee correctness with probability at least $1-\delta$, while minimizing the expected number of arm pulls $\mathbb{E}[\tau]$. We study the 1-identification problem and make two main contributions. First, for instances with at least one qualified arm, we derive a new lower bound on $\mathbb{E}[\tau]$ via a novel optimization formulation. Second, we propose a new algorithm and establish upper bounds that match the lower bounds up to polynomial logarithmic factors uniformly over all instances. Our result complements the analysis of $\mathbb{E}\tau$ when there are multiple qualified arms, which is an open problem in the literature.
\end{abstract}

\section{Introduction}

\subsection{Model \& Problem Formulation}
We study the 1-identification problem\footnote{The terminology originates from \cite{katz2020true,pmlr-v267-li25f}. \cite{degenne2019pure} call it Any Low Problem.}, a variant of the best arm identification problem. An instance of 1-identification is specified by the tuple $\nu=([K], \rho= \{\rho_a\}_{a\in [K]}, \mu_0,\delta)$, where $[K] \coloneqq \{1,2,\ldots,K\}$ denotes the set of arms with $K \ge 2$, and $\rho_a$ is the unknown reward distribution associated with arm $a$. Upon pulling arm $a$, the agent observes a random reward $R \sim \rho_a$. We denote by $r_a^{\nu} \coloneqq \mathbb{E}_{R \sim \rho_a}[R]$ the mean reward of arm $a$ under instance $\nu$, and assume that the noise $R - r_a^{\nu}$ is $1$-sub-Gaussian. The threshold $\mu_0 \in \mathbb{R}$ is known to the agent.
The confidence parameter $\delta \in (0,1)$ controls the allowable probability of an incorrect decision. The objective of the agent is to determine whether there exists an arm $a \in [K]$
such that $r_a^{\nu} \ge \mu_0$.
In the affirmative case, the agent must identify (at least) one such arm;
otherwise, it must output the symbol $\textsf{None}$.

The 1-identification problem serves as a framework for many decision-making scenarios. For instance, a clinical trial screening phase where researchers must decide whether any of several new drug candidates meet a specific efficacy threshold to justify the cost of further development. Similarly, considering anomaly detection in power grids, an agent may need to determine if if any substation is experiencing a voltage fluctuation that exceeds safety limits. In both cases, the agent does not necessarily need to rank every option; rather, it must quickly identify at least one "qualified" arm or confidently conclude that none exist.

\textbf{Online model and algorithms.} An algorithm for the 1-identification problem is defined by a triple $(\pi, \tau, \hat{a})$, where: (i) $\pi = \{\pi_t\}_{t \ge 1}$ is a (possibly randomized) sampling rule, (ii) $\tau$ is a stopping time, and (iii) $\hat{a} \in [K] \cup \{\textsf{None}\}$ is a recommendation rule.
When executed on an instance $\nu$, the algorithm generates a random history $\{U_0\} \cup \{(A_t, X_t, U_t)\}_{t=1}^{\tau}$, where $A_t \in [K]$ is the arm selected at time $t$, $X_t \sim \rho_{A_t}$ is the corresponding reward, and $U_t \sim \mathrm{Unif}(0,1)$ is for internal randomization of the algorithm. The initial seed $U_0$ allows for randomized initialization.
At each time $t$, the arm $A_t$ is chosen according to $A_t = \pi_t(H(t-1))$, where $H(t-1) \coloneqq \{U_0\} \cup \{(A_s, X_s, U_s)\}_{s=1}^{t-1}$. The stopping time $\tau$ is defined with respect to the filtration $\{\sigma(H(t))\}_{t \ge 1}$. Upon stopping, the algorithm outputs a recommendation $\hat{a}$ that is measurable with respect to $\sigma(H(\tau))$.
Outputting $\hat{a} \in [K]$ indicates the conclusion that $r_{\hat{a}}^{\nu} \ge \mu_0$,
whereas outputting $\hat{a} = \textsf{None}$ indicates the conclusion $\max_{a\in [K]} r_{a}^{\nu} < \mu_0$.

\textbf{Instance classes.} Following \cite{pmlr-v267-li25f}, an instance $\nu$ is called a \emph{positive instance} if $\max_{a \in [K]} r_a^{\nu} > \mu_0$, and a \emph{negative instance} if $\max_{a \in [K]} r_a^{\nu} < \mu_0$.
For $\Delta > 0$, we define the margin-separated classes $\mathcal{S}^{\mathrm{pos}}_{\Delta}
:=\bigl\{\nu : \max_{a} r_a^{\nu} - \mu_0 \ge \Delta \bigr\}$, $\mathcal{S}^{\mathrm{neg}}_{\Delta}\coloneqq \bigl\{\nu : \mu_0 - \max_{a} r_a^{\nu} \ge \Delta \bigr\}$. Let $\mathcal{S}^{\mathrm{pos}} \coloneqq \bigcup_{\Delta>0} \mathcal{S}^{\mathrm{pos}}_{\Delta}$, $\mathcal{S}^{\mathrm{neg}} \coloneqq \bigcup_{\Delta>0} \mathcal{S}^{\mathrm{neg}}_{\Delta}$. We define the set of correct answers $i^*(\nu)$ by $i^*(\nu) \coloneqq
\begin{cases}
\{a \in [K] : r_a^{\nu} \ge \mu_0\}, & \text{if } \nu \in \mathcal{S}^{\mathrm{pos}}, \\[4pt]
\{\textsf{None}\}, & \text{if } \nu \in \mathcal{S}^{\mathrm{neg}}.
\end{cases}$. In this work, We ignore instances for which $\max_a r_a^{\nu} = \mu_0$.

\textbf{PAC guarantees.} We adopt a PAC framework and focus on algorithms with correctness guarantees.
\begin{definition}[$\delta$-PAC]
An algorithm is \emph{$\delta$-PAC}, if
for any $\nu \in \mathcal{S}^{\mathrm{pos}} \cup \mathcal{S}^{\mathrm{neg}}$, $\Pr_{\nu}\!\left(\tau < \infty,\; \hat{a} \in i^*(\nu)\right) \ge 1 - \delta.$
\end{definition}
As shown in Appendix A.3 in \cite{pmlr-v267-li25f}, a $\delta$-PAC algorithm may still satisfy $\mathbb{E}_{\nu}[\tau] = \infty$ for some $\nu \in \mathcal{S}^{\mathrm{pos}}$. To rule out such pathological behavior, we introduce a stronger notion.
\begin{definition}[Strong $\delta$-PAC]
An algorithm is Strong $\delta$-PAC if it is $\delta$-PAC and,
for all $\Delta, \delta>0$, we have $\sup_{\nu \in \mathcal{S}^{\mathrm{pos}}_{\Delta} \cup\mathcal{S}^{\mathrm{neg}}_{\Delta}}\mathbb{E}_{\nu}[\tau] < \infty$.
\end{definition}

\textbf{Objective.} The goal of the agent is to design a Strong $\delta$-PAC algorithm $(\pi,\tau,\hat{a})$ that minimizes the expected sampling complexity $\mathbb{E}_{\nu}[\tau]$.

\subsection{Literature Review}
\label{sec:lit-review}
Before illustrating our main contribution, we review the existing literature on the 1-identification problem and highlight its current limitations. We begin by unifying notation across prior works in order to facilitate comparison.
Throughout this paper, we consider a reward vector $(\mu_a)_{a=1}^K \in [0, 1]^K$ satisfying that $\mu_1\geq\mu_2\geq\cdots\geq \mu_K$, together with a known threshold $\mu_0 \in [0,1]$. For any $i,j \in [K]\cup\{0\}$, we define the gap $\Delta_{i,j} \coloneqq |\mu_i - \mu_j|$. In the analysis of upper bound(section \ref{sec:upper-bound-analysis}), we only consider instance $\nu$ with $r_a^{\nu}=\mu_a$.\footnote{The agent does not know the monotonic ordering of $\{\mu_a\}_{a=1}^K$, and this default setting is for convenience in discussions.} In the analysis of lower bound(section \ref{sec:Lower-Bound-for-Positive-Case}), we further consider instances $\nu'$ whose reward vector $r_a^{\nu'}$ is not ordered. We next recall several instance-dependent complexity measures that appear in the literature:
\begin{align}
    \label{eqn:Definition-of-H}
    \begin{split}
        H_1^{\text{neg}}= & \sum_{a=1}^K\frac{2}{\Delta^2_{0, a}},H^{\text{low}}_1=\sum_{a:\mu_a<\mu_0}\frac{2}{\Delta_{1,a}^2},
         H = \frac{2}{\Delta^2_{0, 1}}\\
        H_1 = &\sum_{a=2}^K\frac{2}{\Delta^2_{1, a}}, H_0=\sum_{a:\mu_a\geq\mu_0}\frac{2}{\Delta^2_{0, a}}\\
        H^{\text{u}} = & \frac{(\frac{K}{m}-1)}{\Delta^2}, \Delta:=\min_{a:\mu_a> \mu_0}\Delta_{a,0}, m:=\arg\max \{a:\mu_a > \mu_0\}\\
        H(j)= & \frac{1}{j}\sum_{a=1}^K \frac{1}{\max\{\Delta_{j,a}^2, \Delta_{a,0}^2\}} 
    \end{split}
\end{align}
The study of sampling complexity naturally separates into two regimes, depending on the instance $\nu$ is negative or positive. For negative instances, the problem is closely related to the classical Best-Arm Identification problem, and the sample complexity is well studied. In particular, Theorem 5.4 of \cite{pmlr-v267-li25f} establishes that, for any $\delta$-PAC algorithm and any negative instance $\nu$, $\mathbb{E}_{\nu}[\tau] \ge \Omega\!\left(H_1^{\mathrm{neg}}\log\frac{1}{\delta}\right)$. On the algorithmic side, HDoC \cite{kano2017Good}, APGAI \cite{jourdan2023anytime}, and SEE \cite{pmlr-v267-li25f} guarantee $\mathbb{E}_{\nu}[\tau] \le O (H_1^{\mathrm{neg}} (\log(1/\delta) + \log H_1^{\mathrm{neg}} ) )$, up to logarithmic factors. Moreover, for Gaussian rewards with unit variance, the Sticky Track-and-Stop (S-TaS) algorithm \cite{degenne2019pure} achieves the asymptotically optimal guarantee $\limsup_{\delta \to 0} \frac{\mathbb{E}_{\nu}[\tau]}{\log(1/\delta)} \le T^*(\nu)$, where $T^*(\nu)=H_1^{\mathrm{neg}}$ for negative instances.

In contrast, the sample complexity of positive instances remains unclear. Existing upper bounds suffer from various looseness and fail to match known lower bounds except in restricted instances. To illustrate with simplicity, we consider a positive instance $\nu$ satisfying $\mu_1 \ge \cdots \ge \mu_m > \mu_0 \ge \mu_{m+1} \ge \cdots \ge \mu_K$, and restrict attention to Gaussian rewards with variance 1. Table \ref{table:Comparison-of-bounds} summarizes representative non-asymptotic upper and lower bounds from the literature. Due to space limit, the bounds in the table are simplified; the original bounds are larger than those reported here. Even under these simplifications, the looseness remains apparent. Following Appendix A.2 of \cite{pmlr-v267-li25f}, we provide a detailed derivation of the simplified bounds in Appendix \ref{sec:Detailed-Literature-Review}.

\begin{table*}[ht]
\caption{Comparison of non-asymptotic upper bounds on $\mathbb{E}\tau$ for a positive instance $\nu$. ``LB'' denotes Lower Bound. $\surd_{\text{cond}}$ means the upper or lower bounds are optimal under some specific instances. The definitions of different $H$'s are at the equation (\ref{eqn:Definition-of-H}). We determine whether a sampling complexity upper bound is nearly optimal by comparing with the lower bound. An upper bound is nearly optimal iff the gap is up to a polynomial of $\log K, \{\log \frac{1}{\Delta_{a,0}^2}\}_{a=1}^K$.
}
\label{table:Comparison-of-bounds}
\vskip 0.15in
\begin{center}
\begin{small}
\begin{tabular}{llcc}
\toprule
Algorithm & Non-asymptotic Bound & Optimality\\
\midrule
\multicolumn{1}{p{4cm}}{S-TaS, \newline
proposed by \cite{degenne2019pure}\newline analyzed by \cite{poiani2025nonasymptoticanalysisstickytrackandstop} } &  $\mathbb{E}\tau \leq O(\frac{\log(1/\delta)}{\Delta_{0, 1}^2} + \frac{K^2}{\Delta_{0, 1}^4})$ & $\times$\\

\multicolumn{1}{p{4cm}}{HDoC, \cite{kano2017Good}}
& $\mathbb{E}\tau\leq O(\frac{\log\frac{K}{\delta}}{\Delta_{0,1}^2}+H_1\log\log\frac{1}{\delta}+\frac{K}{\epsilon^2})$ & $\times$\\

\multicolumn{1}{p{4cm}}{lilHDoC, \cite{tsai2024lil}}
& $\mathbb{E}\tau\leq O(\frac{\log\frac{K}{\delta}}{\Delta_{0,1}^2}+H_1\log\log\frac{1}{\delta}+\frac{K}{\epsilon^2})$ & $\times$ \\

\multicolumn{1}{p{4cm}}{APGAI, \cite{jourdan2023anytime}} & $\mathbb{E}\tau\leq O(H_0(\log\frac{K}{\delta}))$ & $\times$ \\

\multicolumn{1}{p{4cm}}{FWER-FWPD\newline \cite{katz2020true}} & $\mathbb{E}\tau\leq O\Big(H^{\text{u}}\log\frac{1}{\delta} \log\Big(H^{\text{u}}\log\frac{1}{\delta}\Big)\Big)$ & $\times$ \\

\multicolumn{1}{p{4cm}}{SEE, \cite{pmlr-v267-li25f}} & $\mathbb{E}\tau \leq O(\frac{\log\frac{1}{\delta}}{\Delta_{0,1}^2} + H(1)\log \frac{K}{\Delta_{0,1}^2})$ & $\surd_{\text{cond}}$\\

\multicolumn{1}{p{4cm}}{PSEEB, This work} & {\footnotesize $\mathbb{E}\tau \leq O\Big(\log K\min\limits_{1\leq j\leq m}\left\{\frac{\log\frac{1}{\delta}}{\Delta_{j,0}^2} + H(j)\log^4 K\log\frac{1}{\Delta_{j,0}^2}\right\}\Big) $} & $\surd$  \\

\midrule

\multicolumn{1}{p{4cm}}{LB, \cite{katz2020true}} & $\mathbb{E}\tau\geq \Omega(\frac{1}{m}H^{\text{low}}_1-\frac{1}{\Delta^2_{1,m+1}})$ & $\times$ \\

\multicolumn{1}{p{4cm}}{LB, \cite{pmlr-v267-li25f}} & $\mathbb{E}\tau\geq \Omega(H\log\frac{1}{\delta}+\frac{1}{m}H^{\text{low}}_1-\frac{1}{\Delta^2_{1,m+1}})$ & $\surd_{\text{cond}}$ \\

\multicolumn{1}{p{4cm}}{LB, This work} & $\mathbb{E}\tau \geq \Omega\left(\min\limits_{1\leq j \leq m} \left\{\frac{\log\frac{1}{\delta}}{\Delta_{j,0}^2}  + \frac{H(j)}{\log^2 (m+1)}\right\}\right)$ & $\surd$ \\

\bottomrule
\end{tabular}
\end{small}
\end{center}
\vskip -0.1in
\end{table*}


\cite{degenne2019pure} propose the S-TaS algorithm, which attains asymptotically optimal guarantees for both negative and positive instances. In particular, they show $\limsup_{\delta\to 0}\frac{\mathbb{E}_{\nu}[\tau]}{\log(1/\delta)} \le \frac{2}{\Delta_{0,1}^2}$, together with a matching lower bound $\lim\inf_{\delta\rightarrow 0 }\frac{\mathbb{E}\tau}{\log(1/\delta)} \geq \frac{2}{\Delta_{0,1}^2}$, establishing asymptotic optimality. A non-asymptotic analysis is left as a future direction in \cite{degenne2019pure} and later addressed by \cite{poiani2025nonasymptoticanalysisstickytrackandstop}. The resulting non-asymptotic bound, however, contains a $\delta$-independent term $\frac{K^2}{\Delta_{0,1}^4}$, which has no counterpart in known lower bounds and leads to looseness.

The HDoC algorithm of \cite{kano2017Good} yields an upper bound $C^{\text{HDoC}}_{\text{pos}}: = \frac{\log (K/\delta)}{\Delta_{0, 1}^2} + H_1\log\log\frac{1}{\delta} + \frac{K}{\epsilon^2}$, where $\epsilon = \min\{\min_{a}\Delta_{0,a},\, \min_{a}\Delta_{a,a+1}/2\}$. Its suboptimality arises from two sources. The first source is $H_1\log\log\frac{1}{\delta}$, which is explicitly acknowledged in \cite{kano2017Good}: ``Therefore, in the case that $\ldots$ $K=\Omega(\frac{\log\frac{1}{\delta}}{\log\log\frac{1}{\delta}})$\ldots  \textbf{there still exists a gap} between the lower bound in \ldots and the upper bound in \ldots''. 
The second source is $\frac{K}{\epsilon^2}$, which is infinite when $\epsilon=0$. The lilHDoC algorithm \cite{tsai2024lil}, despite incorporating LIL-based confidence bounds and an additional warm-up phase, inherits the same fundamental limitations.
The APGAI algorithm \cite{jourdan2023anytime} is originally designed for Good Arm Identification, aiming to output all the arms whose reward is above a given threshold $\mu_0$. But the analysis in  \cite{jourdan2023anytime} also implies an upper bound $O(H_0\log(K/\delta))$ for 1-identification. This bound is loose primarily due to the dependence on $H_0$, which might scale linearly with $K$, leading to a suboptimal coefficient of $\log(1/\delta)$. Moreover, if there exists an arm with $\mu_a=\mu_0$, then $H_0=\infty$. In this case, the bound is vacuous.

\cite{katz2020true} study $\epsilon$-good identification and 1-identification via the notion of Unverifiable Sample Complexity. Although this notion differs from the sampling complexity considered here, their results can be adapted as discussed in Appendix A.1 of \cite{pmlr-v267-li25f}. The resulting upper bound is suboptimal for two reasons: it scales as $\log(1/\delta)\log\log(1/\delta)$ rather than $\log(1/\delta)$, and it depends on the complexity term $H^{\mathrm{u}}$, which is governed by the smallest gap above $\mu_0$ and may be much smaller than $\Delta_{0,1}$. When $m \ll K$ and $\Delta \ll \Delta_{0,1}$, this leads to a significant gap compared to the lower bound.

\cite{pmlr-v267-li25f} propose the SEE algorithm and derive explicit upper and lower bounds, decomposing $\mathbb{E}_{\nu}[\tau]$ into $\delta$-dependent and $\delta$-independent components. They show that when $m=1$, i.e. there is only one arm whose reward is greater than $\mu_0$, the gap between upper and lower bounds is at most polynomial in $\log K$ and $\log(1/\Delta_{0,1}^2)$, implying near-optimality. For $m>1$, however, they comment the $\delta$-independent term $H(1)\log\!\bigl(K/\Delta_{0,1}^2\bigr)$ may be improvable, which is explicitly left as an open problem.

\subsection{Main Contribution}
\label{sec:Main-Contribution}

In this paper, we develop novel upper and lower sample complexity bounds for the 1-identification problem. Our results characterize the sample complexity in both positive and negative instances and close the gap for instances with multiple qualified arms.
We first establish a new lower bound for positive instances.
\begin{theorem}
    \label{theorem:main-lower-bound}
    For any $\delta$-PAC alg and any reward vector $\{\mu_a\}_{a=1}^K\in [0, 1]^K$ satisfying that $\mu_1\geq \cdots\geq \mu_m > \mu_0\geq \mu_{m+1}\geq \cdots\geq \mu_K$, any tolerance level $\delta$ such that $\delta<\min\{10^{-8}, \frac{1}{64m^2}\}$, we can find a permutation $\sigma:[K]\rightarrow [K]$ and define a positive instance $\nu$ with $r_{\sigma(a)}^{\nu} = \mu_{a}, a\in [K]$ with unit-variance Gaussian reward, such that $\mathbb{E}_{\text{alg}, \nu}\tau\geq \Omega\left(\min\limits_{1\leq j \leq m} \left\{\frac{\log\frac{1}{\delta}}{\Delta_{j,0}^2}  + \frac{H(j)}{\log^2 (m+1)}\right\}\right)$.
\end{theorem}
Theorem \ref{theorem:main-lower-bound} is our main contribution, achieving nearly joint optimality on both $\delta$-dependent and $\delta$-independent term. To the best of our knowledge, the ``Simulator'' framework \cite{simchowitz2017simulator, katz2020true} was the \textbf{only known method} for deriving $\delta$-independent lower bounds. However, results in \cite{katz2020true} are only tight under restricted settings ($\mu_1=\dots=\mu_m > \mu_0$), without joint optimality on both parts. The bottleneck is constructing a ``Simulator'' that depends on the sampling path and the reward vector. Our proposed method reduces this complex construction to finding a feasible solution for a convex optimization problem, which not only significantly simplifies the proof burden but also yields a tighter, jointly optimal result. We provide sketch proof in Section \ref{sec:Lower-Bound-for-Positive-Case} and full proof in Appendix \ref{sec:Proof-of-lower-bound-full}.

On the algorithmic side, we combine and refine ideas from \cite{katz2020true,pmlr-v267-li25f} to propose a new algorithm, \emph{Parallel Sequential Exploration--Exploitation on Brackets} (Algorithm~\ref{alg:Parallel-SEE-on-Bracket}, abbreviated as PSEEB). We establish the following guarantees.
\begin{theorem}
    \label{theorem:delta-pac}
    Algorithm \ref{alg:Parallel-SEE-on-Bracket} is $\delta$-PAC.
\end{theorem}
\begin{theorem}
    \label{theorem:Parallel-SEE-Etau-upper-bound}
    Apply Algorithm \ref{alg:Parallel-SEE-on-Bracket} to a 1-identification instance $\nu$, we have
    \begin{align*}
        \mathbb{E}_{\nu}\tau \leq \begin{cases}
             O\left(\log K\min\limits_{1\leq j\leq m}\left\{\frac{\log\frac{1}{\delta}}{\Delta_{j,0}^2} + \log^3 K(\log(4K) + \log_2\frac{1}{\Delta_{j,0}^2}) H(j)\right\}\right) & \nu \text{ is positive}\\
            O\left((\log K)H_1^{\text{neg}}\log\frac{H_1^{\text{neg}}}{\delta} \right) & \nu \text{ is negative}
        \end{cases}
    \end{align*}
\end{theorem}
Combining Theorems~\ref{theorem:delta-pac} and~\ref{theorem:Parallel-SEE-Etau-upper-bound}, we conclude that Algorithm~\ref{alg:Parallel-SEE-on-Bracket} is Strong $\delta$-PAC. Moreover, the gap between the lower bound in Theorem~\ref{theorem:main-lower-bound} and the upper bound in Theorem~\ref{theorem:Parallel-SEE-Etau-upper-bound} is at most polynomial in $\log K$ and $\{\log(1/\Delta_{0,a}^2)\}_{a=1}^m$. This shows that PSEEB is nearly optimal for all 1-identification instances and identifies $\min_{1 \le j \le m} \left\{\frac{\log(1/\delta)}{\Delta_{j,0}^2} + H(j)\right\}$ as the appropriate instance-dependent complexity measure on positive instances.

Compared to prior work, our results resolve an open problem posed in \cite{pmlr-v267-li25f} by providing the first nearly optimal upper bound for instances with multiple qualified arms.


\subsection{Literature Review on Other Related Works}
If the threshold $\mu_0$ is viewed as an additional arm with deterministic reward, then 1-identification can be reduced to a standard Best-Arm Identification (BAI) problem \cite{even2002pac,gabillon2012best,jamieson2013finding,kalyanakrishnan2012pac,karnin2013almost,jamieson2014best,chen2015optimal}. However, this reduction is known to be inefficient. Both asymptotic analyses \cite{degenne2019pure,garivier2016optimal} and non-asymptotic results \cite{pmlr-v267-li25f,kano2017Good} show that such a formulation incurs unnecessary sampling. Intuitively, 1-identification does not require high-confidence pairwise comparisons among all arms, and therefore does not necessitate the intensive sampling in BAI Algorithms.

Beyond 1-identification, \cite{degenne2019pure} research on several pure-exploration variants with multiple correct answers. These include identifying all arms in $\{a : \mu_a \ge \mu_1 - \epsilon\}$ \cite{mason2020finding}, or identifying a single arm from this set \cite{even2002pac,gabillon2012best,kalyanakrishnan2012pac,kaufmann2013information,katz2020true}. As observed in \cite{katz2020true}, when the gap $\Delta_{0,1}$ is known and $\mu_1 > \mu_0$, the answer sets of $\epsilon$-good arm identification and 1-identification coincide by setting $\epsilon = \Delta_{0,1}$, allowing the same sampling and recommendation rules. When the gaps are unknown, however, the precise relationship between these problems remains unclear. In addition, thresholding bandits \cite{locatelli2016optimal,mukherjee2017thresholding} aim to identify arms with mean value above $\mu_0$ under a fixed budget. The correct output might not be unique under their proposed problem setting. \cite{kano2017Good} also target at identifying all the arms with mean value above $\mu_0$, but they focus on the fixed confidence setting. Since \cite{kano2017Good} provide an upper bound for each outputting time index, their result can be generalized to 1-identification.

Finally, our algorithmic design builds on the idea of running multiple algorithm copies in parallel, which has appeared in prior work. In particular, \cite{chen2015optimal,chen2017nearly,katz2020true} employ a master procedure to manage multiple executions of a base algorithm. The approach of \cite{katz2020true} allocates $2^{\ell}$ rounds to the execution indexed by $\ell$, while \cite{chen2015optimal,chen2017nearly} schedule parallel executions according to a global round index $t \in \mathbb{N}$, activating an execution when its index divides $t$. In contrast, our method adopts a simple round-robin scheduling scheme, which enables a balanced execution of parallel algorithm instances.

\section{Lower Bound of Sample Complexity, for Positive Instances}
\label{sec:Lower-Bound-for-Positive-Case}
In this section, we continue to work with a reward vector $\{\mu_a\}_{a=1}^K$ satisfying $\mu_1\geq \mu_2\geq \cdots \geq \mu_m > \mu_0\geq \mu_{m+1}\geq\cdots \geq \mu_K$ and consider the instance $\nu$ defined by $r_a^{\nu} = \mu_a$. 
Throughout the lower-bound analysis, we assume unit-variance Gaussian rewards. For any instance $\nu'$ and any permutation $\sigma:[K]\to[K]$ (a bijection), we define the permuted instance $\sigma(\nu')$ by $r^{\sigma(\nu')}_{\sigma(a)} \coloneqq r_a^{\nu'}, \forall a\in[K]$.
We begin by introducing a symmetry notion that allows us to restrict attention to algorithms that do not prefer specific arm labels.
\begin{definition}[Symmetric algorithm]
    An algorithm $\widetilde{\textnormal{alg}}$ is \emph{symmetric} if for any instance $\nu$ with reward vector $\{\mu_a\}_{a=1}^K$, any sequence $\{a_t\}_{t=1}^T \in [K]^T$, any integer $T\ge 1$, and any two permutations $\sigma',\sigma'':[K]\to[K]$, it holds that
    \begin{align*}
        \Pr_{\sigma'(\nu),\,\widetilde{\textnormal{alg}}}\!\bigl(\forall t\in[T],\, A_t=\sigma'(a_t),\, \tau>T\bigr)
        =\Pr_{\sigma''(\nu),\,\widetilde{\textnormal{alg}}}\!\bigl(\forall t\in[T],\, A_t=\sigma''(a_t),\, \tau>T\bigr).
    \end{align*}
\end{definition}
Intuitively, under a symmetric algorithm $\widetilde{\text{alg}}$, the reward distribution of the pulled arm $A_t$ at time $t$ is invariant under any permutation on the arms' indexes.
In particular, we have $\mathbb{E}_{\sigma'(\nu),\,\widetilde{\textnormal{alg}}}[\tau]=\mathbb{E}_{\sigma''(\nu),\,\widetilde{\textnormal{alg}}}[\tau],\text{for permutations } \sigma',\sigma''$.
A similar definition also appears in \cite[Definition~12]{simchowitz2017simulator}, which further requires the invariance of output distribution; this extra requirement is not needed here. The following lemma shows that it suffices to consider symmetric algorithms for lower bounds.
\begin{lemma}
    \label{lemma:property-of-symmetric-algorithm}
    Let $S$ denote the set of all permutations of $[K]$. For any $\delta$-PAC algorithm $\textnormal{alg}$, there exists a $\delta$-PAC symmetric algorithm $\widetilde{\textnormal{alg}}$ such that for every instance $\nu'$, $\max\limits_{\sigma\in S} \mathbb{E}_{\textnormal{alg},\,\sigma(\nu')}[\tau] \geq  \mathbb{E}_{\widetilde{\textnormal{alg}},\,\nu'}[\tau]$.
\end{lemma}
The proof of Lemma \ref{lemma:property-of-symmetric-algorithm} is in Appendix \ref{sec:Proof-of-Preparation-lower-bound}. 
It is obtained by randomizing a permutation of arm labels prior to executing $\textnormal{alg}$. Given Lemma \ref{lemma:property-of-symmetric-algorithm}, to prove Theorem~\ref{theorem:main-lower-bound}, it suffices to establish the following lower bound for symmetric algorithms.
\begin{theorem}
    \label{theorem:lower-bound-positive-case}
    For any symmetric and $\delta$-PAC alg and any reward vector $\{\mu_a\}_{a=1}^K\in [0, 1]^K$ satisfying that $\mu_1\geq \cdots\geq \mu_m > \mu_0\geq \mu_{m+1}\geq \cdots\geq \mu_K$, any tolerance level $\delta$ such that $\delta<\min\{10^{-8}, \frac{1}{64m^2}\}$, define a positive instance $\nu$ with $r_{a}^{\nu} = \mu_{a}, a\in [K]$ with unit-variance Gaussian reward. We have $\mathbb{E}_{\text{alg}, \nu}\tau\geq \Omega\left(\min\limits_{1\leq j \leq m} \left\{\frac{\log\frac{1}{\delta}}{\Delta_{j,0}^2}  + \frac{H(j)}{\log^2 (m+1)}\right\}\right)$.
\end{theorem}
The assumption $\delta < \min\{10^{-8}, 1/(64m^2)\}$ is to ensure $\delta\cdot m^2$ is upper bounded by an absolute constant. The upper bound eliminates the possibility of being $\delta$-PAC by outputting a uniformly chosen arm without any arm pulls.

% ...

\begingroup
\renewcommand{\proofname}{Sketch Proof of Theorem \ref{theorem:lower-bound-positive-case}}
\begin{proof}
We outline the main steps;
the full proof is in Appendix \ref{section:proof-of-lower-bound-positive-instance}. The main idea is to establish Lemma \ref{lemma:lower-bound-positive-optimization-formulation}, 
\begin{lemma}
    \label{lemma:lower-bound-positive-optimization-formulation}
    Assume $\delta < \min\{10^{-8}, 1/(64m^2)\}$, for any symmetric and $\delta$-PAC alg, $\mathbb{E}_{\nu,\text{alg}}\tau$ is lower bounded by the $\frac{1}{3200}$ times optimal value of the following optimization problem
    \begin{align}
        \label{eqn:opt-formualtion-lower-bound-positive}
        \min_{v,\{p_j\}_{j=1}^m}\quad & v \\[0.3em]
        \text{s.t.}\quad
        & 0 \le p_j \le 1 \quad \forall j\in[m]
        \tag{C1}\label{con:lb-pos-C1}\\
        & v \ge 0 \tag{C2}\label{con:lb-pos-C2}\\
        & v \ge \sum_{j=1}^m \frac{p_j\log\frac{1}{\delta}}{\Delta_{j,0}^2}  \tag{C3}\label{con:lb-pos-C3}\\
        & \sum_{j=1}^m p_j \ge \frac{1}{2}
        \tag{C4}\label{con:lb-pos-C4}\\
        & v\geq \sum_{a=1}^K
        \frac{p_j}{\max\{\Delta_{a,0}^2,\Delta_{a,j}^2\}
        \left(1+\log\frac{1}{p_j}\right)}
        \quad \forall j\in[m]
        \tag{C5}\label{con:lb-pos-C5}.
    \end{align}
\end{lemma}
Given Lemma \ref{lemma:lower-bound-positive-optimization-formulation}, we proceed to prove the Theorem 1 as follows. We first prove the above optimization problem (\ref{eqn:opt-formualtion-lower-bound-positive}) is convex. Then, we turn to construct an explicit feasible dual solution, and prove the objective value of this feasible solution is lower bounded by $\Omega\left(\min\limits_{1\leq j \leq m} \left\{\frac{\log\frac{1}{\delta}}{\Delta_{j,0}^2}  + \frac{H(j) }{\log^2 (m+1)}\right\}\right)$. By the strong duality, we complete the proof of Theorem \ref{theorem:lower-bound-positive-case}.

Denote event $\mathcal{E}_j=\left\{N_j(\tau)\geq \frac{\log\frac{1}{\delta}}{100\Delta_{j,0}^2}, \hat{a}=j\right\}$. To prove Lemma \ref{lemma:lower-bound-positive-optimization-formulation}, we show taking $v = 3200 \mathbb{E}_{\nu}\tau$, $p_j=\Pr_{\nu} (\mathcal{E}_j)$ for $j\in [m]$ is a feasible solution to (\ref{eqn:opt-formualtion-lower-bound-positive}). 
The argument repeatedly applies the transportation (change-of-measure) inequality \cite[Lemma~18]{kaufmann2016complexity} to relate $\Pr_{\nu}(\mathcal{E}_j)$ to probabilities under suitable alternative instances.
We prove the feasibility of $v = 3200 \mathbb{E}_{\nu}\tau$, $p_j=\Pr_{\nu} (\mathcal{E}_j)$ by going through the constraint (\ref{con:lb-pos-C1}), (\ref{con:lb-pos-C2}), (\ref{con:lb-pos-C3}), (\ref{con:lb-pos-C4}) (\ref{con:lb-pos-C5}).

For (\ref{con:lb-pos-C1}), (\ref{con:lb-pos-C2}), they are evident. As the expected pulling times must be non-negative and the probability must be in the interval $[0, 1]$.

For (\ref{con:lb-pos-C3}), notice that $\mathds{1}(\tau=+\infty)+\mathds{1}(\tau<+\infty,\hat{a}\in [K]\cup\{\textsf{None}\})=1$, we have $\sum_{j=1}^m\mathds{1}(\hat{a}=j)\leq 1$, further $\sum_{j=1}^m\mathds{1}(\mathcal{E}_j)\leq 1$. This observation concludes
\begin{align*}
    3200\mathbb{E}\tau 
    \geq \mathbb{E}\tau 
    \geq \sum_{j=1}^m\mathbb{E}\tau\mathds{1}(\mathcal{E}_j)\geq  \sum_{j=1}^m\frac{\log\frac{1}{\delta}}{100\Delta_{j,0}^2}\Pr_{\nu}(\mathcal{E}_j).
\end{align*}

For (\ref{con:lb-pos-C4}), we start by proving
\begin{align}
    \Pr_{\nu}(\tau<+\infty,\hat{a}\in [m]) \geq 1-5\delta, \label{eqn:output-hata-Delta>0}
\end{align}
which must hold regardless of whether $\{a:\mu_a=\mu_0\}$ is empty. The main idea is to utilize the assumption that the algorithm is $\delta$-PAC.

Then, we are ready to prove (\ref{con:lb-pos-C4}) by contradiction. Denote $\bar{\mathcal{E}_j}=\{N_{j}(\tau) < \frac{\log\frac{1}{\delta}}{100\Delta_{j,0}^2}, \hat{a}=j\}$. If $\sum_{j=1}^m \Pr_{\nu}(\mathcal{E}_j) <\frac{1}{2}$, together with (\ref{eqn:output-hata-Delta>0}), we can find a $j_0\in [m]$ such that
\begin{align}
    \Pr_{\nu}\left(\bar{\mathcal{E}}_{j_0}\right)\geq \frac{1}{8m} \geq \sqrt{\delta}.\label{eqn:j_0-existence-wrong-output}
\end{align}
We consider an alternative instance $\nu'_{j_0}(\epsilon)$ with $r_{a'}^{\nu'_{j_0}}=\begin{cases}\mu_0-\epsilon & a'=j_0 \\ r_{a'}^{\nu} & else\end{cases}$. The $\delta$-PAC requirement of algorithm means that
\begin{align}
    \Pr_{\nu'_{j_0}(\epsilon)}(\bar{\mathcal{E}}_{j_0}) < \delta\label{eqn:j_0-existence-wrong-output-nuj0}
\end{align}
Combining (\ref{eqn:j_0-existence-wrong-output-nuj0}),(\ref{eqn:j_0-existence-wrong-output}), we have $\log\frac{\Pr_{\nu}\left(\bar{\mathcal{E}}_{j_0}\right)}{\Pr_{\nu'_{j_0}(\epsilon)}(\bar{\mathcal{E}}_{j_0})}\geq \frac{1}{2}\log\frac{1}{\delta}$ holds for all $\epsilon > 0$. Meanwhile, Lemma \ref{lemma:enhanced-probability-gap-between-unique-arm-differs-instances}, which utilizes the Transportation Equality \cite[Lemma~18]{kaufmann2016complexity}, suggests that
\begin{align}
    \label{eqn:bar-j0-transportation}
    & \log\frac{\Pr_{\nu}\left(\bar{\mathcal{E}}_{j_0}\right)}{\Pr_{\nu'_{j_0}(\epsilon)}(\bar{\mathcal{E}}_{j_0})}\notag\\
    \leq & \sqrt{2\frac{\log\frac{1}{\delta}}{100\Delta_{j_0,0}^2}(\Delta_{j_0,0}+\epsilon)^2\log\frac{1}{\sqrt{\delta}}}+\sqrt{2\pi \frac{\log\frac{1}{\delta}}{100\Delta_{j,0}^2}(\Delta_{j_0,0}+\epsilon)^2} + \frac{1}{2} \frac{\log\frac{1}{\delta}}{100\Delta_{j_0,0}^2}(\Delta_{j_0,0}+\epsilon)^2
\end{align}
Taking $\epsilon\rightarrow 0$, we can show the right hand side of (\ref{eqn:bar-j0-transportation}) is less than $\frac{1}{2}\log\frac{1}{\delta}$ for $\delta<10^{-8}$, which leads to a contradiction. Lemma \ref{lemma:high-prob-lower-bound-positive-instance} merges all the above ideas and conclude $\sum_{j=1}^m \Pr_{\nu}(\mathcal{E}_j) \geq \frac{1}{2}$.


For (\ref{con:lb-pos-C5}), we first prove Lemma \ref{lemma:lower-bound-for-EN_a-tau-Pr-hata-j}, asserting that
\begin{align}
    \mathbb{E}_{\nu}N_a(\tau)\geq \frac{1}{4}T_j(a)\Pr_{\nu}\left(N_j(\tau)\geq \frac{\log\frac{1}{\delta}}{100\Delta_{j,0}^2}, \hat{a}=j\right).\label{eqn:lower-bound-EN_a-prob-j}
\end{align}
for each $j\in [m], a\in [K], j\neq a$, where $T_j(a)=\lceil \frac{1}{200 \max\{\Delta_{j,0}^2, \Delta_{a,j}^2\}\Big(1+\log\frac{1}{\Pr_{\nu}(\mathcal{E}_j)}\Big) }\rceil$. Given Lemma \ref{lemma:lower-bound-for-EN_a-tau-Pr-hata-j} holds, sum up (\ref{eqn:lower-bound-EN_a-prob-j}) for each $a\in [K]$ and use the fact that $\frac{1}{\max\{\Delta_{j,0}^2, \Delta_{a,j}^2\}} \geq \frac{1}{4\max\{\Delta_{a,0}^2, \Delta_{a,j}^2\}}$, we prove (\ref{con:lb-pos-C5}). Then, the remaining work is to prove (\ref{eqn:lower-bound-EN_a-prob-j}). In the following, we present the sketch idea, which highly relies on the definition of $T_j(a)$. 



For $n\in\mathbb{N}$, we define $\tau_n^{j}(a)$ by looking for an time index $t$, such that $N_j(t)+N_a(t)=n, A_{t+1}\in \{j, a\}$. Such a time index $t$ is unique and $t<\tau$, if it exists. We take $\tau_n^{j}(a)=t$ if $t$ exists, or $\tau_n^{j}(a)=+\infty$ if $t$ doesn't exist. Since $A_{t+1}$ is fully determined by the history up to time index $t$, for $t\in\mathbb{N}$. $\tau_n^{j}(a)$ is a stopping time regrading the filtration $\{\sigma(H(t))\}_{t \ge 1}$.

Denote $\tilde{\mathcal{E}}^{aj}(n) = \{\tau_n^{j}(a) < \tau\}$, $\tilde{\mathcal{E}}^{aj}_j(n) = \{A_{\tau_n^{j}(a)+1}=j, \tau_n^{j}(a) < \tau\}$, $\tilde{\mathcal{E}}^{aj}_{a}(n) = \{A_{\tau_n^{j}(a)+1}=a, \tau_n^{j}(a) < \tau\}$. We prove two inequalities.
\begin{align}
    & \Pr_{\nu}(\tilde{\mathcal{E}}^{aj}(n)) \geq  \Pr_{\nu}(\mathcal{E}_j), \forall  n\leq T_j(a)-1\label{eqn:choose-arm-a-over-prediction}.\\
    & \Pr_{\nu}(\tilde{\mathcal{E}}^{aj}_{a}(n)) \geq \frac{1}{4} \Pr_{\nu}(\tilde{\mathcal{E}}^{aj}(n)), \forall  n\leq T_j(a)-1. \label{eqn:choose-arm-a-over-finite-tauja}
\end{align}
Given (\ref{eqn:choose-arm-a-over-prediction}) and (\ref{eqn:choose-arm-a-over-finite-tauja}), from the fact that we only pull arm a at $\tau_n^{j}(a)+1$, we can conclude
\begin{align*}
    \mathbb{E}_{\nu}N_a(\tau) 
    = &\sum_{n=0}^{+\infty}\Pr_{\nu}(\tilde{\mathcal{E}}^{aj}_{a}(n)) \geq \sum_{n=0}^{T_j(a)-1}\Pr_{\nu}(\tilde{\mathcal{E}}^{aj}_{a}(n))
    \stackrel{(\ref{eqn:choose-arm-a-over-finite-tauja})}{\geq} \sum_{n=0}^{T_j(a)-1}\frac{1}{4} \Pr_{\nu}(\tilde{\mathcal{E}}^{aj}(n))\\
    \stackrel{(\ref{eqn:choose-arm-a-over-prediction})}{\geq} & \frac{1}{4}\sum_{n=0}^{T_j(a)-1} \Pr_{\nu}(\mathcal{E}_j)
    = \frac{1}{4}T_j(a) \Pr_{\nu}(\mathcal{E}_j).
\end{align*}
Thus, to prove (\ref{eqn:lower-bound-EN_a-prob-j}), we suffice to show (\ref{eqn:choose-arm-a-over-prediction}) (\ref{eqn:choose-arm-a-over-finite-tauja}) are both correct.

For (\ref{eqn:choose-arm-a-over-prediction}), we notice that  $\frac{1}{\max\{\Delta_{j,0}^2, \Delta_{a,j}^2\}} \leq \frac{1}{\Delta_{j,0}^2}$, and that $T_j(a)< \frac{\log\frac{1}{\delta}}{100\Delta_{j,0}^2}$. That means for $n\leq T_j(a)$, we have $\mathcal{E}_j \Rightarrow \{N_j(\tau)\geq \frac{\log\frac{1}{\delta}}{100\Delta_{j,0}^2}\} \Rightarrow \{\tau_n^{j}(a) < \tau\}$, which suggests (\ref{eqn:choose-arm-a-over-prediction}).

For (\ref{eqn:choose-arm-a-over-finite-tauja}), we introduce an alternative instance $\nu_{j,a}$, with $r_{a'}^{\nu_{j,a}}=\begin{cases}r_a^{\nu} & a'=j\\ r_j^{\nu} & a'=a\\ r_{a'}^{\nu} & else\end{cases}$. The instance distribution $\nu_{j,a}$ swaps the position of arm $a$ and $j$ compared to the instance $\nu$. 
Notice that 
\begin{align}
    \frac{1}{\max\{\Delta_{j,0}^2, \Delta_{a,j}^2\}} \leq \frac{1}{\Delta_{j,a}^2},\quad  T_j(a) \leq \frac{O(1)}{\Delta_{j,a}^2(1+\log\frac{1}{\Pr_{\nu}(\mathcal{E}_j)})}\label{eqn:T_ja-upper-bound}
\end{align}
Lemma \ref{lemma:inequality-for-kl-diver-of-conditional-prob} and \ref{lemma:unability-to-differentiate-two-arms-when-total-pulling-small} combine (\ref{eqn:T_ja-upper-bound}) and the Transportation Equality(Lemma 18 in \cite{kaufmann2016complexity}), to conclude
\begin{align}
    \text{kl}\Big(\Pr_{\nu}(\tilde{\mathcal{E}}^{aj}_j(n)), \Pr_{\nu_{ja}}(\tilde{\mathcal{E}}^{aj}_j(n))\Big) \leq 0.4\Pr_{\nu}(\tilde{\mathcal{E}}^{aj}(n)), \forall  n\leq T_j(a)-1\label{eqn:kl-nu-nuja-pull-j}
\end{align}
where $\text{kl}(x,y)=x\log\frac{x}{y}+(1-x)\log\frac{1-x}{1-y}$. $1+ \log \frac{1}{\Pr_{\nu}(\mathcal{E}_j)}$ in the denominator of $T_j(a)$ is crucial to show (\ref{eqn:kl-nu-nuja-pull-j}). (\ref{eqn:kl-nu-nuja-pull-j}) suggests that the algorithm is unable to distinguish $\nu$ and $\nu_{ja}$ at the round $t$ such that $N_a(t)+N_j(t)\leq T_j(a)$. For a symmetric algorithm, we have $\Pr_{\nu_{ja}}(\tilde{\mathcal{E}}^{aj}_j(n))=\Pr_{\nu}(\tilde{\mathcal{E}}^{aj}_a(n))$, which means
\begin{align}
    \text{kl}\Big(\Pr_{\nu}(\tilde{\mathcal{E}}^{aj}_j(n)), \Pr_{\nu}(\tilde{\mathcal{E}}^{aj}_a(n))\Big) \leq 0.4\Pr_{\nu}(\tilde{\mathcal{E}}^{aj}(n)), \forall  n\leq T_j(a)-1\label{eqn:kl-nu-pull-j-a}
\end{align}
By the Pinsker's Inequality $\text{kl}(x,y)\geq 2(x-y)^2$ and $\Pr_{\nu}(\tilde{\mathcal{E}}^{aj}_j(n))+ \Pr_{\nu}(\tilde{\mathcal{E}}^{aj}_a(n))=\Pr_{\nu}(\tilde{\mathcal{E}}^{aj}(n))$, we can derive
\begin{align*}
    \Pr_{\nu}(\tilde{\mathcal{E}}^{aj}_a(n)) \geq \frac{1}{4}\Pr_{\nu}(\tilde{\mathcal{E}}^{aj}(n)), \forall  n\leq T_j(a)-1.
\end{align*}
The above inequality suggests that we have proved (\ref{eqn:choose-arm-a-over-finite-tauja}), hence (\ref{con:lb-pos-C5}) is justified.
\end{proof}
\endgroup

\section{Algorithm Design for the Upper Bound of the Sample Complexity}
\label{sec:upper-bound-analysis}
In this section, we present our proposed algorithm, \emph{Parallel Sequential Exploration Exploitation on Brackets} (PSEEB), together with a high-level discussion of its design and analysis.


\begin{algorithm}
    \caption{Parallel Sequential Exploration Exploitation on Brackets(PSEEB)}
    \label{alg:Parallel-SEE-on-Bracket}
    {\bfseries Input:} Threshold $\mu_0$, Error tolerance level $\delta$, Arm number $K$\;

    Randomly generate a permutation of $\sigma:[K]\rightarrow [K]$, take $B_b=\{\sigma^{-1}(i)\}_{i=1}^{\min\{ 2^{b-1}, K\} }$ for $b=1,2,\cdots, \lceil \log_2 K\rceil + 1$\tcp*[l]{$\sigma^{-1}(i)$ denotes the arm index in position $i$}\label{alg-line:random-permutation}

    Create algorithm copies $\{\text{alg}_b\}_{b=1}^{\lceil \log_2 K\rceil + 1}$ by specifying $\text{alg}_b$ through calling Algorithm \ref{alg:SEE-with-Bracket} with $B\gets B_b$, $\mu_0\gets\mu_0$, $\delta\gets \frac{\delta}{\lceil \log_2 K\rceil + 1}$, $C\gets 1.01$, $K\gets K$. Samples across brackets are not shared\;\label{alg-line:create-alg-copy}
    \For{round index $t=1,2,\cdots$}{\label{alg-line:global-round-main-boday}
        \For{Bracket Index $b=1,2,\cdots,\lceil \log_2 K\rceil + 1$}{
            Execute $\text{alg}_b$ for one time round\;
            \If{$\text{alg}_b$ terminates with output $\hat{a}\in [K]\cup \{\textsf{None}\}$}{
                Terminate the whole algorithm and return $\hat{a}$.
            }
        }
    }
\end{algorithm}


Algorithm \ref{alg:Parallel-SEE-on-Bracket} merges and adapts the bracketing idea of \cite{katz2020true} with the Sequential Exploration--Exploitation (SEE) framework of \cite{pmlr-v267-li25f}.
In \cite{katz2020true}, each bracket is generated by an independent random shuffle. In contrast, PSEEB defines \emph{all} brackets from a \emph{single} random permutation of $[K]$, yielding a nested family $\{B_b\}_{b=1}^{\lceil\log_2 K\rceil + 1}$ satisfying $B_1\subset B_2\subset\cdots\subset B_{\lceil\log_2 K\rceil + 1}=[K]$. For a bracket index $1\leq b \leq \lceil\log_2 K\rceil$, we have $|B_b|=2^b$, and $|B_{\lceil\log_2 K\rceil + 1}|=\min\{2^{\lceil\log_2 K\rceil}, K\}$. 
This monotone structure is crucial for our analysis, in particular for Lemma \ref{lemma:minimum-bracket-index-that-contains-[j]-enhanced}, which guarantees that with nontrivial probability the first several brackets contain at least one qualified arm when multiple qualified arms exist.
\begin{lemma}
    \label{lemma:minimum-bracket-index-that-contains-[j]-enhanced}
    Given a qualified arm index $j\in [m]$, denote bracket index $b_j=\min\{b\in [\lceil \log_2 K\rceil + 1 ]: B_b\cap [j]\neq \emptyset\}$. We have $\Pr(b_j\geq \tilde{b})\leq \exp\left(-\frac{j|B_{\tilde{b}-1}|}{K}\right), \forall \tilde{b} \in [\lceil \log_2 K\rceil + 1]$.
    Here we denote $B_0=\emptyset$ and take $|B_0|=0$. 
\end{lemma}

Then, PSEEB instantiates $\lceil\log_2 K\rceil+1$ independent algorithm copies via Algorithm \ref{alg:SEE-with-Bracket} (given in Appendix \ref{sec:Proof-of-Upper-Bound}). This procedure is adapted from Algorithms~2--5 of \cite{pmlr-v267-li25f}, named as Sequential Exploration Exploitation(SEE). We adapt SEE with modified budget allocation and stopping rules to accommodate bracketing. Algorithm copy $b\in [\lceil\log_2 K\rceil + 1]$ takes bracket $B_b$ as an input parameter and only pull arms in $B_b$. Detailed discussion of Algorithm \ref{alg:SEE-with-Bracket} is in Appendix \ref{sec:Proof-of-Upper-Bound}.

Algorithm \ref{alg:Parallel-SEE-on-Bracket} manages all the algorithm copies in a round-robin manner. In each global round (Line \ref{alg-line:global-round-main-boday}), Algorithm \ref{alg:Parallel-SEE-on-Bracket} execute one of the algorithm copies with one arm pull, and the overall algorithm terminates as soon as some copy outputs $\hat{a}\in [K]\cup\{\textsf{None}\}$. By construction, only the copy $\textnormal{alg}_{\lceil\log_2 K\rceil+1}$ (which operates on $B_{\lceil\log_2 K\rceil+1}=[K]$) is permitted to output $\textsf{None}$. For any other bracket $B_b$ with $|B_b|<K$, the corresponding copy never outputs $\textsf{None}$; even if it becomes convinced that $\max_{a\in B_b}\mu_a<\mu_0$, it has to continue sampling. This differs from SEE \cite{pmlr-v267-li25f} and is necessary for the analysis of $\mathbb{E}\tau$ under bracketing.

Let $\tau$ denote the stopping time of PSEEB, and let $\tau(B_b)$ denote the number of rounds executed by copy $\textnormal{alg}_b$ up to its own termination (possibly $\infty$). Since the Algorithm \ref{alg:Parallel-SEE-on-Bracket} terminates upon the first terminating copy and adopts the output, we have $\tau \leq (\lceil\log_2 K\rceil +1) \tau(B_b)$ holds for all $b\in [\lceil\log_2 K\rceil+1]$, and hence $\mathbb{E}\tau \leq [\lceil\log_2 K\rceil+1] \mathbb{E}\tau(B_b)$. 


For a positive instance $\nu$, the intuition of sample complexity $O\Big(\log K \min_{1\leq j\leq m}\big\{\log(1/\delta)/\Delta_{j,0}^2 + H(j)\log^4 K \log 1/\Delta_{j,0}^2\big\}\Big)$ in Theorem \ref{theorem:Parallel-SEE-Etau-upper-bound} is as follows. Consider a ``lucky'' agent with prior knowledge $j^* = \arg\min_{1\leq j \leq m} \{(\log(1/\delta))/\Delta_{j,0}^2 + H(j)\}$. This agent samples $K/j^*$ arms in $[K]$ uniformly at random, successfully capturing an arm $j^*$ with reward $\mu_{j^*}$. However, the arm's unknown position requires the agent $H(j^*)$ pulls to identify $j^*$ at a low confidence level. Subsequently, it must pull arm $j^*$ another $O(\log(1/\delta)/\Delta_{j^*,0}^2)$ times to confirm $j^*$ is qualified with confidence $1-\delta$. Algorithm \ref{alg:Parallel-SEE-on-Bracket} maintains multiple parallel copies, wishing $j^*$ is captured in an early bracket. Since the agent cannot pre-identify the correct bracket and low-confidence identification may be erroneous, the expected sample complexity includes an additional factor of $\log^4 K \log \frac{1}{\Delta_{j,0}^2}$ for the rough identification phase.

Theorems \ref{theorem:delta-pac} and \ref{theorem:Parallel-SEE-Etau-upper-bound} are derived from several technical lemmas, leveraging Lemma \ref{lemma:minimum-bracket-index-that-contains-[j]-enhanced}. A proof sketch is provided in Appendix \ref{sec:Sketch-Proof-of-upper-bound}, with the full proof in Appendix \ref{sec:performance-alg-parallel-SEE}.

\section{Conclusion}

This paper works on the open problem 1-identification, providing a novel lower bound for positive instances derived by a specialized optimization formulation. We further design an algorithm that is nearly optimal across all instances. Further refining logarithmic factors remains a future research.


\bibliographystyle{plainnat}
\bibliography{colt-main}

\begin{thebibliography}{23}
\providecommand{\natexlab}[1]{#1}
\providecommand{\url}[1]{\texttt{#1}}
\expandafter\ifx\csname urlstyle\endcsname\relax
  \providecommand{\doi}[1]{doi: #1}\else
  \providecommand{\doi}{doi: \begingroup \urlstyle{rm}\Url}\fi

\bibitem[Chen and Li(2015)]{chen2015optimal}
Lijie Chen and Jian Li.
\newblock On the optimal sample complexity for best arm identification.
\newblock \emph{arXiv preprint arXiv:1511.03774}, 2015.

\bibitem[Chen et~al.(2017)Chen, Li, and Qiao]{chen2017nearly}
Lijie Chen, Jian Li, and Mingda Qiao.
\newblock Nearly instance optimal sample complexity bounds for top-k arm selection.
\newblock In \emph{Artificial Intelligence and Statistics}, pages 101--110. PMLR, 2017.

\bibitem[Degenne and Koolen(2019)]{degenne2019pure}
R{\'e}my Degenne and Wouter~M Koolen.
\newblock Pure exploration with multiple correct answers.
\newblock \emph{Advances in Neural Information Processing Systems}, 32, 2019.

\bibitem[Even-Dar et~al.(2002)Even-Dar, Mannor, and Mansour]{even2002pac}
Eyal Even-Dar, Shie Mannor, and Yishay Mansour.
\newblock Pac bounds for multi-armed bandit and markov decision processes.
\newblock In \emph{Computational Learning Theory: 15th Annual Conference on Computational Learning Theory, COLT 2002 Sydney, Australia, July 8--10, 2002 Proceedings 15}, pages 255--270. Springer, 2002.

\bibitem[Gabillon et~al.(2012)Gabillon, Ghavamzadeh, and Lazaric]{gabillon2012best}
Victor Gabillon, Mohammad Ghavamzadeh, and Alessandro Lazaric.
\newblock Best arm identification: A unified approach to fixed budget and fixed confidence.
\newblock \emph{Advances in Neural Information Processing Systems}, 25, 2012.

\bibitem[Garivier and Kaufmann(2016)]{garivier2016optimal}
Aur{\'e}lien Garivier and Emilie Kaufmann.
\newblock Optimal best arm identification with fixed confidence.
\newblock In \emph{Conference on Learning Theory}, pages 998--1027. PMLR, 2016.

\bibitem[Jamieson and Nowak(2014)]{jamieson2014best}
Kevin Jamieson and Robert Nowak.
\newblock Best-arm identification algorithms for multi-armed bandits in the fixed confidence setting.
\newblock In \emph{2014 48th Annual Conference on Information Sciences and Systems (CISS)}, pages 1--6. IEEE, 2014.

\bibitem[Jamieson et~al.(2013)Jamieson, Malloy, Nowak, and Bubeck]{jamieson2013finding}
Kevin Jamieson, Matthew Malloy, Robert Nowak, and Sebastien Bubeck.
\newblock On finding the largest mean among many.
\newblock \emph{arXiv preprint arXiv:1306.3917}, 2013.

\bibitem[Jamieson et~al.(2014)Jamieson, Malloy, Nowak, and Bubeck]{jamieson2014lil}
Kevin Jamieson, Matthew Malloy, Robert Nowak, and S{\'e}bastien Bubeck.
\newblock lil’ucb: An optimal exploration algorithm for multi-armed bandits.
\newblock In \emph{Conference on Learning Theory}, pages 423--439. PMLR, 2014.

\bibitem[Jourdan and R{\'e}da(2023)]{jourdan2023anytime}
Marc Jourdan and Cl{\'e}mence R{\'e}da.
\newblock An anytime algorithm for good arm identification.
\newblock \emph{arXiv preprint arXiv:2310.10359}, 2023.

\bibitem[Kalyanakrishnan et~al.(2012)Kalyanakrishnan, Tewari, Auer, and Stone]{kalyanakrishnan2012pac}
Shivaram Kalyanakrishnan, Ambuj Tewari, Peter Auer, and Peter Stone.
\newblock Pac subset selection in stochastic multi-armed bandits.
\newblock In \emph{ICML}, volume~12, pages 655--662, 2012.

\bibitem[Kano et~al.(2017)Kano, Honda, Sakamaki, Matsuura, and Sugiyama]{kano2017Good}
Hideaki Kano, Junya Honda, Kentaro Sakamaki, Kentaro Matsuura, and Masashi Sugiyama.
\newblock Good arm identification via bandit feedback.
\newblock \emph{Machine Learning}, \penalty0 (2), 2017.

\bibitem[Karnin et~al.(2013)Karnin, Koren, and Somekh]{karnin2013almost}
Zohar Karnin, Tomer Koren, and Oren Somekh.
\newblock Almost optimal exploration in multi-armed bandits.
\newblock In \emph{International conference on machine learning}, pages 1238--1246. PMLR, 2013.

\bibitem[Katz-Samuels and Jamieson(2020)]{katz2020true}
Julian Katz-Samuels and Kevin Jamieson.
\newblock The true sample complexity of identifying good arms.
\newblock In \emph{International Conference on Artificial Intelligence and Statistics}, pages 1781--1791. PMLR, 2020.

\bibitem[Kaufmann and Kalyanakrishnan(2013)]{kaufmann2013information}
Emilie Kaufmann and Shivaram Kalyanakrishnan.
\newblock Information complexity in bandit subset selection.
\newblock In \emph{Conference on Learning Theory}, pages 228--251. PMLR, 2013.

\bibitem[Kaufmann et~al.(2016)Kaufmann, Capp{\'e}, and Garivier]{kaufmann2016complexity}
Emilie Kaufmann, Olivier Capp{\'e}, and Aur{\'e}lien Garivier.
\newblock On the complexity of best arm identification in multi-armed bandit models.
\newblock \emph{Journal of Machine Learning Research}, 17:\penalty0 1--42, 2016.

\bibitem[Li and Cheung(2025)]{pmlr-v267-li25f}
Zitian Li and Wang~Chi Cheung.
\newblock Near optimal non-asymptotic sample complexity of 1-identification.
\newblock In Aarti Singh, Maryam Fazel, Daniel Hsu, Simon Lacoste-Julien, Felix Berkenkamp, Tegan Maharaj, Kiri Wagstaff, and Jerry Zhu, editors, \emph{Proceedings of the 42nd International Conference on Machine Learning}, volume 267 of \emph{Proceedings of Machine Learning Research}, pages 34144--34184. PMLR, 13--19 Jul 2025.
\newblock URL \url{https://proceedings.mlr.press/v267/li25f.html}.

\bibitem[Locatelli et~al.(2016)Locatelli, Gutzeit, and Carpentier]{locatelli2016optimal}
Andrea Locatelli, Maurilio Gutzeit, and Alexandra Carpentier.
\newblock An optimal algorithm for the thresholding bandit problem.
\newblock In \emph{International Conference on Machine Learning}, pages 1690--1698. PMLR, 2016.

\bibitem[Mason et~al.(2020)Mason, Jain, Tripathy, and Nowak]{mason2020finding}
Blake Mason, Lalit Jain, Ardhendu Tripathy, and Robert Nowak.
\newblock Finding all epsilon-good arms in stochastic bandits.
\newblock \emph{Advances in Neural Information Processing Systems}, 33:\penalty0 20707--20718, 2020.

\bibitem[Mukherjee et~al.(2017)Mukherjee, Naveen, Sudarsanam, and Ravindran]{mukherjee2017thresholding}
Subhojyoti Mukherjee, Kolar~Purushothama Naveen, Nandan Sudarsanam, and Balaraman Ravindran.
\newblock Thresholding bandits with augmented ucb.
\newblock \emph{arXiv preprint arXiv:1704.02281}, 2017.

\bibitem[Poiani et~al.(2025)Poiani, Bernasconi, and Celli]{poiani2025nonasymptoticanalysisstickytrackandstop}
Riccardo Poiani, Martino Bernasconi, and Andrea Celli.
\newblock Non-asymptotic analysis of (sticky) track-and-stop, 2025.
\newblock URL \url{https://arxiv.org/abs/2505.22475}.

\bibitem[Simchowitz et~al.(2017)Simchowitz, Jamieson, and Recht]{simchowitz2017simulator}
Max Simchowitz, Kevin Jamieson, and Benjamin Recht.
\newblock The simulator: Understanding adaptive sampling in the moderate-confidence regime.
\newblock In \emph{Conference on Learning Theory}, pages 1794--1834. PMLR, 2017.

\bibitem[Tsai et~al.(2024)Tsai, Tsai, and Lin]{tsai2024lil}
Tzu-Hsien Tsai, Yun-Da Tsai, and Shou-De Lin.
\newblock lil’hdoc: an algorithm for good arm identification under small threshold gap.
\newblock In \emph{Pacific-Asia Conference on Knowledge Discovery and Data Mining}, pages 78--89. Springer, 2024.

\end{thebibliography}

\appendix

\section{Detailed Literature Review}
\label{sec:Detailed-Literature-Review}

\subsection{Simplification of History Result}
In this section, we only consider Gaussian reward with unit variance. We discuss more details of the upper bounds presented in Table \ref{table:Comparison-of-bounds}. We simplify the result in the history literature with smaller but explicit formulations, for the comparison in Table \ref{table:Comparison-of-bounds}.

\subsubsection{Theorem 2 in \cite{poiani2025nonasymptoticanalysisstickytrackandstop}}
Theorem 2 in \cite{poiani2025nonasymptoticanalysisstickytrackandstop} presents the Non-asymptotic analysis for algorithm S-TaS. For an instance $\nu$ with $r_a^{\nu}=\mu_a$, $\mu_1\geq \cdots \geq \mu_m > \mu_0 \geq \mu_{m+1} \geq \cdots \geq \mu_K$ and , they define $\neg a =\{\nu': r_a^{\nu'} < \mu_0\}$ for $a\in [m]$, and 
\begin{align*}
    i_F(\nu) = \arg\max_{a\in [K]}\sup_{w\in \Delta_K} \inf_{\nu'\in \neg a}\sum_{a'\in K} \frac{w_{a'}(r_{a'}^{\nu'}-r_{a'}^{\nu})^2}{2}.
\end{align*}
Then, Theorem 2 in \cite{poiani2025nonasymptoticanalysisstickytrackandstop} considers $\epsilon_{\nu}$ such that for all instance $\nu'$ such that
\begin{align*}
    \|\{r_a^{\nu'}\}_{a=1}^K-\{r_a^{\nu}\}_{a=1}^K\|_{\infty}\leq \epsilon_{\nu}\implies i_F(\nu')\subset i_F(\nu)\cup ([K]\setminus i^*(\nu)).
\end{align*}
Here $i^*(\nu)=[m]$. Let
\begin{align*}
    T_{\nu}=\max\left\{10K^4, \inf\left\{n\in\mathbb{N}:\sqrt{\frac{8D_K\log n}{\sqrt{\sqrt{n}+K^2}-2K}}\leq \epsilon_{\nu}\right\}\right\}
\end{align*}
where $D_K\geq 10$. The upper bound for $\mathbb{E}_{\nu}\tau$ is
\begin{align*}
    \mathbb{E}_{\nu}\tau \leq 10K^4 + T_0(\delta) + \frac{\pi^2}{24}
\end{align*}
where
\begin{align*}
    T_0(\delta)=\inf\left\{t\in\mathbb{N}: \log(Ct^2/\delta) \leq (t-\sqrt{t}-1-T_{\nu})\frac{\Delta_{0, 1}^2}{2}-g(t)\right\}
\end{align*}
and $C\geq e\sum_{t=1}^{+\infty}(\frac{e}{K})^K\frac{(\log^2(Ct^2)\log(t))^K}{t^2}$, $g(t)=\tilde{O}(K^2\sqrt{D_Kt}+\sqrt{D_K t^{\frac{3}{2}}})$.

To simplify the above result, we can prove $\frac{2\log(1/\delta)}{\Delta_{0, 1}^2} + \frac{K^2}{\Delta_{0, 1}^4}$ is a lower bound of $10K^4 + T_0(\delta) + \frac{\pi^2}{24}$. From the definition of $T_{\nu}$, we know
\begin{align*}
    T_{\nu} \geq & \max\left\{10K^4, \inf\left\{n\in\mathbb{N}:\sqrt{\frac{8D_K\log n}{\sqrt{\sqrt{n}+K^2}-2K}}\leq \epsilon_{\nu}\right\}\right\}\\
    \geq & \max\left\{10K^4, \inf\left\{n\in\mathbb{N}: \sqrt{\frac{8D_K}{\sqrt{\sqrt{n}+K^2}-2K}}\leq \epsilon_{\nu}\right\}\right\}\\
    = & \max\left\{10K^4, \inf\left\{n\in\mathbb{N}: \frac{80}{\epsilon_{\nu}^2}+2K\leq \sqrt{\sqrt{n}+K^2}\right\}\right\}\\
    \geq & \max\left\{10K^4, \inf\left\{n\in\mathbb{N}: \frac{320K}{\epsilon_{\nu}^2}+4K^2\leq \sqrt{n}+K^2\right\}\right\}\\
    \geq & \max\left\{10K^4,  \frac{90000K^2}{\epsilon_{\nu}^4}\right\}.
\end{align*}
We can further derive
\begin{align*}
    T_0(\delta) \geq & \inf\left\{t\in\mathbb{N}: \log(C/\delta) \leq \left(t-\max\left\{10K^4,  \frac{90000K^2}{\epsilon_{\nu}^4}\right\}\right)\frac{\Delta_{0, 1}^2}{2}\right\}\\
    = & \frac{2\log(C/\delta)}{\Delta_{0, 1}^2} + \max\left\{10K^4,  \frac{90000K^2}{\epsilon_{\nu}^4}\right\}.
\end{align*}
Thus, we can conclude $\frac{2\log(1/\delta)}{\Delta_{0, 1}^2} + \frac{K^2}{\epsilon_{\nu}^4}$ is smaller than the proposed upper bound in \cite{poiani2025nonasymptoticanalysisstickytrackandstop}. The remaining work is to show $\epsilon_{\nu} \leq \Delta_{0,1}$. Denote $a^{\text{sub}}=\arg\max_{a:\mu_a < \mu_1} \mu_a$. We have
\begin{itemize}
    \item If $\mu_{a^{\text{sub}}} \geq \mu_0$, we assert $\epsilon_{\nu}\leq \Delta_{1,a^{\text{sub}}} < \Delta_{1,0}$. If not, we can set $\nu'$ with $r_{a^{\text{sub}}}^{\nu'}=\mu_1$ while keeping $r_a^{\nu'}=r_a^{\nu'}$ for $a\neq a^{\text{sub}}$. With this construction, we have $a^{\text{sub}}\in i_F(\nu'), a^{\text{sub}}\notin \{1, m+1,m+2,\cdots, K\}$, contradicting to the requirement of $\epsilon_{\nu}$.
    \item If $\mu_{a^{\text{sub}}} < \mu_0$, we assert $\epsilon_{\nu}\leq  \Delta_{1,0}$. If $\epsilon_{\nu}>  \Delta_{1,0}$, we can take $\nu'$ with $r_a^{\nu'}=\begin{cases}\mu_1-\epsilon_{\nu} & \text{if }\mu_a = \mu_1\\ \mu_a & \text{if }\mu_a < \mu_1\end{cases}$ for $a\in [K]$. Easy to validate $\|\{r_a^{\nu'}\}_{a=1}^K-\{r_a^{\nu}\}_{a=1}^K\|_{\infty}\leq \epsilon_{\nu}$ and $i_F(\nu')=\textsf{None} \notin i_F(\nu)\cup ([K]\setminus i^*(\nu))$
\end{itemize}
Thus, we have shown $\frac{2\log(1/\delta)}{\Delta_{0, 1}^2} + \frac{K^2}{\Delta_{0, 1}^4}$ is a lower bound of $10K^4 + T_0(\delta) + \frac{\pi^2}{24}$.

\subsubsection{Theorem 3 in \cite{kano2017Good}}
Following Appendix A.2 in \cite{pmlr-v267-li25f}, we can also simplify the Theorem 3 in \cite{kano2017Good}. By applying the algorithm HDoC, for $\nu\in\mathcal{S}^{\text{pos}}$, we have
\begin{align*}
    \mathbb{E}_{\nu}\tau\leq n_1 + \sum_{a=2}^K\left(\frac{\log(K \max_{j\in[K]} n_j)}{2(\Delta_{1,a}-2\epsilon)^2}+\delta n_a\right) + \frac{K^{2-\frac{\epsilon^2}{(\min_{a\in [K]} \Delta_{0,a}-\epsilon)^2}}}{2\epsilon^2} + \frac{K(5+\log\frac{1}{2\epsilon^2})}{4\epsilon^2},
\end{align*}
where $n_a=\frac{1}{(\Delta_{0,a}-\epsilon)^2} \log\left(\frac{4\sqrt{K/\delta}}{(\Delta_{0,a}-\epsilon)^2}\log\frac{5\sqrt{K/\delta}}{(\Delta_{0,a}-\epsilon)^2}\right)$, $\epsilon=\min\Big\{\min\limits_{a\in[K]}\Delta_{0, a}, \min\limits_{a\in[K-1]}\frac{\Delta_{a, a+1}}{2}\Big\}$. From direct calculation, we can observe
\begin{align*}
    & n_1 \geq \frac{\log\frac{K}{\delta}}{2\Delta_{0,1}^2},\quad \frac{K(5+\log\frac{1}{2\epsilon^2})}{4\epsilon^2}\geq \frac{K}{4\epsilon^2}\\
    & \log(K \max_{j\in[K]} n_j) \geq \log (\frac{K}{2}\log\frac{1}{\delta}) \geq \log\log\frac{1}{\delta}.
\end{align*}
From the above inequality, we can conclude
\begin{align*}
    & \sum_{a=2}^K\frac{\log(K \max_{j\in[K]} n_j)}{2(\Delta_{1,a}-2\epsilon)^2} \\
    \geq & \sum_{a=2}^K \frac{\log\log\frac{1}{\delta}}{2(\Delta_{1,a}-2\epsilon)^2}\\
    \geq & \sum_{a=2}^K \frac{\log\log\frac{1}{\delta}}{2\Delta_{1,a}^2}\\
    \geq & \frac{H_1}{2}\log\log\frac{1}{\delta}.
\end{align*}
Thus, we can conclude 
\begin{align*}
     & n_1 + \sum_{a=2}^K\left(\frac{\log(K \max_{j\in[K]} n_j)}{2(\Delta_{1,a}-2\epsilon)^2}+\delta n_a\right) + \frac{K^{2-\frac{\epsilon^2}{(\min_{a\in [K]} \Delta_{0,a}-\epsilon)^2}}}{2\epsilon^2} + \frac{K(5+\log\frac{1}{2\epsilon^2})}{4\epsilon^2}\\
     \geq & \frac{\log\frac{K}{\delta}}{2\Delta_{0,1}^2}+\frac{H_1}{2}\log\log\frac{1}{\delta} + \frac{K}{4\epsilon^2}.
\end{align*}


\subsubsection{Thereom 2 in \cite{jourdan2023anytime}}
\cite{jourdan2023anytime} do not explicitly provide upper bound for $\mathbb{E}\tau$. Therorem 2 in \cite{jourdan2023anytime} suggests that for a positive instance $\nu$.
\begin{align*}
    \mathbb{E}_{\nu}\tau & \leq C^{\text{pos}}(\delta) + \frac{K\pi^2}{6} + 1
\end{align*}
where 
\begin{align*}
    C^{\text{pos}}(\delta) = & \sup\left\{t: t\leq \left(\sum_{a:\mu_a\geq \mu_0}\frac{2}{\Delta_{0,a}^2}\right) (\sqrt{c(t,\delta)}+\sqrt{3\log t})^2 + D^{\text{pos}}(\{\mu_a\}_{a=1}^K)\right\}\\
    2c(t,\delta) = & \bar{W}_{-1}\left(2\log \frac{K}{\delta}+4\log\log (e^4t) + \frac{1}{2}\right)\\
    \bar{W}_{-1}(x) = & -W_{-1}(-e^{-x})\approx x+\log x.
\end{align*}
$W_{-1}$ is the negative branch of the Lambert W function, and $D^{\text{pos}}(\{\mu_a\}_{a=1}^K)$ is further defined in the Appendix F in \cite{jourdan2023anytime}, as
\begin{align*}
    D^{\text{pos}}(\{\mu_a\}_{a=1}^K) = & S_{\mu}+1-\frac{6\log S_{\mu}}{\Delta_{0,1}^2}\\
    S_{\mu} = & h_1\left(12\sum_{a=1}^K\frac{1}{\Delta_{0,a}^2}, n_0K+2m\right)\\
    h_1(z,y) = & z\bar{W}_{-1}(\frac{y}{z}+\log z)\\
    \bar{W}_{-1}(x) = & -W_{-1}(-e^{-x})
\end{align*}
From the definition, we know
\begin{align*}
    C^{\text{pos}}(\delta) \geq & \sup\left\{t: t\leq \left(\sum_{a:\mu_a\geq \mu_0}\frac{2}{\Delta_{0,a}^2}\right)c(t,\delta)\right\}\\
    \geq & \sup\left\{t: t\leq \left(\sum_{a:\mu_a\geq \mu_0}\frac{2}{\Delta_{0,a}^2}\right)\left(\log \frac{K}{\delta}+2\log\log (e^4t) + \frac{1}{4}\right)\right\}\\
    \geq & \sup\left\{t: t\leq \left(\sum_{a:\mu_a\geq \mu_0}\frac{2}{\Delta_{0,a}^2}\right)\log \frac{K}{\delta}+2\log4\right\}\\
    \geq & \left(\sum_{a:\mu_a\geq \mu_0}\frac{2}{\Delta_{0,a}^2}\right)\log \frac{K}{\delta}.
\end{align*}
Thus, we can conclude $\left(\sum_{a:\mu_a\geq \mu_0}\frac{2}{\Delta_{0,a}^2}\right)\log \frac{K}{\delta}$ is a lower bound of $C^{\text{pos}}(\delta) + \frac{K\pi^2}{6} + 1$.

\subsubsection{Theorem 8 in \cite{katz2020true}}
As discussed in the Appendix A.1 in \cite{pmlr-v267-li25f}, proposed Algorithm FWER-FWPD only work on the positive instance with a different definition of $\mathbb{E}\tau$. But we can adapt their result. Apply Algorithm FWER-FWPD on a positive instance $\nu$, Theorem 8 in \cite{katz2020true} suggests that
\begin{align*}
    \mathbb{E}\tau\leq O\left(\left((\frac{K}{m}-1)\cdot \frac{\log\frac{1}{\delta}}{\Delta^2}+\log\frac{\frac{K}{m}}{\delta}\right)\log\left((\frac{K}{m}-1)\cdot \frac{\log\frac{1}{\delta}}{\Delta^2}+\log\frac{\frac{K}{m}}{\delta}\right)\right),
\end{align*}
where $\Delta = \min_{a:\mu_a > \mu_0} \Delta_{a,0}$. 

Not hard to see
\begin{align*}
    & \left((\frac{K}{m}-1)\cdot \frac{\log\frac{1}{\delta}}{\Delta^2}+\log\frac{\frac{K}{m}}{\delta}\right)\log\left((\frac{K}{m}-1)\cdot \frac{\log\frac{1}{\delta}}{\Delta^2}+\log\frac{\frac{K}{m}}{\delta}\right)\\
    \geq & \left((\frac{K}{m}-1)\cdot \frac{\log\frac{1}{\delta}}{\Delta^2}\right)\log\left((\frac{K}{m}-1)\cdot \frac{\log\frac{1}{\delta}}{\Delta^2}\right).
\end{align*}

\subsection{Notation Description}

Table \ref{tab:notation} presents part of the notation description used in the following appendix.
\begin{longtable}{p{3cm}lp{8cm}}
\caption{Notation Summary}
\label{tab:notation}\\
\hline
\textbf{Symbol} & \textbf{Type} & \textbf{Meaning} \\
\hline
\endfirsthead

\hline
\textbf{Symbol} & \textbf{Type} & \textbf{Meaning} \\
\hline
\endhead

\hline
\endfoot

\hline
\endlastfoot


$K$ & Integer & Number of arms. \\
$a,i,j$ & Integer in $[K]$ & Arm Index \\
$\delta$ & Real & Tolerance level \\
$\hat{a}$ & NA & Output of Algorithm \ref{alg:Parallel-SEE-on-Bracket}, \ref{alg:SEE-with-Bracket}, takes value in $[K]\cup\{\textsf{None}\}$\\
$\hat{a}_k$ & NA & Rough guess of an exploration period in Algorithm \ref{alg:SEE-with-Bracket}, takes value in $[K]\cup\{\textsf{None},\textsf{Not Complete}\}$\\
$k$ & Integer & Phase index in Algorithm \ref{alg:SEE-with-Bracket}\\
$\nu,\nu',\nu'',\tilde{\nu}$ & NA & An 1-identification instance \\
$\rho_a$ & NA & Reward distribution of arm $a$ \\
$\sigma(\nu)$ & NA & Switch the arm positions in instance $\nu$ with a permutation $\sigma$\\
$\mu_a$ & Real & Mean value of arm $a$. \\
$m$ & Integer & Maximum arm index whose mean value is above $\mu_0$, we usually assume $\mu_1\geq \cdots\geq\mu_m \geq\mu_0 > \mu_{m+1}\geq \cdots \geq \mu_K$\\
$r_a^{\nu}$ & Real & Mean value of arm $a$ instance $\nu$. \\
$\Delta_{i,j}$ & Real & Mean gap between arm index $i$ and $j$.  \\
$A_t$ & Integer in $[K]$ & Action in round $t$\\
$B$ & Subset of $[K]$ & A bracket\\
$b,\tilde{b}$ & Integer & Bracket index in $[\lceil \log_2 K\rceil+1]$ \\
$\tau, \tau(B)$ & Random Variable & $\tau$ is the termination round of Algorithm \ref{alg:Parallel-SEE-on-Bracket}, $\tau(B)$ is the termination round of an Algorithm \ref{alg:SEE-with-Bracket}\\
$\tau_k^{\text{ee}}(B), \tau_k^{\text{et}}(B)$ & Random Variable & The total length of the exploration / exploitation periods at to the end of phase $k$, of Algorithm \ref{alg:SEE-with-Bracket} with bracket $B$ as input\\
$\tau, \tau_n, \tilde{\tau}, \tau', \tau_n^j(a)$ & Random Variable & Stopping times for analysis\\
$\mathcal{E}$ & Random Event & Random Event\\
$Z_s, Z_t$ & Random Variable & Realized KL divergence, usually follows $N(\frac{\Delta^2}{2}, \Delta^2)$ as we only consider Gaussian distribution with variance $1$ as reward distribution in the proof lower bound \\
$D, D_i$ & Random Variable & $D:=\max_{1\leq t\leq n}\sum_{s=1}^{t}\left(Z_s-\frac{\Delta^2}{2}\right)$. See the proof of Lemma \ref{lemma:inequality-for-kl-diver-of-conditional-prob}\\
$N_a(t)$ & Random Variable & Pulling times of arm $a$ up to round $t$\\
$T_j(a)$ & Real Value & A pulling times threshold dependent on arm $a$ and $j$\\
$U(t,\delta)$ & Real Value & Half length of confidence interval with pulling times $t$\\
$\mathcal{H}^{ee}, \mathcal{H}^{et}, Q$ & Random Set & Collected history in Algorithm \ref{alg:SEE-with-Bracket}\\
$N_a^{\text{ee}}(\mathcal{H}^{ee}), N_a^{\text{et}}(\mathcal{H}^{et})$ & Random Variable & Pulling times of arm $a$ in history $\mathcal{H}^{ee}, \mathcal{H}^{et}$\\
$N_a^{\text{ee}}(t)$ & Random Variable & Pulling times of arm $a$ at round $t$ of an exploration period\\
$N_a^{\text{et}}(t)$ & Random Variable & Pulling times of arm $a$ at round $t$ of an exploitation period\\
$\hat{\mu}_a^{\text{ee}}(\mathcal{H}^{ee}),  \hat{\mu}_a^{\text{et}}(\mathcal{H}^{et})$ & Random Variable & Empirical mean value of arm $a$ in history $\mathcal{H}^{ee}, \mathcal{H}^{et}$\\
$\text{UCB}^{\text{ee}}_a(\mathcal{H}^{ee},\delta)$ & Random Variable & Upper Confidence Bound calculated from history $\mathcal{H}^{ee}$\\ 
$\text{LCB}^{\text{ee}}_a(\mathcal{H}^{ee},\delta)$ & Random Variable & Lower Confidence Bound calculated from history $\mathcal{H}^{ee}$\\ 
$\text{UCB}^{\text{et}}_a(\mathcal{H}^{et},\delta)$ & Random Variable & Upper Confidence Bound calculated from history $\mathcal{H}^{et}$\\ 
$\text{LCB}^{\text{et}}_a(\mathcal{H}^{et},\delta)$, & Random Variable & Lower Confidence Bound calculated from history $\mathcal{H}^{et}$\\ 
$C$ & Real Value & Fixed Constant In Algorithm \ref{alg:SEE-with-Bracket}, we take $C$ to any real number greater than 1\\
$t^{\text{ee}},t^{\text{et}}$ & Integer & Counter of the total pulling times in exploration and exploitation periods, used in Algorithm \ref{alg:SEE-with-Bracket}\\
$\delta_k, \beta_k,\alpha_k$ & Real Value & Hyperparameters of Algorithm \ref{alg:SEE-with-Bracket} to control the pulling process\\
$X_{a,s}^{\text{ee}}$ & Random Variable & Realized reward by pulling arm $a$ at the $s$-th times in the exploration period \\
$X_{a,s}^{\text{et}}$ & Random Variable & Realized reward by pulling arm $a$ at the $s$-th times in the exploitation period \\
$\kappa^{\text{ee}}_b, \kappa^{\text{ee}}_B$ & Random Variable & Minimum phase index such that the concentration event holds in the exploration period of algorithm copy $b$. If we only discuss a given bracket $B$, we do not emphasize its bracket index and may use $\kappa^{\text{ee}}_B$.\\
$\kappa^{\text{et}}_b, \kappa^{\text{et}}_B$ & Random Variable & Minimum phase index such that the concentration event holds in the exploitation period of algorithm copy $b$. If we only discuss a given bracket $B$, we do not emphasize its bracket index and may use $\kappa^{\text{et}}_B$.\\
$\hat{\mu}^{ee}_{a,b}(t)$ & Random Variable & The empirical mean value of arm $a$ of the first $t$ collected reward in the exploration period of algorithm copy $b$\\
$\hat{\mu}^{et}_{a,b}(t)$ & Random Variable & The empirical mean value of arm $a$ of the first $t$ collected reward in the exploitation period of algorithm copy $b$\\
$\hat{\mu}_{a,t}$ & Random Variable & The empirical mean value of arm $a$ of the first $t$ collected reward. We use this notation if a specific bracket and a specific exploration/exploitation period is given, we may turn to use $\hat{\mu}_{a,t}$ instead of $\hat{\mu}^{ee}_{a,b}(t),\hat{\mu}^{et}_{a,b}(t)$.\\
\end{longtable}

\newpage

\section{Proof of Lower Bound}
\label{sec:Proof-of-lower-bound-full}

\subsection{Proof of Preparation}
\label{sec:Proof-of-Preparation-lower-bound}
The first step is to prove Lemma \ref{lemma:property-of-symmetric-algorithm}, showing that we suffice restrict attention to symmetric algorithms.
\begin{proof}[Proof of Lemma \ref{lemma:property-of-symmetric-algorithm}]

    Let $S$ be the set of all possible permutations of $[K]$.

    Given the alg, we define $\widetilde{\text{alg}}$ as the following. Shuffle all the arms randomly before the pulling process, and then apply the algorithm alg. We firstly show $\widetilde{\text{alg}}$ is indeed symmetric. For any pulling sequence $\{a_t\}_{t=1}^T\in [K]^T$ and any permutation $\sigma$, we have
    \begin{align*}
        \Pr_{\widetilde{\text{alg}}, \sigma(\nu')}(\forall t\in [T], A_t=\sigma(a), \tau > t)=\frac{\sum_{\tilde{\sigma}\in S}\Pr_{\text{alg}, \tilde{\sigma}(\nu')}(\forall t\in [T], A_t=\tilde{\sigma}(a), \tau > t)}{K!},
    \end{align*}
    which means $\widetilde{\textnormal{alg}}$ fits in the definition. Then, we have
    \begin{align*}
        & \max_{\sigma\in S} \mathbb{E}_{\text{alg}, \sigma(\nu')} \tau \\
        \geq & \frac{\sum_{\sigma\in S} \mathbb{E}_{\text{alg}, \sigma(\nu')} \tau}{K!}\\
        \geq & \frac{\sum_{\sigma\in S} \sum_{t=1}^{+\infty}\Pr_{\text{alg}, \sigma(\nu')}( \tau\geq t)}{K!}\\
        = & \mathbb{E}_{\widetilde{\text{alg}}, \nu'} \tau
    \end{align*}
    The last thing is to show $\widetilde{\text{alg}}$ is also $\delta$-PAC. Since alg is $\delta$-PAC, we know for any permutation $\sigma$, we have
    \begin{align*}
        \Pr_{\text{alg}, \sigma(\nu)}(\tau < +\infty,\hat{a}\in i^*(\sigma(\nu))) > 1-\delta.
    \end{align*}
    That means
    \begin{align*}
        & \Pr_{\widetilde{\text{alg}}, \sigma(\nu)}(\tau < +\infty,\hat{a}\in i^*(\sigma(\nu)))\\
        = & \frac{\sum_{\sigma\in S}\Pr_{\text{alg}, \sigma(\nu)}(\tau < +\infty,\hat{a}\in i^*(\sigma(\nu)))}{K!}\\
        > & \frac{\sum_{\sigma\in S}1-\delta}{K!}\\
        = & 1-\delta.
    \end{align*}
    We can conclude $\widetilde{\text{alg}}$ is also $\delta$-PAC.
\end{proof}

The next step establishes a technical change-of-measure bound that controls a conditional KL divergence. This lemma is the main device used to argue that a symmetric algorithm cannot reliably distinguish two arms when the total number of samples from the pair is small.
\begin{lemma}
    \label{lemma:inequality-for-kl-diver-of-conditional-prob}
    Fix a pair $i,j\in [K], i\neq j$ and an arbitrary instance $\nu'$, denote $\Delta=|r^{\nu'}_i-r^{\nu'}_j|$.  
    Define
    \begin{align*}
        \tau_n = \min\{t: N_i(t)+N_j(t)=n, A_{t+1}\in \{i,j\}\}.
    \end{align*}
    Apply a symmetric algorithm, we have
    \begin{align*}
         & \Pr_{\nu'}(\tau_n<\tau, A_{\tau_n+1}=i)\log\frac{\Pr_{\nu'}(\tau_n<\tau, A_{\tau_n+1}=i)}{\Pr_{\nu'}(\tau_n<\tau, A_{\tau_n+1}=j)} + \\
         & \Pr_{\nu'}(\tau_n<\tau, A_{\tau_n+1}=j)\log\frac{\Pr_{\nu'}(\tau_n<\tau, A_{\tau_n+1}=j)}{\Pr_{\nu'}(\tau_n<\tau, A_{\tau_n+1}=i)}\\
        \leq & \sqrt{2n\Delta^2\log\frac{1}{\Pr_{\nu'}(\tau_n<\tau)}}\Pr_{\nu'}(\tau_n<\tau)+\sqrt{2\pi n\Delta^2}\Pr_{\nu'}(\tau_n<\tau) + \frac{1}{2} n\Delta^2\Pr_{\nu'}(\tau_n<\tau).
    \end{align*}
\end{lemma}
An observation is given a fixed sample path and a symmetric algorithm, the set $\{t: N_i(t)+N_j(t)=n, A_{t+1}\in \{i,j\}\}=: S_{ij}(n)$ contains at most one element. If $|S_{ij}(n)|=1$, $\tau_n$ will be the unique element, and $\tau_n<\tau$ holds; if $|S_{ij}(n)|=0$, then $\tau_n=\infty$ and hence $\tau_n \geq \tau$. Moreover,
\begin{small}
    \begin{align*}
        \frac{\Pr_{\nu'}(\tau_n<\tau, A_{\tau_n+1}=i)\log\frac{\Pr_{\nu'}(\tau_n<\tau, A_{\tau_n+1}=i)}{\Pr_{\nu'}(\tau_n<\tau, A_{\tau_n+1}=j)} + \Pr_{\nu'}(\tau_n<\tau, A_{\tau_n+1}=j)\log\frac{\Pr_{\nu'}(\tau_n<\tau, A_{\tau_n+1}=j)}{\Pr_{\nu'}(\tau_n<\tau, A_{\tau_n+1}=i)}}{\Pr_{\nu'}(\tau_n<\tau)}
    \end{align*}
\end{small}
is the kl-divergence between two Bernoulli distribution with probability $\Pr_{\nu'}(A_{\tau_n+1}=i | \tau_n<\tau)$, $\Pr_{\nu'}(A_{\tau_n+1}=j | \tau_n<\tau)$. From the definition of $\tau_n$, not hard to see $\Pr_{\nu'}(A_{\tau_n+1}=i | \tau_n<\tau) + \Pr_{\nu'}(A_{\tau_n+1}=j | \tau_n<\tau)=1$.
\begin{proof}[Proof of Lemma \ref{lemma:inequality-for-kl-diver-of-conditional-prob}]

    Denote $\nu''$ as the instance which switch the position of arm $i$ and arm $j$ in $\nu'$, i.e.
    $r_a^{\nu''}=\begin{cases}r_j^{\nu'}&a=i\\ r_i^{\nu'}&a=j\\r_a^{\nu'} & \text{else} \end{cases}$. We first derive
    \begin{small}
        \begin{align}
            \Pr_{\nu'}(\tau_n<\tau, A_{\tau_n+1}=i)\log\frac{\Pr_{\nu'}(\tau_n<\tau, A_{\tau_n+1}=i)}{\Pr_{\nu''}(\tau_n<\tau, A_{\tau_n+1}=i)} \leq \mathbb{E}_{\nu'}\mathds{1}(\tau_n<\tau, A_{\tau_n+1}=i) \left(\sum_{s=1}^{N_i(\tau_n) + N_j(\tau_n)}Z_s\right)\label{eqn:A_tau_n+1=i}\\
            \Pr_{\nu'}(\tau_n<\tau, A_{\tau_n+1}=j)\log\frac{\Pr_{\nu'}(\tau_n<\tau, A_{\tau_n+1}=j)}{\Pr_{\nu''}(\tau_n<\tau, A_{\tau_n+1}=j)} \leq \mathbb{E}_{\nu'}\mathds{1}(\tau_n<\tau, A_{\tau_n+1}=j) \left(\sum_{s=1}^{N_i(\tau_n) + N_j(\tau_n)}Z_s\right)\label{eqn:A_tau_n+1=j}
        \end{align}
    \end{small}
    where $Z_s | Z_1,Z_2,\cdots, Z_{s-1}\sim N(\frac{\Delta^2}{2}, \Delta^2)$. Since these two inequalities are symmetric, we suffice to prove the first (\ref{eqn:A_tau_n+1=i}). Take $\mathcal{E} = \{\tau_n<\tau, A_{\tau_n+1}=i\} $. Notice that $\mathcal{E}$ is $\tau_n$-measurable, as both $\{\tau_n<\tau\}, \{A_{\tau_n+1}=i\}$ are $\tau_n$-measurable. From the Lemma 18 in \cite{kaufmann2016complexity},
    \begin{align*}
        \Pr_{\nu''}(\mathcal{E}) = & \mathbb{E}_{\nu'}\mathds{1}(\mathcal{E})\exp\left(-\sum_{s=1}^{N_i(\tau_n)}\hat{KL}_{i,s}-\sum_{s=1}^{N_j(\tau_n)}\hat{KL}_{j,s}\right)\\
        = & \mathbb{E}_{\nu'}\left[\mathds{1}(\mathcal{E})\exp\left(-\sum_{s=1}^{N_i(\tau_n)}\hat{KL}_{i,s}-\sum_{s=1}^{N_j(\tau_n)}\hat{KL}_{j,s}\right)|\mathcal{E}\right]\Pr_{\nu'}(\mathcal{E}) + \\
        & \mathbb{E}_{\nu'}\left[\mathds{1}(\mathcal{E})\exp\left(-\sum_{s=1}^{N_i(\tau_n)}\hat{KL}_{i,s}-\sum_{s=1}^{N_j(\tau_n)}\hat{KL}_{j,s}\right)|\neg\mathcal{E}\right]\Pr_{\nu'}(\neg\mathcal{E})\\
        = & \mathbb{E}_{\nu'}\left[\exp\left(-\sum_{s=1}^{N_i(\tau_n)}\hat{KL}_{i,s}-\sum_{s=1}^{N_j(\tau_n)}\hat{KL}_{j,s}\right)|\mathcal{E}\right]\Pr_{\nu'}(\mathcal{E})\\
        \geq & \exp\left(\mathbb{E}_{\nu'}\left[-\sum_{s=1}^{N_i(\tau_n)}\hat{KL}_{i,s}-\sum_{s=1}^{N_j(\tau_n)}\hat{KL}_{j,s}|\mathcal{E}\right]\right)\Pr_{\nu'}(\mathcal{E}),
    \end{align*}
    where $\hat{KL}_{i,s}:=\log\frac{\exp(-\frac{(X_{i,s}-r_i^{\nu'})^2}{2})}{\exp(-\frac{(X_{i,s}-r_i^{\nu''})^2}{2})}=\frac{(2X_{i,s}-r_i^{\nu'}-r_j^{\nu'})(r_i^{\nu'}-r_j^{\nu'})}{2}$, $X_{i,s}\sim N(r_i^{\nu'}, 1)$, $\hat{KL}_{j,s}:=\log\frac{\exp(-\frac{(X_{j,s}-r_j^{\nu'})^2}{2})}{\exp(-\frac{(X_{j,s}-r_j^{\nu''})^2}{2})}=\frac{(2X_{j,s}-r_j^{\nu'}-r_i^{\nu'})(r_j^{\nu'}-r_i^{\nu'})}{2}$, $X_{j,s}\sim N(r_j^{\nu'}, 1)$. The last line is by the Jensen's Inequality. Following the last line, we have
    \begin{align*}
        \log\frac{\Pr_{\nu'}(\mathcal{E})}{\Pr_{\nu''}(\mathcal{E})} \leq & \mathbb{E}_{\nu'}\left[\sum_{s=1}^{N_i(\tau_n)}\hat{KL}_{i,s}+\sum_{s=1}^{N_j(\tau_n)}\hat{KL}_{j,s}|\mathcal{E}\right]\\
        = & \frac{\mathbb{E}_{\nu'}\left[\mathds{1}(\mathcal{E})\left(\sum_{s=1}^{N_i(\tau_n)}\hat{KL}_{i,s}+\sum_{s=1}^{N_j(\tau_n)}\hat{KL}_{j,s}\right)\right]}{\Pr_{\nu'}(\mathcal{E})},
    \end{align*}
    and we can conclude $\Pr_{\nu'}(\mathcal{E})\log\frac{\Pr_{\nu'}(\mathcal{E})}{\Pr_{\nu''}(\mathcal{E})} \geq \mathbb{E}_{\nu'}\left[\mathds{1}(\mathcal{E})\left(\sum_{s=1}^{N_i(\tau_n)}\hat{KL}_{i,s}+\sum_{s=1}^{N_j(\tau_n)}\hat{KL}_{j,s}\right)\right]$. From the above calculation, we know $\hat{KL}_{i,s}, \hat{KL}_{j,s}$ both follow $N(\frac{(r_j^{\nu'}-r_i^{\nu'})^2}{2}, (r_j^{\nu'}-r_i^{\nu'}))$, and we can rewrite the conclusion as
    \begin{align*}
        \Pr_{\nu'}(\mathcal{E})\log\frac{\Pr_{\nu'}(\mathcal{E})}{\Pr_{\nu''}(\mathcal{E})} \leq \mathbb{E}_{\nu'}\left[\mathds{1}(\mathcal{E})\sum_{s=1}^{N_i(\tau_n) + N_j(\tau_n)}Z_s\right]
    \end{align*}
    where $Z_s | Z_1,Z_2,\cdots, Z_{s-1}\sim N(\frac{\Delta^2}{2}, \Delta^2 )$. The above inequality suggests that we have proved (\ref{eqn:A_tau_n+1=i}) and (\ref{eqn:A_tau_n+1=j}).
    
    By symmetry of the algorithm, we can conclude $\Pr_{\nu''}(\tau_n<\tau, A_{\tau_n+1}=i)=\Pr_{\nu'}(\tau_n<\tau, A_{\tau_n+1}=j)$ and $\Pr_{\nu''}(\tau_n<\tau, A_{\tau_n+1}=j)=\Pr_{\nu'}(\tau_n<\tau, A_{\tau_n+1}=i)$. Sum up inequality (\ref{eqn:A_tau_n+1=i}) and (\ref{eqn:A_tau_n+1=j}), we get
    \begin{align*}
        & \Pr_{\nu'}(\tau_n<\tau, A_{\tau_n+1}=i)\log\frac{\Pr_{\nu'}(\tau_n<\tau, A_{\tau_n+1}=i)}{\Pr_{\nu'}(\tau_n<\tau, A_{\tau_n+1}=j)} + \\
        & \Pr_{\nu'}(\tau_n<\tau, A_{\tau_n+1}=j)\log\frac{\Pr_{\nu'}(\tau_n<\tau, A_{\tau_n+1}=j)}{\Pr_{\nu'}(\tau_n<\tau, A_{\tau_n+1}=i)}\\
        = & \Pr_{\nu'}(\tau_n<\tau, A_{\tau_n+1}=i)\log\frac{\Pr_{\nu'}(\tau_n<\tau, A_{\tau_n+1}=i)}{\Pr_{\nu''}(\tau_n<\tau, A_{\tau_n+1}=i)} +\\
        & \Pr_{\nu'}(\tau_n<\tau, A_{\tau_n+1}=j)\log\frac{\Pr_{\nu'}(\tau_n<\tau, A_{\tau_n+1}=j)}{\Pr_{\nu''}(\tau_n<\tau, A_{\tau_n+1}=j)}\\
        \leq & \mathbb{E}_{\nu'}\mathds{1}(\tau_n<\tau) \left(\sum_{s=1}^{N_i(\tau_n) + N_j(\tau_n)}Z_s\right).
    \end{align*}
    The remaining work is to prove 
    \begin{align*}
        & \mathbb{E}_{\nu'}\mathds{1}(\tau_n<\tau) \left(\sum_{s=1}^{N_i(\tau_n) + N_j(\tau_n)}Z_s\right)\\
        \leq & \sqrt{2n\Delta^2\log\frac{1}{\Pr_{\nu'}(\tau_n<\tau)}}\Pr_{\nu'}(\tau_n<\tau)+\sqrt{2\pi n\Delta^2}\Pr_{\nu'}(\tau_n<\tau) + \frac{1}{2} n\Delta^2\Pr_{\nu'}(\tau_n<\tau).
    \end{align*}
    Let $D:=\max_{1\leq t\leq n}\sum_{s=1}^{t}\left(Z_s-\frac{\Delta^2}{2}\right)$. Since $\{\tau_n<\tau\}$ implies $N_i(\tau_n)+N_j(\tau_n)\le n$, we have
    \begin{align*}
        & \mathbb{E}_{\nu'}\mathds{1}(\tau_n<\tau) \left(\sum_{s=1}^{N_i(\tau_n) + N_j(\tau_n)}Z_s\right)\\
        \leq & \mathbb{E}_{\nu'}\mathds{1}(\tau_n<\tau) \left(\max_{1\leq t\leq n}\sum_{s=1}^{n}Z_s\right)\\
        = & \mathbb{E}_{\nu'}\mathds{1}(\tau_n<\tau) \left(\max_{1\leq t\leq n}\sum_{s=1}^{n}Z_s-\frac{\Delta^2}{2} + \frac{\Delta^2}{2}\right)\\
        \leq & \mathbb{E}_{\nu'}\mathds{1}(\tau_n<\tau) \left(\max_{1\leq t\leq n}\sum_{s=1}^{n}Z_s-\frac{\Delta^2}{2}\right) + \mathbb{E}_{\nu'}\mathds{1}(\tau_n<\tau)\frac{n\Delta^2}{2}\\
        = &  \mathbb{E}_{\nu'}\mathds{1}(\tau_n<\tau) D + \frac{n\Delta^2}{2}\Pr_{\nu'}(\tau_n<\tau).
    \end{align*}
    The remaining work is to prove $\mathbb{E}_{\nu'}\mathds{1}(\tau_n<\tau) D\leq \sqrt{2n\Delta^2\log\frac{1}{\Pr_{\nu'}(\tau_n<\tau)}}\Pr_{\nu'}(\tau_n<\tau)+\sqrt{2\pi n\Delta^2}\Pr_{\nu'}(\tau_n<\tau)$. To simplify the notation, we take event $\mathcal{E}_n=\{\tau_n<\tau\}$ here and use the layer-cake representation,
    \begin{align*}
        & \mathbb{E}_{\nu'}\left[\mathds{1}(\mathcal{E}_n) D\right]\\
        \leq & \mathbb{E}_{\nu'}\left[\mathds{1}(\mathcal{E}_n)\mathds{1}(D\geq 0) D\right]\\
        = & \int_{0}^{+\infty} \Pr_{\nu'} \Big(\mathds{1}(\mathcal{E}_n)\mathds{1}(D\geq 0) D > s\Big) ds\\
        \leq & \int_{0}^{+\infty} \Pr_{\nu'} \Big(\mathcal{E}_n, D > s\Big) ds\\
        \leq & \int_{0}^{+\infty} \min\left\{\Pr_{\nu'}(\mathcal{E}_n), \Pr_{\nu'}\Big( D > s\Big)\right\} ds
    \end{align*}
    From Lemma \ref{lemma:Gaussian-summation-submartingale}, we can conclude $\Pr_{\nu'}\big( D\geq s\big)\leq \exp(-\frac{s^2}{2n\Delta^2})$.
    Plugging this tail bound yields
    \begin{align*}
        \mathbb{E}_{\nu'}\left[\mathds{1}(\mathcal{E}_n) D\right] \leq & \int_{0}^{+\infty} \min\left\{\Pr_{\nu'}(\mathcal{E}_n), \exp(-\frac{s^2}{2n\Delta^2})\right\} ds\\
        = & \int_{0}^{\sqrt{2n\Delta^2\log\frac{1}{\Pr_{\nu'}(\mathcal{E}_n)}}} \Pr_{\nu'}(\mathcal{E}_n) ds + \int_{\sqrt{2n\Delta^2\log\frac{1}{\Pr_{\nu'}(\mathcal{E}_n)}}}^{+\infty} \exp(-\frac{s^2}{2n\Delta^2}) ds\\
        = & \sqrt{2n\Delta^2\log\frac{1}{\Pr_{\nu'}(\mathcal{E}_n)}}\Pr_{\nu'}(\mathcal{E}_n)+\frac{\sqrt{2\pi}\sqrt{n\Delta^2}}{\sqrt{2\pi}}\int_{\sqrt{2\log\frac{1}{\Pr_{\nu'}(\mathcal{E}_n)}}}^{+\infty}\exp(-\frac{s^2}{2}) ds\\
        \leq & \sqrt{2n\Delta^2\log\frac{1}{\Pr_{\nu'}(\mathcal{E}_n)}}\Pr_{\nu'}(\mathcal{E}_n)+\sqrt{2\pi n\Delta^2}\exp(-\frac{\sqrt{2\log\frac{1}{\Pr_{\nu'}(\mathcal{E}_n)}}}{2})\\
        = & \sqrt{2n\Delta^2\log\frac{1}{\Pr_{\nu'}(\mathcal{E}_n)}}\Pr_{\nu'}(\mathcal{E}_n)+\sqrt{2\pi n\Delta^2}\Pr_{\nu'}(\mathcal{E}_n).
    \end{align*}
    The second last step is by the equality $\int_t^{+\infty}\frac{1}{\sqrt{2\pi}}e^{-\frac{s^2}{2}} ds \leq e^{-\frac{t^2}{2}}$ for $t > 0$. The reason is
    \begin{align*}
        \int_t^{+\infty}\frac{1}{\sqrt{2\pi}}e^{-\frac{s^2}{2}} ds = \int_0^{+\infty}\frac{1}{\sqrt{2\pi}}e^{-\frac{(s+t)^2}{2}}ds\leq e^{-\frac{t^2}{2}}\int_0^{+\infty}\frac{1}{\sqrt{2\pi}}e^{-\frac{s^2}{2}}ds\leq e^{-\frac{t^2}{2}}.
    \end{align*}
    We complete the proof.
\end{proof}

Besides Lemma \ref{lemma:inequality-for-kl-diver-of-conditional-prob}, we further introduce a Lemma which is similar, but focuses on the single side.

\begin{lemma}
    \label{lemma:enhanced-probability-gap-between-unique-arm-differs-instances}
    Consider two instances $\nu'$, $\nu''$, which only differ in arm $a_1, a_2,\cdots,a_M$, i.e. $r^{\nu'}_a\neq r^{\nu''}_a$ holds for all $a\in \{a_1, a_2,\cdots,a_M\}$ and $r^{\nu'}_a = r^{\nu''}_a$ holds for all $a\notin \{a_1, a_2,\cdots,a_M\}$. Denote $\Delta_i = |r_{a_i}^{\nu'}-r_{a_i}^{\nu''}|$, and $\tilde{\tau}$ is an arbitrary stopping time. If event $\mathcal{E}\in \mathcal{F}_{\tilde{\tau}}$ fulfills $\mathcal{E}\subset \{N_{a_i}(\tilde{\tau})\leq n_{i},\forall i\in [M]\}$ with some $\{n_{i}\}_{i=1}^{M} \in \mathbb{N}^{M}$, we can conclude
    \begin{align*}
        \log\frac{\Pr_{\nu'}(\mathcal{E}) }{\Pr_{\nu''}(\mathcal{E}) } \leq \sum_{i=1}^M \sqrt{2n_i\Delta_i^2\log\frac{1}{\Pr_{\nu'}(\mathcal{E})}}+\sqrt{2\pi n_i\Delta_i^2} + \frac{1}{2} n_i\Delta_i^2
    \end{align*}
\end{lemma}
\begin{proof}[Proof of Lemma \ref{lemma:enhanced-probability-gap-between-unique-arm-differs-instances}]
    From the Lemma 18 in \cite{kaufmann2016complexity}, we have
    \begin{align*}
        \Pr_{\nu''}(\mathcal{E}) = & \mathbb{E}_{\nu'}\mathds{1}(\mathcal{E})\exp\left(-\sum_{i=1}^M\sum_{s=1}^{N_i(\tilde{\tau})}\hat{KL}_{i,s}\right)\\
        = & \mathbb{E}_{\nu'}\left[\mathds{1}(\mathcal{E})\exp\left(-\sum_{i=1}^M\sum_{s=1}^{N_i(\tilde{\tau})}\hat{KL}_{i,s}\right)|\mathcal{E}\right]\Pr_{\nu'}(\mathcal{E}) + \\
        & \mathbb{E}_{\nu'}\left[\mathds{1}(\mathcal{E})\exp\left(-\sum_{i=1}^M\sum_{s=1}^{N_i(\tilde{\tau})}\hat{KL}_{i,s}\right)|\neg\mathcal{E}\right]\Pr_{\nu'}(\neg\mathcal{E})\\
        = & \mathbb{E}_{\nu'}\left[\exp\left(-\sum_{i=1}^M\sum_{s=1}^{N_i(\tilde{\tau})}\hat{KL}_{i,s}\right)|\mathcal{E}\right]\Pr_{\nu'}(\mathcal{E})\\
        \geq & \exp\left(\mathbb{E}_{\nu'}\left[-\sum_{i=1}^M\sum_{s=1}^{N_i(\tilde{\tau})}\hat{KL}_{i,s}|\mathcal{E}\right]\right)\Pr_{\nu'}(\mathcal{E}),
    \end{align*}
    where $\hat{KL}_{i,s}:=\log\frac{\exp(-\frac{(X_{i,s}-r_{a_i}^{\nu'})^2}{2})}{\exp(-\frac{(X_{i,s}-r_{a_i}^{\nu''})^2}{2})}=\frac{(2X_{i,s}-r_{a_i}^{\nu'}-r_{a_i}^{\nu''})(r_{a_i}^{\nu'}-r_{a_i}^{\nu''})}{2}$, $X_{i,s}\sim N(r_{a_i}^{\nu'}, 1)$.
    Following the last line, we have
    \begin{align*}
        \log\frac{\Pr_{\nu'}(\mathcal{E})}{\Pr_{\nu''}(\mathcal{E})} \leq & \mathbb{E}_{\nu'}\left[\sum_{i=1}^M\sum_{s=1}^{N_i(\tilde{\tau})}\hat{KL}_{i,s}|\mathcal{E}\right]\\
        = & \frac{\mathbb{E}_{\nu'}\left[\mathds{1}(\mathcal{E})\sum_{i=1}^M\sum_{s=1}^{N_i(\tilde{\tau})}\hat{KL}_{i,s}\right]}{\Pr_{\nu'}(\mathcal{E})}.
    \end{align*}
    and we can conclude $\Pr_{\nu'}(\mathcal{E})\log\frac{\Pr_{\nu'}(\mathcal{E})}{\Pr_{\nu''}(\mathcal{E})} \leq \sum_{i=1}^M\mathbb{E}_{\nu'}\left[\mathds{1}(\mathcal{E})\sum_{s=1}^{N_i(\tilde{\tau})}\hat{KL}_{i,s}\right]$. We turn to find an upper bound for each $\mathbb{E}_{\nu'}\left[\mathds{1}(\mathcal{E})\sum_{s=1}^{N_i(\tilde{\tau})}\hat{KL}_{i,s}\right]$. Since $\mathcal{E}\Rightarrow N_{a_i}(\tilde{\tau}) < n_i$, we have
    \begin{align*}
        & \mathbb{E}_{\nu'}\left[\mathds{1}(\mathcal{E})\sum_{s=1}^{N_i(\tilde{\tau})}\hat{KL}_{i,s}\right]\\
        \leq & \mathbb{E}_{\nu'}\left[\mathds{1}(\mathcal{E})\left(\max_{1\leq t\leq n_i}\sum_{s=1}^{t}\hat{KL}_{i,s}-\frac{\Delta_i^2}{2}\right)\right] + \frac{n_i\Delta_i^2}{2}\Pr_{\nu'}(\mathcal{E})
    \end{align*}
    Denote $D_i:=\max_{1\leq t\leq n_i}\sum_{s=1}^{t}\left(Z_s-\frac{\Delta^2_i}{2}\right)$, $Z_s | Z_1,Z_2,\cdots, Z_{s-1}\sim N(\frac{\Delta_i^2}{2}, \Delta_i^2)$, we have
    \begin{align*}
        & \mathbb{E}_{\nu'}\left[\mathds{1}(\mathcal{E}) D_i\right]\\
        \leq & \mathbb{E}_{\nu'}\left[\mathds{1}(\mathcal{E})\mathds{1}(D\geq 0) D_i\right]\\
        = & \int_{0}^{+\infty} \Pr_{\nu'} \Big(\mathds{1}(\mathcal{E})\mathds{1}(D_i\geq 0) D_i > s\Big) ds\\
        \leq & \int_{0}^{+\infty} \Pr_{\nu'} \Big(\mathcal{E}, D_i > s\Big) ds\\
        \leq & \int_{0}^{+\infty} \min\left\{\Pr_{\nu'}(\mathcal{E}), \Pr_{\nu'}\Big( D_i > s\Big)\right\} ds
    \end{align*}
    From Lemma \ref{lemma:Gaussian-summation-submartingale}, we can conclude $\Pr_{\nu'}\big( D_i\geq s\big)\leq \exp(-\frac{s^2}{2n_i\Delta^2_i})$. Plug in the inequality, we get
    \begin{align*}
        \mathbb{E}_{\nu'}\left[\mathds{1}(\mathcal{E}) D_i\right] \leq & \int_{0}^{+\infty} \min\left\{\Pr_{\nu'}(\mathcal{E}), \exp(-\frac{s^2}{2n_i\Delta^2_i})\right\} ds\\
        = & \int_{0}^{\sqrt{2n_i\Delta^2_i\log\frac{1}{\Pr_{\nu'}(\mathcal{E})}}} \Pr_{\nu'}(\mathcal{E}) ds + \int_{\sqrt{2n_i\Delta^2_i\log\frac{1}{\Pr_{\nu'}(\mathcal{E})}}}^{+\infty} \exp(-\frac{s^2}{2n_i\Delta^2}) ds\\
        = & \sqrt{2n_i\Delta^2_i\log\frac{1}{\Pr_{\nu'}(\mathcal{E})}}\Pr_{\nu'}(\mathcal{E})+\frac{\sqrt{2\pi}\sqrt{n_i\Delta^2_i}}{\sqrt{2\pi}}\int_{\sqrt{2\log\frac{1}{\Pr_{\nu'}(\mathcal{E})}}}^{+\infty}\exp(-\frac{s^2}{2}) ds\\
        \leq & \sqrt{2n_i\Delta^2_i\log\frac{1}{\Pr_{\nu'}(\mathcal{E})}}\Pr_{\nu'}(\mathcal{E})+\sqrt{2\pi n_i\Delta^2_i}\exp(-\frac{\sqrt{2\log\frac{1}{\Pr_{\nu'}(\mathcal{E})}}}{2})\\
        = & \sqrt{2n_i\Delta^2_i\log\frac{1}{\Pr_{\nu'}(\mathcal{E})}}\Pr_{\nu'}(\mathcal{E})+\sqrt{2\pi n_i\Delta^2_i}\Pr_{\nu'}(\mathcal{E}).
    \end{align*}
    In summary, we have proved
    \begin{align*}
        & \Pr_{\nu'}(\mathcal{E})\log\frac{\Pr_{\nu'}(\mathcal{E})}{\Pr_{\nu''}(\mathcal{E})}\\
        \leq & \sum_{i=1}^M\sqrt{2n_i\Delta^2\log\frac{1}{\Pr_{\nu'}(\mathcal{E})}}\Pr_{\nu'}(\mathcal{E})+\sqrt{2\pi n_i\Delta^2}\Pr_{\nu'}(\mathcal{E}) + \frac{n_i\Delta_i^2}{2}\Pr_{\nu'}(\mathcal{E}),
    \end{align*}
    which is equivalent to the conclusion we need to prove.
\end{proof}

\subsection{Proof of Theorem \ref{theorem:lower-bound-positive-case}}
\label{section:proof-of-lower-bound-positive-instance}
Following the sketch proof in section \ref{sec:Lower-Bound-for-Positive-Case}, we prove the following lemmas one by one. The first step is to build a ``probabilistic constant lower bound''. Without further description, we only consider symmetric algorithms in this section.
\begin{lemma}
    \label{lemma:high-prob-lower-bound-positive-instance}
    Apply a $\delta$-PAC and symmetric algorithm to a positive instance $\nu$, If $\max_{i,j\in [K]}\Delta_{i,j} < 1$ and $\delta < 10^{-8}$, $\frac{1}{8m} \geq \sqrt{\delta}$, all holds, we have
    \begin{align*}
        \sum_{j=1}^m \Pr_{\nu}\Big(\hat{a}=j, N_j(\tau) > \frac{\log\frac{1}{\delta}}{100\Delta_{j,0}^2}\Big) \geq \frac{1}{2}
    \end{align*}
\end{lemma}
\begin{proof}[Proof of Lemma \ref{lemma:high-prob-lower-bound-positive-instance}]

    Recall the notation $\mu_1\geq \mu_2\geq \cdots\geq \mu_m >\mu_0\geq \mu_{m+1}\geq\cdots \geq \mu_K$. The $\delta$-PAC assumption implies
    \begin{align*}
        \Pr_{\nu}(\tau<+\infty,\hat{a}\in [m]\cup \{a: \mu_a = \mu_0\} ) > 1-\delta
    \end{align*}
    Before the proof of the Lemma \ref{lemma:high-prob-lower-bound-positive-instance}, we first show a claim 
    \begin{align}
        \Pr_{\nu}(\tau<+\infty, \hat{a}\in  \{a: \mu_a = \mu_0\} ) < 4\delta\label{eqn:claim_cannot_output_mu0}
    \end{align}
    This claim is to show $\Pr_{\nu}(\tau<+\infty,\hat{a}\in [m])\geq 1-\Theta(\delta)$. If $\{a: \mu_a = \mu_0\}= \emptyset$, (\ref{eqn:claim_cannot_output_mu0}) is immediate. We therefore assume $\{a:\mu_a=\mu_0\}\neq\emptyset$ and argue by contradiction. If $\Pr_{\nu}(\tau<+\infty, \hat{a}\in  \{a: \mu_a = \mu_0\} ) \geq 4\delta$ holds, then there exists a (possibly $\delta$-dependent) constant $T_0(\delta)$, such that
    \begin{align}
        \label{eqn:T0-existence_cannot-output-mu0}
        \Pr_{\nu}(\tau<T_0, \hat{a}\in  \{a: \mu_a = \mu_0\} ) \geq 2\delta.
    \end{align}
    We define $\nu'(\epsilon)$ with $r_a^{\nu'(\epsilon)}=\begin{cases}\mu_0-\epsilon & \mu_a=\mu_0\\ r_a^{\nu} & \text{else} \end{cases}$, where $\epsilon$ is an arbitrary positive number. We know $i^*(\nu'(\epsilon)) = [m]$. By $\delta$-PAC assumption, 
    \begin{align*}
        & \Pr_{\nu'(\epsilon)}(\tau<+\infty, \hat{a}\in  [m] ) > 1-\delta\\
        \Rightarrow & \Pr_{\nu'(\epsilon)}(\tau=+\infty \text{ or } \hat{a}\notin  [m] ) < \delta\\
        \Rightarrow & \Pr_{\nu'(\epsilon)}(\hat{a}\in  \{a: \mu_a = \mu_0\} ) < \delta\\
        \Rightarrow & \Pr_{\nu'(\epsilon)}(\tau< T_0, \hat{a}\in \{a: \mu_a = \mu_0\} ) < \delta.
    \end{align*}
    Apply Lemma \ref{lemma:enhanced-probability-gap-between-unique-arm-differs-instances} by taking event $\mathcal{E}=\{\tau<T_0, \hat{a}\in  \{a: \mu_a = \mu_0\}\}$, we have
    \begin{align*}
        \log\frac{\Pr_{\nu}(\mathcal{E}) }{\Pr_{\nu'(\epsilon)}(\mathcal{E}) } \leq & \sum_{a:\mu_a=\mu_0} \sqrt{2T_0\epsilon^2\log\frac{1}{\Pr_{\nu'}(\mathcal{E})}}+\sqrt{2\pi T_0\epsilon^2} + \frac{1}{2} T_0\epsilon^2\\
        \log\frac{\Pr_{\nu}(\mathcal{E}) }{\Pr_{\nu'(\epsilon)}(\mathcal{E}) } \geq & \log\frac{2\delta }{\delta} = \log 2
    \end{align*}
    The second line is by (\ref{eqn:T0-existence_cannot-output-mu0}) and $\Pr_{\nu'(\epsilon)}(\tau< T_0, \hat{a}\in \{a: \mu_a = \mu_0\} ) < \delta$. Notice that we take $\epsilon\rightarrow 0$ in the right hand side of the first line, we can conclude $\log 2 \leq 0$, which is a contradiction.

    With the claim (\ref{eqn:claim_cannot_output_mu0}) established, we deduce
    \begin{align*}
        & \Pr_{\nu}(\tau<+\infty,\hat{a}\in [m]\cup \{a: \mu_a = \mu_0\} ) > 1-\delta\\
        \Leftrightarrow & \Pr_{\nu}(\tau<+\infty,\hat{a}\in [m] ) + \Pr_{\nu}(\tau<+\infty,\hat{a}\in \{a: \mu_a = \mu_0\} ) > 1-\delta\\
        \Rightarrow & \Pr_{\nu}(\tau<+\infty,\hat{a}\in [m] ) > 1-\delta-4\delta\\
        \Rightarrow & \Pr_{\nu}(\tau=+\infty \text{ or }\hat{a}\notin [m] ) < 5\delta\\
        \Rightarrow & \Pr_{\nu}(\hat{a}\notin [m] ) < 5\delta
    \end{align*}
    
    Now, we are ready to prove the Lemma. We prove the conclusion by contradiction. If the Lemma doesn't hold, we have
    \begin{align*}
        & \Pr_{\nu}\Big(\exists j\in[m], \hat{a}=j, N_j(\tau) > \frac{\log\frac{1}{\delta}}{100\Delta_{j,0}^2}\Big) < \frac{1}{2}\\
        \Leftrightarrow & \Pr_{\nu}\left(\cap_{j=1}^m \Big(\{\hat{a}\neq j\} \cup \{ N_j(\tau) \leq \frac{\log\frac{1}{\delta}}{100\Delta_{j,0}^2}\}\Big)\right) \geq \frac{1}{2}\\
        \Rightarrow & \Pr_{\nu}\left(\cap_{j=1}^m\{\hat{a}\neq j\}\right)+\Pr_{\nu}\left(\cap_{j=1}^m \{ N_j(\tau) \leq \frac{\log\frac{1}{\delta}}{100\Delta_{j,0}^2}\}\right) \geq \frac{1}{2}\\
        \Rightarrow & 5\delta+\Pr_{\nu}\left(\cap_{j=1}^m \{ N_j(\tau) \leq \frac{\log\frac{1}{\delta}}{100\Delta_{j,0}^2}\}\right) \geq \frac{1}{2}\\
        \Rightarrow & \Pr_{\nu}\left(\cap_{j=1}^m \{ N_j(\tau) \leq \frac{\log\frac{1}{\delta}}{100\Delta_{j,0}^2}\}\right) \geq \frac{1}{4}\\
        \Rightarrow & \Pr_{\nu}\left(\Big(\cap_{j=1}^m \{ N_j(\tau) \leq \frac{\log\frac{1}{\delta}}{100\Delta_{j,0}^2}\}\Big)\cap \{\hat{a}\in [m]\}\right) + \Pr_{\nu}\left(\{\hat{a}\notin [m]\}\right) \geq \frac{1}{4}\\
        \Rightarrow & \Pr_{\nu}\left(\Big(\cap_{j=1}^m \{ N_j(\tau) \leq \frac{\log\frac{1}{\delta}}{100\Delta_{j,0}^2}\}\Big)\cap \{\hat{a}\in [m]\}\right) + 5\delta\geq \frac{1}{4}\\
        \Rightarrow & \Pr_{\nu}\left(\Big(\cap_{j=1}^m \{ N_j(\tau) \leq \frac{\log\frac{1}{\delta}}{100\Delta_{j,0}^2}\}\Big)\cap \{\hat{a}\in [m]\}\right)\geq \frac{1}{8}\\
        \Rightarrow & \sum_{j'=1}^m\Pr_{\nu}\left(\Big(\cap_{j=1}^m \{ N_j(\tau) \leq \frac{\log\frac{1}{\delta}}{100\Delta_{j,0}^2}\}\Big)\cap \{\hat{a}=j'\}\right)\geq \frac{1}{8}\\
    \end{align*}
    From the last step, we know there exists $j_0$, such that
    \begin{align*}
        \Pr_{\nu}\left(\Big(\cap_{j=1}^m \{ N_j(\tau) \leq \frac{\log\frac{1}{\delta}}{100\Delta_{j,0}^2}\}\Big)\cap \{\hat{a}=j_0\}\right)\geq \frac{1}{8m} \geq \sqrt{\delta}.
    \end{align*}
    We can consider instance $\nu_{j_0}^{(0)}$ whose reward of arm $j_0$ is $\mu_0-\epsilon$ for some $\epsilon > 0$, others remain the same as $\nu$. By the Lemma \ref{lemma:enhanced-probability-gap-between-unique-arm-differs-instances}, we have
    \begin{align*}
        & \log\frac{\sqrt{\delta}}{\delta}\\
        \leq & \log\frac{\Pr_{\nu}\left(\Big(\cap_{j=1}^m \{ N_j(\tau) \leq \frac{\log\frac{1}{\delta}}{100\Delta_{j,0}^2}\}\Big)\cap \{\hat{a}=j_0\}\right)}{\Pr_{\nu_{j_0}^{(0)}}\left(\Big(\cap_{j=1}^m \{ N_j(\tau) \leq \frac{\log\frac{1}{\delta}}{100\Delta_{j,0}^2}\}\Big)\cap \{\hat{a}=j_0\}\right)}\\
        \leq & \sqrt{2\frac{\log\frac{1}{\delta}}{100\Delta_{j_0,0}^2}(\Delta_{j_0,0}+\epsilon)^2\log\frac{1}{\Pr_{\nu'}(\mathcal{E})}}+\sqrt{2\pi \frac{\log\frac{1}{\delta}}{100\Delta_{j,0}^2}(\Delta_{j_0,0}+\epsilon)^2} + \frac{1}{2} \frac{\log\frac{1}{\delta}}{100\Delta_{j_0,0}^2}(\Delta_{j_0,0}+\epsilon)^2\\
        \leq & \sqrt{2\frac{\log\frac{1}{\delta}}{100\Delta_{j_0,0}^2}(\Delta_{j_0,0}+\epsilon)^2\log\frac{1}{\sqrt{\delta}}}+\sqrt{2\pi \frac{\log\frac{1}{\delta}}{100\Delta_{j,0}^2}(\Delta_{j_0,0}+\epsilon)^2} + \frac{1}{2} \frac{\log\frac{1}{\delta}}{100\Delta_{j_0,0}^2}(\Delta_{j_0,0}+\epsilon)^2.
    \end{align*}
    Taking $\epsilon\rightarrow 0$, the last line suggests that
    \begin{align*}
        \frac{1}{2}\log\frac{1}{\delta} \leq \sqrt{\frac{\log\frac{1}{\delta}}{100}\log\frac{1}{\delta}}+\sqrt{2\pi \frac{\log\frac{1}{\delta}}{100}} + \frac{1}{2} \frac{\log\frac{1}{\delta}}{100},
    \end{align*}
    which leads to a contradiction, as $\delta<10^{-8}$.
\end{proof}

We next quantify the algorithm's inability to distinguish two arms when the total arm pulls of the pair is small.
\begin{lemma}
    \label{lemma:unability-to-differentiate-two-arms-when-total-pulling-small}
    Assume $\delta < \frac{1}{e}$. For any $j\in [m]$, $a\in [K], a\neq j$, define
    \begin{align*}
        \tau_n^{j}(a) = & \min \{t:N_j(t)+N_a(t)=n, A_{t+1}\in \{j, a\}\}\\
        T_j(a)= & \lceil \frac{1}{200 \max\{\Delta_{j,0}^2, \Delta_{a,j}^2\}\left(1+\log\frac{1}{\Pr_{\nu}(N_j(\tau)\geq \frac{\log\frac{1}{\delta}}{100\Delta_{j,0}^2}, \hat{a}=j)}\right) }\rceil.
    \end{align*}
    Then, for $n=0, 1, 2, \cdots, T_j(a)-1$, we have
    \begin{align*}
        & \Pr_{\nu}(\tau_n^j(a)<\tau, A_{\tau_n^j(a)+1}=j) \log\frac{\Pr_{\nu}(\tau_n^j(a)<\tau, A_{\tau_n^j(a)+1}=j)}{\Pr_{\nu}(\tau_n^j(a)<\tau, A_{\tau_n^j(a)+1}=a)} + \\
        & \Pr_{\nu}(\tau_n^j(a)<\tau, A_{\tau_n^j(a)+1}=a) \log\frac{\Pr_{\nu}(\tau_n^j(a)<\tau, A_{\tau_n^j(a)+1}=a)}{\Pr_{\nu}(\tau_n^j(a)<\tau, A_{\tau_n^j(a)+1}=j)}\\
        < & 0.4\Pr_{\nu}(\tau_n^j(a)<\tau)
    \end{align*}
\end{lemma}
Dividing $\Pr_{\nu}(\tau_n^j(a)<\tau)$ on both sides, the conclusion is equivalent to
\begin{align*}
    \text{kl}\Big(\Pr_{\nu}(A_{\tau_n^j(a)+1}=a|\tau_n^j(a)<\tau ), 1-\Pr_{\nu}(A_{\tau_n^j(a)+1}=a|\tau_n^j(a)<\tau )\Big)\leq 0.4
\end{align*}
where $\text{kl}(x,y)=x\log\frac{x}{y}+(1-x)\log\frac{1-x}{1-y}$. The reason is $\Pr_{\nu}(\tau_n^j(a)<\tau, A_{\tau_n^j(a)+1}=j) + \Pr_{\nu}(\tau_n^j(a)<\tau, A_{\tau_n^j(a)+1}=a) = \Pr_{\nu}(\tau_n^j(a)<\tau)$.
\begin{proof}[Proof of Lemma \ref{lemma:unability-to-differentiate-two-arms-when-total-pulling-small}]
    The main idea is to utilize Lemma \ref{lemma:inequality-for-kl-diver-of-conditional-prob}. Take $\Delta=\Delta_{a,j}$. From the definition, we know $T_j(a) \geq 1$. 
    For $n=1,2,\cdots,T_j(a)-1$, we have
    \begin{align*}
        & \sqrt{2n\Delta^2\log\frac{1}{\Pr_{\nu'}(\tau_n<\tau)}}\\
        \leq & \sqrt{2n\Delta^2\log\frac{1}{\Pr_{\nu}(N_a(\tau)+N_j(\tau)>n)}}\\
        \leq & \sqrt{2\Big(\frac{2}{200 \max\{\Delta_{j,0}^2, \Delta_{a,j}^2\} \log\frac{e}{\Pr_{\nu}(N_j(\tau)\geq \frac{\log\frac{1}{\delta}}{100\Delta_{j,0}^2},\hat{a}=j)}}\Big) \Delta_{a,j}^2 \log\frac{1}{\Pr_{\nu}(N_j(\tau)\geq \frac{\log\frac{1}{\delta}}{100\Delta_{j,0}^2},\hat{a}=j)}}\\
        \leq & \sqrt{\frac{1}{50}}.
    \end{align*}
    The third line is by the fact that $\{N_j(\tau)\geq \frac{\log\frac{1}{\delta}}{100\Delta_{j,0}^2},\hat{a}=j\}\subset \{N_a(\tau)+N_j(\tau)>n\}$ holds for all $n=0,1,\cdots, T_j(a)-1$. The last step is by the fact that $\frac{\Delta_{a,j}^2}{\max\{\Delta_{j,0}^2, \Delta_{a,j}^2\}}\leq 1$ and 
    \begin{align*}
        \frac{\log\frac{1}{\Pr_{\nu}(N_j(\tau)\geq \frac{\log\frac{1}{\delta}}{100\Delta_{j,0}^2},\hat{a}=j)}}{1+\log\frac{1}{\Pr_{\nu}(N_j(\tau)\geq \frac{\log\frac{1}{\delta}}{100\Delta_{j,0}^2},\hat{a}=j)}}\leq 1.
    \end{align*}
    Similarly, we have
    \begin{align*}
        \sqrt{2\pi n\Delta^2} \leq &\sqrt{\frac{\pi}{50}}\\
        \frac{n\Delta^2 }{2} \leq & \frac{1}{200}.
    \end{align*}
    Not hard to validate that $\sqrt{\frac{1}{50}} + \sqrt{\frac{\pi}{50}} + \frac{1}{200} < 0.4$. Applying Lemma \ref{lemma:inequality-for-kl-diver-of-conditional-prob}, we prove the result.
\end{proof}

The next lemma converts Lemma~\ref{lemma:unability-to-differentiate-two-arms-when-total-pulling-small} into a conclusion relating $\mathbb{E}_{\nu}N_a(\tau)$ to the probability of correctly recommending a given qualified arm with sufficiently many samples.
\begin{lemma}
    \label{lemma:lower-bound-for-EN_a-tau-Pr-hata-j}
    Following the same notation as Lemma \ref{lemma:unability-to-differentiate-two-arms-when-total-pulling-small} and still assume $\delta < \frac{1}{e}$, we have
    \begin{align*}
        \mathbb{E}_{\nu}N_a(\tau) \geq \frac{ T_j(a)}{4}\Pr_{\nu}\left(N_j(\tau)\geq \frac{\log\frac{1}{\delta}}{100\Delta_{j,0}^2}, \hat{a}=j\right)
    \end{align*}
    holds for all $j\in [m], a\in[K],a\neq j$.
\end{lemma}
\begin{proof}[Proof of Lemma \ref{lemma:lower-bound-for-EN_a-tau-Pr-hata-j}]

    We first claim that if $\tau_{n+1}^j(a)<\tau$, for all $t$ such that $\tau_{n+1}^j(a)< t<\tau_{n+1}^j(a)+1$, we have $A_t\neq a$. We prove this claim by contradiction. Assume there exists $t$ such that $\tau_n^j(a) + 1< t<\tau_{n+1}^j(a)+1$, $A_t=a$. Then, from the definition of $\tau_n^j(a)$, we know $N_j(\tau_n^j(a))+N_a(\tau_n^j(a))=n, N_j(\tau_n^j(a)+1)+N_a(\tau_n^j(a)+1)=n+1$. Since $t > \tau_n^j(a)+1$, which means $t-1\geq \tau_n^j(a)+1$, we can conclude 
    \begin{align*}
        & N_j(t)+N_a(t)\\
        = & \mathds{1}(A_t=j)+\mathds{1}(A_t=a) +N_j(t-1)+N_a(t-1)\\
        = & 1 +N_j(t-1)+N_a(t-1)\\
        \geq & 1 +N_j(\tau_n^j(a)+1)+N_a(\tau_n^j(a)+1)\\
        = & n+2.
    \end{align*}
    However, $t\leq\tau_{n+1}^j(a)$ suggests that $N_j(t)+N_a(t) \leq N_1(\tau_{n+1}^j(a))+N_a(\tau_{n+1}^j(a))=n+1$, which leads to a contradiction. The claim follows.
    
    Using the claim, we can rewrite
    \begin{align}
        & \mathbb{E}_{\nu}N_a(\tau)\notag\\
        = & \mathbb{E} \sum_{t=1}^{+\infty}\mathds{1}(A_t=a)\mathds{1}(\tau \geq t)\notag\\
        = & \mathbb{E} \sum_{n=0}^{+\infty}\mathds{1}(A_{\tau_n^j(a)+1}=a)\mathds{1}(\tau > \tau_n^j(a))\notag\\
        \geq & \sum_{n=0}^{T_j(a)-1}\Pr_{\nu}\Big(A_{\tau_n^j(a)+1}=a, \tau > \tau_n^j(a)\Big).\label{eqn:lower-ENatau}
    \end{align}
    
    From Lemma \ref{lemma:unability-to-differentiate-two-arms-when-total-pulling-small}, we know
    \begin{align*}
        \text{kl}\Bigg(\Pr_{\nu}\Big(A_{\tau_n^j(a)+1}=a| \tau > \tau_n^j(a)\Big), 1-\Pr_{\nu}\Big(A_{\tau_n^j(a)+1}=a| \tau > \tau_n^j(a)\Big)\Bigg) < 0.4
    \end{align*}
    where $\text{kl}(x, y)=x\log \frac{x}{y}+(1-y)\log\frac{1-y}{1-x}$. From Pinsker's Inequality, we can further conclude
    \begin{align*}
        & 2\Bigg(1-2\Pr_{\nu}\Big(A_{\tau_n^j(a)+1}=a| \tau > \tau_n^j(a)\Big)\Bigg)^2 < 0.4\\
        \Rightarrow & \frac{1-\sqrt{0.2}}{2} < \Pr_{\nu}\Big(A_{\tau_n^j(a)+1}=a| \tau > \tau_n^j(a)\Big)\\
        \Rightarrow & \Pr_{\nu}\Big(A_{\tau_n^j(a)+1}=a, \tau > \tau_n^j(a)\Big) > \frac{1}{4}\Pr_{\nu}\Big(\tau > \tau_n^j(a)\Big).
    \end{align*}
    Combine the last line and (\ref{eqn:lower-ENatau}), we get
    \begin{align*}
        & \mathbb{E}_{\nu}N_a(\tau)\\
        \geq & \sum_{n=0}^{T_j(a)-1}\Pr_{\nu}\Big(A_{\tau_n^j(a)+1}=a, \tau > \tau_n^j(a)\Big)\\
        \geq & \frac{1}{4}\sum_{n=0}^{T_j(a)-1}\Pr_{\nu}\Big(\tau > \tau_n^j(a)\Big)\\
        \geq & \frac{1}{4}\sum_{n=0}^{T_j(a)-1}\Pr_{\nu}\Big(N_j(\tau) > \frac{\log\frac{1}{\delta}}{100\Delta_{j,0}^2},\hat{a}=j\Big)\\
        = & \frac{T_j(a)}{4}\Pr_{\nu}\Big(N_j(\tau) > \frac{\log\frac{1}{\delta}}{100\Delta_{j,0}^2},\hat{a}=j\Big),
    \end{align*}
    which is the conclusion we wish to prove.
\end{proof}

Finally, we use Lemma \ref{lemma:lower-bound-for-EN_a-tau-Pr-hata-j} to lower bound $\mathbb{E}_{\nu}\tau$ by the optimal value of an optimization problem, as stated in Lemma \ref{lemma:lower-bound-positive-optimization-formulation}.
\begin{proof}[Proof of Lemma \ref{lemma:lower-bound-positive-optimization-formulation}]

    For any $j\in [m]$, by Lemma \ref{lemma:lower-bound-for-EN_a-tau-Pr-hata-j}, we have
    \begin{align*}
        & \mathbb{E}\tau \\
        = &  \mathbb{E} N_j(\tau) + \sum_{a\neq j}\mathbb{E} \sum_{n=0}^{+\infty}\mathds{1}(A_{\tau_n^j(a)+1}=a)\mathds{1}(\tau > \tau_n^j(a))\\
        \geq &  \mathbb{E} N_j(\tau)\mathds{1}(\hat{a}=j , N_j(\tau) > \frac{\log\frac{1}{\delta}}{100\Delta_{j,0}^2 }) + \sum_{a\neq j}\mathbb{E} \sum_{n=0}^{T_j(a)-1}\mathds{1}(A_{\tau_n^j(a)+1}=a)\mathds{1}(\tau > \tau_n^j(a))\\
        \geq & \frac{\log\frac{1}{\delta}}{100\Delta_{j,0}^2 }\Pr(\hat{a}=j , N_j(\tau) > \frac{\log\frac{1}{\delta}}{100\Delta_{j,0}^2 })+\sum_{a\neq j}\frac{1}{4} T_j(a) \Pr_{\nu}\Big(N_j(\tau) > \frac{\log\frac{1}{\delta}}{100\Delta_{j,0}^2},\hat{a}=j\Big)\\
        \geq & \frac{1}{100\Delta_{j,0}^2 }\Pr_{\nu}\Big(N_j(\tau) > \frac{\log\frac{1}{\delta}}{100\Delta_{j,0}^2},\hat{a}=j\Big)+\sum_{a\neq j}\frac{1}{4} T_j(a) \Pr_{\nu}\Big(N_j(\tau) > \frac{\log\frac{1}{\delta}}{100\Delta_{j,0}^2},\hat{a}=j\Big)\\
        \geq & \sum_{a=1}^K\frac{1}{4} T_j(a) \Pr_{\nu}\Big(N_j(\tau) > \frac{\log\frac{1}{\delta}}{100\Delta_{j,0}^2},\hat{a}=j\Big).
    \end{align*}
    The third step is by the Lemma \ref{lemma:lower-bound-for-EN_a-tau-Pr-hata-j}. Since the last line holds for any $j\in[m]$, 
    \begin{align}
        & \mathbb{E}\tau \notag\\
        \geq & \max_{1\leq j\leq m} 
        \sum_{a=1}^K\frac{1}{4} T_j(a) \Pr_{\nu}\Big(N_j(\tau) > \frac{\log\frac{1}{\delta}}{100\Delta_{j,0}^2},\hat{a}=j\Big)\notag\\
        \geq & \max_{1\leq j\leq m} 
        \sum_{a=1}^K
        \frac{\Pr_{\nu}\Big(N_j(\tau) > \frac{\log\frac{1}{\delta}}{100\Delta_{j,0}^2},\hat{a}=j\Big)}{800 \max\{\Delta_{j,0}^2, \Delta_{a,j}^2\}\left(1+\log\frac{1}{\Pr_{\nu}(N_j(\tau)\geq \frac{\log\frac{1}{\delta}}{100\Delta_{j,0}^2}, \hat{a}=j)}\right) }
        \label{eqn:Etau-lower-max1}
    \end{align}
    Meanwhile, we can also derive
    \begin{align}
        \mathbb{E}\tau 
        \geq & \sum_{j=1}^m\mathbb{E}\tau\mathds{1}(\hat{a}=j, N_j(\tau) > \frac{\log\frac{1}{\delta}}{100\Delta_{j,0}^2})\notag\\
        \geq & \sum_{j=1}^m\frac{\log\frac{1}{\delta}}{100\Delta_{j,0}^2}\Pr_{\nu}(\hat{a}=j, N_j(\tau) > \frac{\log\frac{1}{\delta}}{100\Delta_{j,0}^2})\label{eqn:Etau-lower-max2}
    \end{align}
    Merge (\ref{eqn:Etau-lower-max1}), (\ref{eqn:Etau-lower-max2}), and notice that Lemma \ref{lemma:high-prob-lower-bound-positive-instance} asserts
    \begin{align*}
        \sum_{j=1}^m \Pr_{\nu}\Big(\hat{a}=j, N_j(\tau) > \frac{\log\frac{1}{\delta}}{100\Delta_{j,0}^2}\Big) \geq \frac{1}{2},
    \end{align*}
    we know $\mathbb{E}\tau$ must be lower bounded by the optimal value of the following optimization problem.
    \begin{align*}
        \begin{array}{cl}
           \min  & \frac{1}{800}\max\left\{\sum_{j=1}^m\frac{\log\frac{1}{\delta}}{\Delta_{j,0}^2}p_j, \max_{1\leq j\leq m} 
        \sum_{a=1}^K
        \frac{p_j}{\max\{\Delta_{j,0}^2, \Delta_{a,j}^2\}\left(1+\log\frac{1}{p_j}\right) }\right\}\\
        s.t.  &  \sum_{j=1}^m p_j \geq \frac{1}{2}\\
        & 0\leq p_j \leq 1
        \end{array}
    \end{align*}
    where for any symmetric and $\delta$-PAC algorithm, $\Big\{p_j = \Pr_{\nu}\Big(\hat{a}=j, N_j(\tau) > \frac{\log\frac{1}{\delta}}{100\Delta_{j,0}^2}\Big)\Big\}_{j=1}^{m}$ is always a feasible solution. 
    
    The next step is to show $\frac{1}{\max\{\Delta_{j,0}^2, \Delta_{a,j}^2\}}\geq \frac{1}{4\max\{\Delta_{a,0}^2, \Delta_{a,j}^2\}}$ holds for $j\in [m], a\in [K]$. If $\mu_a \geq \mu_j$, we have $\frac{1}{\max\{\Delta_{j,0}^2, \Delta_{a,j}^2\}} \geq \frac{1}{\max\{\Delta_{a,0}^2, \Delta_{a,j}^2\}}$. If $\mu_j > \mu_a > \mu_0$, we have $\frac{1}{\max\{\Delta_{j,0}^2, \Delta_{a,j}^2\}} = \frac{1}{\Delta_{j,0}^2} > \frac{1}{4\max\{\Delta_{a,0}^2, \Delta_{a,j}^2\}}$. If $ \mu_0\geq \mu_a$, we have $\frac{1}{\max\{\Delta_{j,0}^2, \Delta_{a,j}^2\}} = \frac{1}{\Delta_{j,a}^2} > \frac{1}{4\max\{\Delta_{a,0}^2, \Delta_{a,j}^2\}}$. Thus, we know 
    $\mathbb{E}\tau$ must be lower bounded by the optimal value of the following optimization problem.
    \begin{align*}
        \begin{array}{cl}
           \min  & \frac{1}{3200}\max\left\{\sum_{j=1}^m\frac{\log\frac{1}{\delta}}{\Delta_{j,0}^2}p_j, \max_{1\leq j\leq m} 
        \sum_{a=1}^K
        \frac{p_j}{\max\{\Delta_{a,0}^2, \Delta_{a,j}^2\}\left(1+\log\frac{1}{p_j}\right) }\right\}\\
        s.t.  &  \sum_{j=1}^m p_j \geq \frac{1}{2}\\
        & 0\leq p_j\leq 1
        \end{array}
    \end{align*}
    Reformulate the above optimization problem, we know the optimal value of (\ref{eqn:opt-formualtion-lower-bound-positive})
    \begin{align*}
        \begin{array}{cl}
            \min & v\\
            s.t.  & v \geq \sum_{j=1}^m\frac{\log\frac{1}{\delta}}{\Delta_{j,0}^2}p_j\\
            & v \geq \sum_{a=1}^K
            \frac{p_j}{\max\{\Delta_{a,0}^2, \Delta_{a,j}^2\}\left(1+\log\frac{1}{p_j}\right) }, \forall j\in [m]\\
            & \sum_{j=1}^m p_j \geq \frac{1}{2}\\
            & 0\leq p_j \leq 1\\
            & v\geq 0.
        \end{array}
    \end{align*}
    is a lower bound of $\mathbb{E}\tau$.
\end{proof}

Given Lemma \ref{lemma:lower-bound-positive-optimization-formulation}, the last step is to prove a lower bound of the optimal value of (\ref{eqn:opt-formualtion-lower-bound-positive}).
\begin{lemma}
    \label{lemma:lower-bound-opt-of-optimization-formulation}
    Denote $\text{opt}$ as the optimal value of problem (\ref{eqn:opt-formualtion-lower-bound-positive}). We can prove
    \begin{align*}
        \text{opt}\geq \Omega\left(\min_{1\leq j \leq m} \frac{\log\frac{1}{\delta}}{\Delta_{j,0}^2}  + \frac{\frac{1}{j} }{\log^2 (m+1)}\sum_{a=1}^K \frac{1}{\max\{\Delta_{a,0}^2, \Delta_{a,j}^2\}}\right)
    \end{align*}
\end{lemma}
\begin{proof}[Proof of Lemma \ref{lemma:lower-bound-opt-of-optimization-formulation}]

    We first show function $f(p)=\frac{p}{1+\log\frac{1}{p}}$ is convex for $p \in (0, 1]$. Easy to calculate
    \begin{align*}
        \frac{\partial \frac{p}{1+\log\frac{2}{p}}}{\partial p} = & \frac{1+\log\frac{2}{p}-p(-\frac{1}{p})}{(1+\log\frac{2}{p})^2}\\
        = & \frac{1}{1+\log\frac{2}{p}} + \frac{1}{(1+\log\frac{2}{p})^2}\\
        \frac{\partial^2 \frac{p}{1+\log\frac{2}{p}}}{\partial p^2} = & \left(-\frac{1}{(1+\log\frac{2}{p})^2}-\frac{2}{(1+\log\frac{2}{p})^3}\right)\left(-\frac{1}{p}\right)\\
        = & \left(\frac{1}{(1+\log\frac{2}{p})^2}+\frac{2}{(1+\log\frac{2}{p})^3}\right)\frac{1}{p}.
    \end{align*}
    It is straight forward to see $\frac{\partial^2 \frac{p}{1+\log\frac{2}{p}}}{\partial p^2}>0$ holds for all $p\in (0, 1]$, suggesting that function $\frac{p}{1+\log\frac{1}{p}}$ is convex, further indicating (\ref{eqn:opt-formualtion-lower-bound-positive}) is a convex optimization problem.

    Introduce dual variables $\alpha\geq 0,\gamma\geq 0, \{\beta_j\}_{j=1}^m \in \mathbb{R}_+^m$ and Lagrangian function
    \begin{align*}
        & L(\alpha,\gamma, \{\beta_j\}_{j=1}^m; \{p_j\}_{j=1}^m, v)\\
        = & v + \left(\sum_{j=1}^m\frac{\log\frac{1}{\delta}}{\Delta_{j,0}^2}p_j-v\right)\alpha +\\
        & \sum_{j=1}^m \left(\sum_{a=1}^K \frac{p_j}{\max\{\Delta_{a,0}^2, \Delta_{a,j}^2\}\left(1+\log\frac{1}{p_j}\right) }-v\right)\beta_j + (\frac{1}{2}-\sum_{j=1}^m p_j)\gamma\\
        = & v(1-\alpha-\sum_{j=1}^m\beta_j) + \sum_{j=1}^m\left(\frac{\log\frac{1}{\delta}}{\Delta_{j,0}^2}p_j + \beta_j\sum_{a=1}^K \frac{p_j}{\max\{\Delta_{a,0}^2, \Delta_{a,j}^2\}\left(1+\log\frac{1}{p_j}\right) }-\gamma p_j\right) + \frac{\gamma}{2}.
    \end{align*}
    
    Take $\alpha^0 = \frac{1}{2}$, $\beta_j^0 = \frac{\frac{1}{2j} }{2\sum_{s=1}^m \frac{1}{2s}}$, $\gamma^0 = \min\limits_{1\leq j \leq m} \frac{\log\frac{1}{\delta}}{\Delta_{j,0}^2}\frac{1}{2}\cdot  + \frac{\frac{1}{2j}\sum_{a=1}^K \frac{1}{\max\{\Delta_{a,0}^2, \Delta_{a,j}^2\}} }{(1+\log (m+1))(1+\log 200 +\log m)}$. Since (\ref{eqn:opt-formualtion-lower-bound-positive}) is a convex optimization problem, by the strong duality, we can assert 
    \begin{align*}
        \text{opt} = & \max_{\alpha\geq 0,\gamma\geq 0, \{\beta_j\}_{j=1}^m \in \mathbb{R}_+^m} \min_{\{p_j\}_{j=1}^m\in [0,1]^{m}, v\geq 0}L(\alpha,\gamma, \{\beta_j\}; \{p_j\}, v)\\
        \geq & \min_{\{p_j\}_{j=1}^m\in [0,1]^{m}, v\geq 0}L(\alpha^0,\gamma^0, \{\beta_j^0\}; \{p_j\}, v).
    \end{align*}
    In the following, we suffice to prove
    \begin{align}
        & \min_{\{p_j\}\in [0,1]^{m}, v\geq 0}L(\alpha^0,\gamma^0, \{\beta_j^0\}_{j=1}^m; \{p_j\}_{j=1}^m, v)\notag\\
        \geq & \Omega\left(\min_{1\leq j \leq m} \frac{\log\frac{1}{\delta}}{\Delta_{j,0}^2}  + \frac{\frac{1}{j} }{\log^2 (m+1)}\sum_{a=1}^K \frac{1}{\max\{\Delta_{a,0}^2, \Delta_{a,j}^2\}}\right).\label{eqn:lower-bound-dual-feasible}
    \end{align}
    We start the proof of (\ref{eqn:lower-bound-dual-feasible}) by claiming
    \begin{align}
        \min_{1\geq p_j\geq 0}\left(\frac{\log\frac{1}{\delta}}{\Delta_{j,0}^2} \alpha^0\cdot p_j + \frac{p_j}{1+\log\frac{2}{p_j}}\sum_{a=1}^K \frac{1}{\max\{\Delta_{a,0}^2, \Delta_{a,j}^2\}}\cdot\beta_j^0 - \gamma^0 p_j\right) \geq -\frac{\gamma^0}{100m}\label{eqn:claim-for-minimized-value}
    \end{align}
    holds for all $j\in [m]$. Define $\bar{L}_j(p)=\frac{\log\frac{1}{\delta}}{\Delta_{j,0}^2} \alpha^0\cdot p + \frac{p}{1+\log\frac{2}{p}}\sum_{a=1}^K \frac{1}{\max\{\Delta_{a,0}^2, \Delta_{a,j}^2\}}\cdot\beta_j^0 - \gamma^0 p$ and $ p^*_j = \arg\min\limits_{1\geq p\geq 0} \bar{L}_j(p)$. Not hard to validate $\bar{L}_j(p)$ is a convex function for $p\in[0, 1]$, and there are two cases.

    If $\frac{\partial \bar{L}_j(p)}{\partial p}\geq 0$ holds for all $0\leq p\leq 1$, we have
    \begin{align*}
        \min_{1\geq p\geq 0}\bar{L}_j(p) = \lim_{p\rightarrow 0}\bar{L}_j(p) = 0 > -\frac{\gamma^0}{100m},
    \end{align*}
    suggesting that the claim (\ref{eqn:claim-for-minimized-value}) is true.

    If $\frac{\partial \bar{L}_j(p)}{\partial p}< 0$ for some $p$, by the fact that $\bar{L}_j(p)$ is convex, we can assert there exists $p_0\in(0,1)$ such that $\bar{L}_j(p)$ is decreasing in the interval $(0, p_0)$, and $\bar{L}_j(p^*_j) < 0$. Then, we can prove $p^*_j<\frac{1}{100 m}$ by contradiction. If $p^*_j\geq \frac{1}{100 m}$, we have
    \begin{align*}
        & \frac{\log\frac{1}{\delta}}{\Delta_{j,0}^2} \alpha^0\cdot p^*_j + \frac{p^*_j}{1+\log\frac{2}{p^*_j}}\sum_{a=1}^K \frac{1}{\max\{\Delta_{a,0}^2, \Delta_{a,j}^2\}}\cdot\beta_j^0 - \gamma^0 p^*_j\\
        \geq & \frac{\log\frac{1}{\delta}}{\Delta_{j,0}^2} \alpha^0\cdot p^*_j + \frac{p^*_j}{1+\log 200 + \log m}\sum_{a=1}^K \frac{1}{\max\{\Delta_{a,0}^2, \Delta_{a,j}^2\}}\cdot\beta_j^0 - \gamma^0 p^*_j\\
        = & p^*_j\left(\frac{\log\frac{1}{\delta}}{\Delta_{j,0}^2} \alpha^0 + \frac{\sum_{a=1}^K \frac{1}{\max\{\Delta_{a,0}^2, \Delta_{a,j}^2\}}\cdot\beta_j^0}{1+\log 200 + \log m} - \gamma^0 \right)\\
        \geq & p^*_j\left(\frac{\log\frac{1}{\delta}}{\Delta_{j,0}^2} \frac{1}{2} + \frac{\sum_{a=1}^K \frac{1}{\max\{\Delta_{a,0}^2, \Delta_{a,j}^2\}}\frac{\frac{1}{2j}}{2\sum_{s=1}^m\frac{1}{2s} }}{1+\log 200 + \log m} - \frac{\log\frac{1}{\delta}}{\Delta_{j,0}^2} \frac{1}{2}- \frac{\sum_{a=1}^K \frac{1}{2j\max\{\Delta_{a,0}^2, \Delta_{a,j}^2\}}}{(1+\log 200 + \log m)(1+ \log (m+1))} \right)\\
         > & 0.
    \end{align*}
    The second last line is by the definition of $\gamma^0$. The last line is by the fact that $2\sum_{s=1}^m \frac{1}{2s} < 1+\log (m+1)$ for $m\geq 1$. This last line contradicts with the fact that $\bar{L}_j(p^*_j) < 0$. Thus, we can conclude $p^*_j< \frac{1}{100 m}$, further
    \begin{align*}
        & \frac{\log\frac{1}{\delta}}{\Delta_{j,0}^2} \alpha^0\cdot p^*_j + \frac{p^*_j}{1+\log\frac{2}{p^*_j}}\sum_{a=1}^K \frac{1}{\max\{\Delta_{a,0}^2, \Delta_{a,j}^2\}}\cdot\beta_j^0 - \gamma^0 p^*_j\\
        \geq & - \gamma^0 p^*_j\\
        \geq & -\frac{\gamma^0}{100 m}.
    \end{align*}
    That means in the case $\frac{\partial \bar{L}_j(p)}{\partial p}< 0$, we can still conclude the claim (\ref{eqn:claim-for-minimized-value}) is true.

    Based on the (\ref{eqn:claim-for-minimized-value}), we have
    \begin{align*}
        & \min_{\{p_j\}_{j=1}^m\in [0,1]^{m}, v\geq 0}L(\alpha^0,\gamma^0, \{\beta_j^0\}; \{p_j\}, v)\\
        = & \min_{v\geq 0}v(1-\alpha^0-\sum_{j=1}^m\beta_j^0) + \\
        & \sum_{j=1}^m\min_{1\geq p_j\geq 0}\left(\frac{\log\frac{1}{\delta}}{\Delta_{j,0}^2}p_j + \beta_j^0\sum_{a=1}^K \frac{p_j}{\max\{\Delta_{a,0}^2, \Delta_{a,j}^2\}\left(1+\log\frac{1}{p_j}\right) }-\gamma^0 p_j\right) + \frac{\gamma^0}{2}\\
        \geq & \sum_{j=1}^m-\frac{\gamma^0}{100m}+ \frac{\gamma^0}{2}\\
        \geq & \frac{\gamma^0}{4}.
    \end{align*}
    The second last line is by the fact that $1-\alpha^0-\sum_{j=1}^m\beta_j^0 = 1-\frac{1}{2}-\sum_{j=1}^m\frac{\frac{1}{2j} }{2\sum_{s=1}^m \frac{1}{2s}} = 0$ and the claim (\ref{eqn:claim-for-minimized-value}). From the definition of $\gamma^0$, we know 
    \begin{align*}
        & \min_{\{p_j\}\in [0,1]^{m}, v\geq 0}L(\alpha^0,\gamma^0, \{\beta_j^0\}; \{p_j\}, v)\\
        \geq & \Omega\left(\min_{1\leq j \leq m} \frac{\log\frac{1}{\delta}}{\Delta_{j,0}^2}  + \frac{\frac{1}{j} }{\log^2 (m+1)}\sum_{a=1}^K \frac{1}{\max\{\Delta_{a,0}^2, \Delta_{a,j}^2\}}\right),
    \end{align*}
    which also completes the proof.
\end{proof}

\section{Proof of Upper Bound}
\label{sec:Proof-of-Upper-Bound}

\subsection{Rigorous Presentation of Algorithm}
Introduce
\begin{align}
    \label{eqn:definition-U_t_delta}
    U(t, \delta)=\frac{\sqrt{2 \cdot 2^{\lceil\log_2 t\rceil^+}\log\frac{2 (\lceil\log_2t\rceil^+)^2}{\delta}}}{t}.
\end{align}
where $\lceil x\rceil:=\max\{x, 1\}$.

Before presenting the full statement of the algorithm, we first introduce some notations
\begin{equation}\label{eqn:Initialize-Parameters-in-SEE}
    \begin{aligned}
        & N_a^{\text{ee}}(\mathcal{H}^{ee}) =  \sum_{s=1}^{|\mathcal{H}^{ee}|}\mathds{1}(A_s(\mathcal{H}^{\text{ee}})=a), N_a^{\text{et}}(\mathcal{H}^{et}) = \sum_{s=1}^{|\mathcal{H}^{et}|}\mathds{1}(A_s(\mathcal{H}^{\text{et}})=a)\\
        & \hat{\mu}_a^{\text{ee}}(\mathcal{H}^{ee}) = \frac{\sum_{s=1}^{|\mathcal{H}^{ee}|}\mathds{1}(A_s(\mathcal{H}^{\text{ee}})=a) X_s(\mathcal{H}^{\text{ee}}) }{N_a^{\text{ee}}(\mathcal{H}^{ee})}, \\
        & \hat{\mu}_a^{\text{et}}(\mathcal{H}^{et}) =\frac{\sum_{s=1}^{|\mathcal{H}^{et}|}\mathds{1}(A_s(\mathcal{H}^{\text{et}})=a) X_s(\mathcal{H}^{\text{et}}) }{N_a^{\text{et}}(\mathcal{H}^{et})}\\
        & \text{UCB}^{\text{ee}}_a(\mathcal{H}^{ee},\delta) = \hat{\mu}_a^{\text{ee}}(\mathcal{H}^{ee}) + U(N_a^{\text{ee}}(\mathcal{H}^{ee}), \frac{\delta}{K})\\
        & \text{LCB}^{\text{ee}}_a(\mathcal{H}^{ee},\delta) = \hat{\mu}_a^{\text{ee}}(\mathcal{H}^{ee}) - C\cdot U(N_a^{\text{ee}}(\mathcal{H}^{ee}), \frac{\delta}{K})\\
        & \text{UCB}^{\text{et}}_a(\mathcal{H}^{et},\delta) = \hat{\mu}_a^{\text{et}}(\mathcal{H}^{et}) + U(N_a^{\text{et}}(\mathcal{H}^{ee}), \frac{\delta}{K})\\
        & \text{LCB}^{\text{et}}_a(\mathcal{H}^{et},\delta) = \hat{\mu}_a^{\text{et}}(\mathcal{H}^{et}) - U(N_a^{\text{et}}(\mathcal{H}^{ee}), \frac{\delta}{K})
    \end{aligned}
\end{equation}
We treat $\mathcal{H}^{\mathrm{ee}}$ and $\mathcal{H}^{\mathrm{et}}$ as stacks with indexed entries: $A_s(\mathcal{H}^{\mathrm{ee}})$ and $X_s(\mathcal{H}^{\mathrm{ee}})$ denote, respectively, the action and reward stored at index $s$ in $\mathcal{H}^{\mathrm{ee}}$, and analogously for $\mathcal{H}^{\mathrm{et}}$.
When $N_a^{\mathrm{ee}}(\mathcal{H}^{\mathrm{ee}})=0$ for some $a\in[K]$, we set $\mathrm{UCB}^{\mathrm{ee}}_a=+\infty$ and $\mathrm{LCB}^{\mathrm{ee}}_a=-\infty$; the same convention applies to $\mathrm{UCB}^{\mathrm{et}}_a$ and $\mathrm{LCB}^{\mathrm{et}}_a$ when $N_a^{\mathrm{et}}(\mathcal{H}^{\mathrm{et}})=0$.

With this notation in place, we present Algorithm~\ref{alg:SEE-with-Bracket}, adapted from~\cite{pmlr-v267-li25f}. Our analysis follows closely Section~4 and Appendix~B of~\cite{pmlr-v267-li25f}; nevertheless, we provide the full algorithmic description and proofs for completeness. In the subsequent subsections, we discuss the algorithm line by line and summarize our modifications relative to the prior literature in Appendix~\ref{sec:Adaptation-of-History-Literature}.

\begin{algorithm}
    \caption{Sequential Exploration Exploitaion (SEE)}
    \label{alg:SEE-with-Bracket}
    {\bfseries Input:} Bracket $B$, Threshold $\mu_0$, Tolerance level $\delta$, Tunable parameter $C>1$, Arm number $K$\; \label{alg-line:input-parameter}
    
    {\bfseries Initialize:} $\delta_k=\frac{1}{3^k}, \beta_k=2^k,\alpha_k=5^k, \forall k\in \mathbb{N}$, History $\mathcal{H}^{\text{ee}}=\mathcal{H}^{\text{et}}=\emptyset$, Temporary Container $Q=\emptyset$. Introduce notations in (\ref{eqn:Initialize-Parameters-in-SEE}), $t^{\text{ee}}=t^{\text{et}}=0$\;\label{alg-line:initialize-parameter}
    \For{phase index $k=1,2,\cdots$}{\label{alg-line:start-of-new-phase}
        \tcc{Phase $k$ starts}
        \While{True}{ \label{alg-line:while-loop-start}
            \If{$t^{\text{ee}}\geq 1$ and $\text{LCB}^{\text{ee}}_{A_t^{\text{ee}}}(\mathcal{H}^{\text{ee}},\delta_k)\geq \mu_0$}{\label{alg-line:rough-guess-of-hat-a_k}
                Take $\hat{a}_k\gets A_t^{\text{ee}}$\; \label{alg-line:hata_k-value2}
                $\mathcal{H}^{\text{ee}} \gets \mathcal{H}^{\text{ee}}\setminus\{(A^{\text{ee}}_{t}, X)\}$, $N_{A^{\text{ee}}_{t}}^{\text{ee}} \gets N_{A^{\text{ee}}_{t}}^{\text{ee}}-1$, $Q \gets Q\cup\{(A^{\text{ee}}_{t}, X)\}$, $\hat{a}\gets A_t^{\text{ee}}$\;\label{alg-line:move-latest-samples-into-Q}
                \While{$N_{\hat{a}}^{\text{et}} \leq \frac{(C+3)^2}{(C-1)^2}\beta_k\log\frac{4K\alpha_k}{\delta}-1$}{\label{alg-line:stop-condition-in-exploitation}
                    Sample $X \sim \rho_{\hat{a}}$, $\mathcal{H}^{\text{et}} \gets \mathcal{H}^{\text{et}}\cup\{(\hat{a}, X)\}$, $N_{\hat{a}}^{\text{et}} \gets N_{\hat{a}}^{\text{et}}+1$, $t^{\text{et}}\gets t^{\text{et}}+1$\;
                    \If{$\text{LCB}^{\text{et}}_{\hat{a}}(\mathcal{H}^{\text{et}}, \frac{\delta}{\alpha_k}) > \mu_0$ }{\label{alg-line:stop-condition-exploitation}
                        Return $\hat{a}$ as a qualified arm\;\label{alg-line:terminate-and-output-arm}
                    }
                }
                Break\tcp*[l]{Enter phase $k+1$}\label{alg-line:exploitation-ends}
            }
            \If{$\forall a\in B$, $\text{UCB}^{\text{ee}}_a(\mathcal{H}^{\text{ee}},\delta_k)\leq \mu_0$}{\label{alg-line:Negative-Output-Condition}
                \eIf{$\delta_k\leq \frac{\delta}{3}$ and $|B|=K$}{\label{alg-line:Negative-branch-output} 
                    Take $\hat{a}_k=\textsf{None}$ and return \textsf{None}\;\label{alg-line:Negative-branch-output-None}
                    \tcc{Only the bracket with size $K$ is allowed to output \textsf{None}} 
                }{
                    Take $\hat{a}_k=\textsf{Not Completed}$ and break\tcp*[l]{Enter phase $k+1$} \label{alg-line:hata_k-value1}
                }
            }
            \If{$N_a^{\text{ee}} > (C+1)^2\beta_k \log\frac{4K}{\delta_k}-1, \forall a\in B$}{\label{alg-line:run-out-of-budget}
                 Take $\hat{a}_k=\textsf{Not Completed}$ and break\tcp*[l]{Enter phase $k+1$}\label{alg-line:hata_k-value3}
            }\label{alg-line:Conditions-Terminate-Exploration-Ends}
            Determine $A^{\text{ee}}_{t}=\arg\max\limits_{a\in B, N_a^{\text{ee}}\leq (C+1)^2\beta_k \log\frac{4K}{\delta_k}-1}\text{UCB}_a(\mathcal{H}^{\text{ee}},\delta_k)$\;\label{line-alg:maximum-N_a-ee}
            \eIf{$\exists (a, X)\in Q$ such that $a=A^{\text{ee}}_{t}$}{\label{alg-line:transfer-Q-to-H-ee-start}
                $\mathcal{H}^{\text{ee}} \gets \mathcal{H}^{\text{ee}}\cup\{(a, X)\}$, $N_{A^{\text{ee}}_{t}}^{\text{ee}}\gets N_{A^{\text{ee}}_{t}}^{\text{ee}}+1$, $Q\gets Q\setminus\{(a, X)\} $\; 
            }{
                Sample $X \sim \rho_{A_t^{\text{ee}}}$, $\mathcal{H}^{\text{ee}} \gets \mathcal{H}^{\text{ee}}\cup\{(A^{\text{ee}}_{t}, X)\}$, $N_{A^{\text{ee}}_{t}}^{\text{ee}} \gets N_{A^{\text{ee}}_{t}}^{\text{ee}}+1$, $t^{\text{ee}}\gets t^{\text{ee}}+1$\;
            }\label{alg-line:exploration-arm-pull}
        }
        Take $\tau^{\text{ee}}_k(B)\gets t^{\text{ee}}, \tau^{\text{et}}_k(B)\gets t^{\text{et}}$\tcp*[l]{Independent of algorithm design, only for analysis} \label{alg-line:phase-ends}
        \tcc{Phase $k$ ends}
    }\label{alg-line:enter-new-phase}
\end{algorithm}

\subsubsection{From Line \ref{alg-line:input-parameter} to \ref{alg-line:initialize-parameter}: Input and Initialize Parameter}
\
Algorithm \ref{alg:SEE-with-Bracket} takes a bracket $B$, a threshold $\mu_0$, a tolerance level $\delta$, the number of arms $K$, tunable parameter $C$ as input, and only pull arms in $B$. 
The parameters $\mu_0$ and $K$ coincide with those used by Algorithm~\ref{alg:Parallel-SEE-on-Bracket}. When Algorithm~\ref{alg:Parallel-SEE-on-Bracket} calls Algorithm~\ref{alg:SEE-with-Bracket}, it sets $\delta \gets \delta/(\lceil \log_2 K\rceil+1)$. We use $C=1.01$ by default, although any value strictly larger than $1$ is admissible.

In addition to the problem parameters, \ref{alg:SEE-with-Bracket} also initialize some parameters to manage the pulling process. Parameter $\{\delta_k=\frac{1}{3^k}, \beta_k=2^k,\alpha_k=5^k\}_{k=1}^{+\infty}$ are used to set up the tolerance level and budget during the pulling process. It is possible to specify other values for $\delta_k, \beta_k, \alpha_k$, see Appendix B.1 in \cite{pmlr-v267-li25f}. 
The containers $\mathcal{H}^{\mathrm{ee}}$, $\mathcal{H}^{\mathrm{et}}$, and $Q$ store samples collected throughout the process. In particular, $\mathcal{H}^{\mathrm{ee}}$ and $\mathcal{H}^{\mathrm{et}}$ are ordered stacks that preserve the arrival order of samples for each arm in $B$, which allows us to remove and transfer the most recently collected sample when needed. The role of $Q$ is to ensure that $\mathrm{LCB}^{\mathrm{ee}}_a$ cannot be greater than $\mu_0$ at the beginning of each phase; see Section~\ref{sec:process-exploration}.

\subsubsection{Line \ref{alg-line:start-of-new-phase} to Line \ref{alg-line:enter-new-phase}: Overall Idea of Algorithm \ref{alg:SEE-with-Bracket} }
At the start of Algorithm \ref{alg:SEE-with-Bracket}, it splits the arm distributions into two parallel and independent copies. One copy is for exploration period, and another copy is for exploitation period. Correspondingly, reward collected in the exploration periods are stored in $\mathcal{H}^{\text{ee}}$. Reward collected in the exploration periods, which is from Line \ref{alg-line:hata_k-value2} to Line \ref{alg-line:exploitation-ends}, are stored in $\mathcal{H}^{\text{et}}$. 

After specifying the parameters through Line \ref{alg-line:input-parameter} and \ref{alg-line:initialize-parameter}, Algorithm \ref{alg:SEE-with-Bracket} splits the pulling process into multiple phases, noted by the For loop in Line \ref{alg-line:start-of-new-phase}. In each phase $k$, Algorithm \ref{alg:SEE-with-Bracket} consists of two parts. 
Each phase consists of two components executed within the \textbf{while}-loop in Line~\ref{alg-line:while-loop-start}:
\begin{itemize}
\item a sequence of stopping/transition checks (Lines~\ref{alg-line:rough-guess-of-hat-a_k}--\ref{alg-line:Conditions-Terminate-Exploration-Ends}), and
\item an \emph{exploration period} (Lines~\ref{line-alg:maximum-N_a-ee}--\ref{alg-line:exploration-arm-pull}) that applies a UCB rule (as in~\cite{jamieson2014best}) to select the next arm to sample.
\end{itemize}
Although the loop condition is \textbf{True}, the stopping/transition checks ensure that the algorithm cannot run indefinitely; see Section~\ref{sec:Stopping-Conditions-Exploration}.

Within phase $k$, exploration uses tolerance level $\delta_k$ to form UCB/LCB bounds, while exploitation uses tolerance level $\delta/\alpha_k$ and samples only a single candidate arm.
Once $\hat{a}$ takes value in $[K]\cup \{\textsf{None}\}$, Algorithm \ref{alg:SEE-with-Bracket} will terminate. 
Moreover, each phase enforces explicit sampling budgets: during exploration, each arm $a\in B$ can be sampled at most $(C+1)^2\beta_k\log\frac{4K}{\delta_k}$ times, while during exploitation the candidate arm is sampled at most $\frac{(C+3)^2}{(C-1)^2}\beta_k\log\frac{4K\alpha_k}{\delta}$ times. Further details appear in Sections~\ref{sec:Stopping-Conditions-Exploration} and~\ref{sec:process-exploration}.

\subsubsection{Line \ref{alg-line:rough-guess-of-hat-a_k} to \ref{alg-line:Conditions-Terminate-Exploration-Ends}: Conditions on Phase Transition}
\label{sec:Stopping-Conditions-Exploration}
There are three possible conditions to end the exploration period in a phase $k$. The first one is Line \ref{alg-line:rough-guess-of-hat-a_k}, judge whether the algorithm should enter the exploitation period. Intuitively speaking, the inequality $\text{LCB}^{\text{ee}}_{A_t^{\text{ee}}}(\mathcal{H}^{\text{ee}},\delta_k)\geq \mu_0$ provides a recommended qualified arm with confidence $1-\Theta(\delta_k)$. Since $A_t^{\text{ee}}$ is the action we take in the last round, this condition also guarantees $\text{LCB}^{\text{ee}}_{A_t^{\text{ee}}}(\mathcal{H}^{\text{ee}},\delta_k) < \mu_0$ holds before the collecting the latest sample of $A_t^{\text{ee}}$. To further confirm whether arm $A_t^{\text{ee}}$ is qualified, we need to keep pulling arm $A_t^{\text{ee}}$ by collecting samples independent of the exploration periods. Line \ref{alg-line:stop-condition-exploitation} suggests that the algorithm will terminate and output $A_t^{\text{ee}}$, if the LCB calculated by the tolerance level $\delta/\alpha_k$ is above $\mu_0$ during the exploitation process. Before entering the while loop (Line \ref{alg-line:stop-condition-in-exploitation}) in the exploitation period, Line \ref{alg-line:move-latest-samples-into-Q} transfer the latest collected sample of $A_t^{\text{ee}}$ in the exploration period into $Q$. Usage of $Q$ is to guarantee Lemma \ref{lemma:LCB-below-mu0-at-start-of-phase}, requiring $\text{LCB}_a^{\text{ee}} < \mu_0$ holds for all $a\in B$ at the start of phase $k$. Similar to the reason described in the section 4.1 of \cite{pmlr-v267-li25f}, Lemma \ref{lemma:LCB-below-mu0-at-start-of-phase} is required to prove Lemma \ref{lemma:output-corretness-of-positive-instance}, further Theorem \ref{theorem:Parallel-SEE-Etau-upper-bound}. We put more discussion of $Q$ to section \ref{sec:process-exploration}.



The second condition is Line \ref{alg-line:Negative-Output-Condition}, aiming to output $\textsf{None}$ for negative case. To be specific, if the Upper Confidence Bound(UCB) in the exploration periods are all below $\mu_0$, Algorithm \ref{alg:SEE-with-Bracket} will leave the exploration periods in phase $k$. Then, the condition in Line \ref{alg-line:Negative-Output-Condition} determines the next period. If the bracket $B$ is indeed $[K]$ and the tolerance level $\delta_k < \frac{\delta}{3}$, Algorithm \ref{alg:SEE-with-Bracket} will terminate the overall pulling process and output $\textsf{None}$. In other cases, Algorithm \ref{alg:SEE-with-Bracket} will enter the exploration period in the next phase $k+1$. In the case $\max_{a\in B}\mu_a < \mu_0$ and $|B|<K$, it is possible that Algorithm \ref{alg:SEE-with-Bracket} never break the for loop in Line \ref{alg-line:start-of-new-phase}.

The third condition is Line \ref{alg-line:run-out-of-budget}, which is correlated to the budget of the exploration period. In phase $k$, each arm $a\in B$ is assigned an exploration budget of $(C+1)^2\beta_k\log\frac{4K}{\delta_k}$.
We only consider pull $A^{\text{ee}}_{t}$ from those arms whose total pulling times from phase 1 to phase $k$ is below $(C+1)^2\beta_k \log\frac{4K}{\delta_k}$ in Line \ref{line-alg:maximum-N_a-ee}. 
If all arms in $B$ have reached their budget (up to the $-1$ offset in the implementation), the algorithm sets $\hat{a}_k=\textsf{Not Completed}$ and transitions to phase $k+1$. Together, Lines~\ref{alg-line:run-out-of-budget} and~\ref{line-alg:maximum-N_a-ee} ensure that the \textbf{while}-loop in Line~\ref{alg-line:while-loop-start} cannot become an infinite loop.

\subsubsection{Line \ref{line-alg:maximum-N_a-ee} to \ref{alg-line:exploration-arm-pull}: Exploration}
\label{sec:process-exploration}

In phase $k$, Algorithm \ref{alg:SEE-with-Bracket} use tolerance level $\delta_k$ to calculate Upper Confidence Bound(UCB) according to (\ref{eqn:Initialize-Parameters-in-SEE}) and determine the arm $A_t^{\text{ee}}$ by Line \ref{line-alg:maximum-N_a-ee}. Line \ref{alg-line:transfer-Q-to-H-ee-start} to \ref{alg-line:exploration-arm-pull} determine where to collect sample of arm $A_t^{\text{ee}}$. If $Q$ contains one sample of arm $A_t^{\text{ee}}$, Algorithm \ref{alg:SEE-with-Bracket} will transfer this sample to $\mathcal{H}^{\text{ee}}$ without collecting a new sample from distribution $\nu_{A_t^{\text{ee}}}$. If not, Algorithm \ref{alg:SEE-with-Bracket} samples a new random variable from $\nu_{A_t^{\text{ee}}}$.

As noted above, the purpose of $Q$ is to enforce Lemma \ref{lemma:LCB-below-mu0-at-start-of-phase}, which requires $\mathrm{LCB}^{\mathrm{ee}}_a<\mu_0$ for all $a\in B$ at the start of each phase. One might ask whether it suffices to simply discard the most recent exploration sample instead of storing it in $Q$. \textbf{This is not sufficient.} The transfer mechanism preserves the per-arm arrival order of samples, which is needed for the following property.
\begin{lemma}
    \label{lemma:order-of-empirical-mean}
    Given an algorithm copy $\text{alg}$ taking bracket $B$ as input. Condition on reward sequence in the exploration period $\{X_{a,s}^{\text{ee}}\}_{s=1}^{+\infty}$, for each $a\in B$, we can assert $\hat{\mu}_a(\mathcal{H}^{\text{ee}}) \in \{\frac{\sum_{s=1}^t X_{a,s}^{\text{ee}}} {t}\}_{t=1}^{+\infty}$ throughout the pulling process.
\end{lemma}
The key point is that the algorithm never permutes the order of $\{X_{a,s}^{\mathrm{ee}}\}_{s\ge 1}$ for any arm $a\in B$; it only temporarily holds out the most recent sample and later reinserts it without changing order. 

For the analysis, we introduce the random phase indices
\begin{align*}
    \kappa^{\text{ee}}_b=\min \left\{k: \forall t,\forall a\in [K], |\hat{\mu}^{ee}_{a,b}(t)-r^{\nu}_a| < U(t,\frac{\delta_k}{K})\right\}\\
    \kappa^{\text{et}}_b=\min \left\{k: \forall t,\forall a\in [K], |\hat{\mu}^{et}_{a,b}(t)-r^{\nu}_a| < U(t,\frac{\delta_k}{K})\right\}.
\end{align*}
$\hat{\mu}^{ee}_{a,b}(t), \hat{\mu}^{et}_{a,b}(t)$ denote the empirical mean of arm $a$ of the first $t$ collected samples in bracket $B_b$, corresponding to the exploration and exploitation periods separately. More details are in section \ref{sec:properties-of-SEE-Oracle}. Lemma \ref{lemma:order-of-empirical-mean} suggests that for phase $k\geq \kappa^{\text{ee}}_b$, we can guarantee $\text{LCB}_a^{\text{ee}} < \mu_a < \text{UCB}_a^{\text{ee}}$ holds for all $a\in B_b$, which further helps to prove upper bounds of the pulling times for all $a\in B_b$. Simply dropping the latest collected samples is unable to guarantee \ref{lemma:order-of-empirical-mean}, making it much harder to conduct the analysis.

Besides, we also care about the size of $Q$ during the pulling process, as we can observe the following Lemmas.
\begin{lemma}
    \label{lemma:size-tau}
    Given an algorithm copy $\text{alg}$, throughout the pulling process, we have
    \begin{align*}
        t^{\text{ee}} = |\mathcal{H}^{\text{ee}}| + |Q|
    \end{align*}
    always hold.
\end{lemma}
The reason is all the collected samples in the exploration periods are either in $\mathcal{H}^{\text{ee}}$ or $Q$. Notice that we only transfer sample from $\mathcal{H}^{\text{ee}}$ to $Q$ in Line \ref{alg-line:move-latest-samples-into-Q}. In this case, $A_t^{\text{ee}}$ just get pulled in the last round of exploration. Meanwhile, once $Q$ contains a sample for an arm $a\in B$, condition in Line \ref{alg-line:transfer-Q-to-H-ee-start} ensures that Algorithm \ref{alg:SEE-with-Bracket} will not draw a fresh exploration sample for arm $a$ until that held-out sample is transferred back. This yields the following bound.
\begin{lemma}
    \label{lemma:size-of-Q}
    Consider an algorithm copy $\text{alg}$ taking bracket $B$ as input, we have $|Q|\leq |B|$ throughout the pulling process
\end{lemma}

At the end of this subsection, Lemma \ref{lemma:size-of-Q-and-empirical-mean-exploit} is similar to \ref{lemma:order-of-empirical-mean}, but for the exploitation period.
\begin{lemma}
    \label{lemma:size-of-Q-and-empirical-mean-exploit}
    Given an algorithm copy $\text{alg}$ taking bracket $B$ as input. Condition on $\{X_{a,s}^{\text{et}}\}_{s=1}^{+\infty}$, for each $a\in B$, we can assert $\hat{\mu}_a(\mathcal{H}^{\text{et}}) \in \{\frac{\sum_{s=1}^t X_{a,s}^{\text{et}}} {t}\}_{t=1}^{+\infty}$ throughout the pulling process.
\end{lemma}
This is correct, as we collect the samples in the exploitation period sequentially in $\mathcal{H}^{\text{et}}$.

\subsubsection{Adaptation of History Literature}
\label{sec:Adaptation-of-History-Literature}
Compared to the Algorithm SEE in \cite{pmlr-v267-li25f}, our main adaptation consists of two points. The first adaptation is 
we place the stopping/transition conditions (Lines~\ref{alg-line:rough-guess-of-hat-a_k}--\ref{alg-line:Conditions-Terminate-Exploration-Ends}) \emph{before} the exploration step (Lines~\ref{line-alg:maximum-N_a-ee}--\ref{alg-line:exploration-arm-pull}). This ordering is required for Lemma~\ref{lemma:pulling-times-of-a-single-arm-negative-instance}, which ensures that the algorithm does not sample any arm $a$ whose exploration upper confidence bound satisfies $\mathrm{UCB}^{\mathrm{ee}}_a\le \mu_0$. 


The second point is for each phase $k$, we set up a maximum pulling times of each arm in the exploration period, rather than a budget on the total number of exploration pulls across the bracket (Line \ref{alg-line:run-out-of-budget} and \ref{line-alg:maximum-N_a-ee}). This modification is motivated by the randomization of arms within each bracket: with randomized arm assignments, it becomes difficult to track how a \emph{total} budget is required for the bracket to produce a candidate $\hat{a}_k$.

\subsection{Properties of Algorithm \ref{alg:SEE-with-Bracket}}
\label{sec:properties-of-SEE-Oracle}
This subsection closely follows Appendix~B of \cite{pmlr-v267-li25f}, with only minor differences. Since we modify several implementation details (and correspondingly adjust parts of the argument), we include the full proofs for completeness.

In Algorithm \ref{alg:Parallel-SEE-on-Bracket}, we create $\lceil \log_2 K\rceil + 1$ independent copies of Algorithm \ref{alg:SEE-with-Bracket}. For each copy index $b\in [\lceil \log_2 K\rceil + 1]$, we further create two independent sampling copies for exploration and exploitation periods separately. We define
\begin{align*}
    \kappa^{\text{ee}}_b=\min \left\{k: \forall t,\forall a\in [K], |\hat{\mu}^{ee}_{a,b}(t)-r^{\nu}_a| < U(t,\frac{\delta_k}{K})\right\}\\
    \kappa^{\text{et}}_b=\min \left\{k: \forall t,\forall a\in [K], |\hat{\mu}^{et}_{a,b}(t)-r^{\nu}_a| < U(t,\frac{\delta_k}{K})\right\}.
\end{align*}
$\hat{\mu}^{ee}_{a,b}(t), \hat{\mu}^{et}_{a,b}(t)$ denote the empirical mean of arm $a$ of the first $t$ collected samples in bracket $B_b$, corresponding to the exploration and exploitation periods separately. The above definition of $\kappa^{\text{ee}}_b, \kappa^{\text{et}}_b$ are conditioned on the realized reward of $\{X_{a,s}^{\text{ee}}\}_{a=1,s=1}^{K,+\infty}$, $\{X_{a,s}^{\text{et}}\}_{a=1,s=1}^{K,+\infty}$ collected by the algorithm copy $b$. $\kappa^{\text{ee}}_b, \kappa^{\text{et}}_b$ are the minimum phase index such that concentration event holds. Since $B_b$ is determined by an exogenous random permutation, we know $B_b$ is independent with $\kappa^{\text{ee}}_b, \kappa^{\text{et}}_b$. 

In the following, we will analyze the properties of Algorithm \ref{alg:SEE-with-Bracket}, a given bracket $B$. For convenience, given a bracket $B$ with index $b$, we denote $\kappa^{\text{ee}}_B, \kappa^{\text{et}}_B$ as corresponding $\kappa^{\text{ee}}_b, \kappa^{\text{et}}_b$, without emphasizing the index $b$. Further, we denote $\tau_k^{\text{ee}}(B), \tau_k^{\text{et}}(B)$ as the value of $t^{\text{ee}},t^{\text{et}}$ at the end of phase $k$, i.e. each time we enter Line \ref{alg-line:phase-ends} in Algorithm \ref{alg:SEE-with-Bracket}. 

\begin{lemma}
    \label{lemma:prob-upper-bound-for-kappa}
    For any $B$, we have $\Pr(\kappa^{\text{ee}}_B\geq k)\leq \frac{\pi^2}{6}\delta_{k-1}$, $\Pr(\kappa^{\text{ee}}_B\geq k)\leq \frac{\pi^2}{6}\frac{\delta}{\alpha_{k-1}}$
\end{lemma}
\begin{proof}[Proof of Lemma \ref{lemma:prob-upper-bound-for-kappa}]

    By the Lemma \ref{lemma:Adapted-lil-UCB}, we have
    \begin{align*}
        \Pr(\kappa^{\text{ee}}_B\geq k) \leq &\Pr\left(\exists t,\exists a\in [K], |\hat{\mu}^{ee}_a(t)-r^{\nu}_a| \geq \sqrt{\frac{4\log\frac{2K}{\delta_{k-1}}+8\log\log (2t)}{t}}\right)\\
        \leq & \sum_{a=1}^K \Pr\left(\exists t,|\hat{\mu}^{ee}_a(t)-r^{\nu}_a| \geq \sqrt{\frac{4\log\frac{2K}{\delta_{k-1}}+8\log\log (2t)}{t}}\right)\\
        \leq & \sum_{a=1}^K \frac{\pi^2 \delta_{k-1}}{6K}\\
        = & \frac{\pi^2 \delta_{k-1}}{6}
    \end{align*}
    and 
    \begin{align*}
        \Pr(\kappa^{\text{et}}_B\geq k) \leq &\Pr\left(\exists t,\exists a\in [K], |\hat{\mu}^{et}_a(t)-r^{\nu}_a| \geq \sqrt{\frac{4\log\frac{2K\alpha_{k-1}}{\delta}+8\log\log (2t)}{t}}\right)\\
        \leq & \sum_{a=1}^K \Pr\left(\exists t,|\hat{\mu}^{ee}_a(t)-r^{\nu}_a| \geq \sqrt{\frac{4\log\frac{2K\alpha_{k-1}}{\delta}+8\log\log (2t)}{t}}\right)\\
        \leq & \sum_{a=1}^K \frac{\pi^2 \delta}{6K}\\
        = & \frac{\pi^2 \delta}{6\alpha_{k-1}}.
    \end{align*}
\end{proof}

\begin{lemma}
    \label{lemma:correctnesss-output-unqualified-arm}
    Consider an Algorithm copy $b$ with bracket $B$, we have $\Pr(\text{alg copy }b\text{ outputs }\hat{a}\text{ such that } \mu_{\hat{a}} \leq \mu_0) \leq \frac{\pi^2}{6}\frac{\delta}{5}$.
\end{lemma}
\begin{proof}[Proof of Lemma \ref{lemma:correctnesss-output-unqualified-arm}]

    By the Line \ref{alg-line:hata_k-value2} to \ref{alg-line:exploitation-ends}, we know
    \begin{align*}
        & \Pr(\text{alg copy }b\text{ outputs }\hat{a}\text{ such that } \mu_{\hat{a}} < \mu_0)\\
        \leq & \Pr\left(\exists a,s.t.\mu_a<\mu_0, t\in \mathbb{N}, \hat{\mu}_{a}^{\text{et}}(t) -\frac{\sqrt{2 \cdot 2^{\lceil\log_2 t\rceil^+}\log\frac{2K\alpha_1 (\lceil\log_2t\rceil^+)^2}{\delta}}}{t} \geq \mu_0\right)\\
        \leq & \sum_{a:\mu_a < \mu_0}\Pr\left(t\in \mathbb{N}, \hat{\mu}_{a}^{\text{et}}(t) -\frac{\sqrt{2 \cdot 2^{\lceil\log_2 t\rceil^+}\log\frac{2K\alpha_1 (\lceil\log_2t\rceil^+)^2}{\delta}}}{t} \geq \mu_a\right)\\
        \leq & \sum_{a:\mu_a \leq \mu_0}\frac{\pi^2}{6}\frac{\delta}{\alpha_1 K}\\
        \leq & \frac{\pi^2}{6}\frac{\delta}{5}.
    \end{align*}
    The second last line is by the Lemma \ref{lemma:Adapted-lil-UCB}.
\end{proof}

\begin{lemma}
    \label{lemma:correctnesss-output-None-while-qualified-exists}
    Assume $\mu_1>\mu_0$. Consider an Algorithm copy $\lceil \log_2 K\rceil + 1$, whose allocated bracket is $[K]$. We have $\Pr(\text{alg copy }\lceil \log_2 K\rceil + 1\text{ outputs }\textsf{None}) \leq \frac{\pi^2}{6}\frac{\delta}{3}$.
\end{lemma}
\begin{proof}[Proof of Lemma \ref{lemma:correctnesss-output-None-while-qualified-exists}]

    By the condition in Line \ref{alg-line:Negative-branch-output}, we know
    \begin{align*}
        & \Pr(\text{alg copy }\lceil \log_2 K\rceil + 1\text{ outputs }\textsf{None}) \\
        \leq & \Pr(\exists t\in\mathbb{N}, \delta_k\leq \frac{\delta}{3}, \hat{\mu}_{1}^{\text{ee}}(t) + \frac{\sqrt{2 \cdot 2^{\lceil\log_2 t\rceil^+}\log\frac{2K (\lceil\log_2t\rceil^+)^2}{\delta_k}}}{t} \leq \mu_0) \\
        \leq & \Pr(\exists t\in\mathbb{N}, \hat{\mu}_{1}^{\text{ee}}(t) + \frac{\sqrt{2 \cdot 2^{\lceil\log_2 t\rceil^+}\log\frac{2K (\lceil\log_2t\rceil^+)^2}{\delta/3}}}{t} \leq \mu_0) \\
        \leq & \frac{\pi^2}{6}\frac{\delta}{3}.
    \end{align*}
    The last line is by the Lemma \ref{lemma:Adapted-lil-UCB}.
\end{proof}

\begin{lemma}
    \label{lemma:budget-upper-bound-for-N_a-tau_ee-tau_et}
     For any $B$, and any $a\in B$, we have $N_a^{\text{ee}, B}\Big(\tau_k^{\text{ee}}(B)\Big)\leq (C+1)^2\beta_k\log\frac{4K}{\delta_k}$.
\end{lemma}
\begin{proof}[Proof of Lemma \ref{lemma:budget-upper-bound-for-N_a-tau_ee-tau_et}]

    We use induction on $k\in \mathbb{N}$. For $k=1$, by the Line \ref{line-alg:maximum-N_a-ee} in Algorithm \ref{alg:SEE-with-Bracket}, we konw
    \begin{align*}
        N_a^{\text{ee}, B}\Big(\tau_1^{\text{ee}}(B)\Big) \leq (C+1)^2\beta_1\log\frac{4K}{\delta_1}.
    \end{align*}
    
    If the conclusion holds for $k$, then for $k+1$, we can still conclude 
    \begin{align*}
        N_a^{\text{ee}, B}\Big(\tau_k^{\text{ee}}(B)\Big)\leq & \max\left\{N_a^{\text{ee}, B}\Big(\tau_{k-1}^{\text{ee}}(B)\Big), (C+1)^2\beta_k\log\frac{4K}{\delta_k}\right\}\\
        \leq & \max\left\{(C+1)^2\beta_{k-}\log\frac{4K}{\delta_{k-1}}, (C+1)^2\beta_k\log\frac{4K}{\delta_k}\right\}\\
        = & (C+1)^2\beta_k\log\frac{4K}{\delta_k}.
    \end{align*}
    The first step is still caused by the Line \ref{line-alg:maximum-N_a-ee} in Algorithm \ref{alg:SEE-with-Bracket}. The second step is from the conclusion at phase index $k$. The last line shows that the induction holds.
\end{proof}

\begin{lemma}
    \label{lemma:budget-upper-bound-for-tau_ee-tau_et}
    For any $B$, $\tau_k^{\text{ee}}(B)\leq |B|(C+1)^2\beta_k\log\frac{4K}{\delta_k}$, $\tau_k^{\text{et}}(B)\leq \frac{(C+3)^2}{(C-1)^2}\beta_k\log\frac{4K\alpha_k}{\delta}$
\end{lemma}
\begin{proof}[Proof of Lemma \ref{lemma:budget-upper-bound-for-tau_ee-tau_et}]

    Regarding $\tau_k^{\text{ee}}(B)\leq |B|(C+1)^2\beta_k\log\frac{4K}{\delta_k}$, we can directly conclude this from Lemma \ref{lemma:budget-upper-bound-for-N_a-tau_ee-tau_et}.

    Regarding $\tau_k^{\text{et}}(B)\leq \sum_{k'=1}^k\frac{(C+3)^2}{(C-1)^2}\beta_{k'}\log\frac{4K\alpha_{k'}}{\delta}$, we suffice to show $\tau_{k}^{\text{et}}(B)-\tau_{k-1}^{\text{et}}(B)\leq \frac{(C+3)^2}{(C-1)^2}\beta_{k}\log\frac{4K\alpha_{k}}{\delta}$ holds for any phase index $k$. By direct calculation, we have
    \begin{align*}
        & \tau_{k}^{\text{et}}(B)-\tau_{k-1}^{\text{et}}(B)\\
        = & \sum_{a=1}^K\left(N_{\hat{a}_k}^{\text{et}}\Big(\tau_{k}^{\text{et}}(B)\Big)-N_{\hat{a}_{k}}^{\text{et}}\Big(\tau_{k-1}^{\text{et}}(B)\Big)\right)\mathds{1}(\hat{a}_k=a)\\
        \leq & \sum_{a=1}^KN_{\hat{a}_k}^{\text{et}}\Big(\tau_{k}^{\text{et}}(B)\Big)\mathds{1}
        (\hat{a}_k=a)\\
        \leq & \sum_{a=1}^K\frac{(C+3)^2}{(C-1)^2}\beta_{k}\log\frac{4K\alpha_{k}}{\delta}\mathds{1}
        (\hat{a}_k=a)\\
        \leq & \frac{(C+3)^2}{(C-1)^2}\beta_{k}\log\frac{4K\alpha_{k}}{\delta}.
    \end{align*}
    The second last line is by the loop condition in Line \ref{alg-line:stop-condition-in-exploitation} of Algorithm \ref{alg:SEE-with-Bracket}.
\end{proof}

\begin{lemma}
    \label{lemma:LCB-below-mu0-at-start-of-phase}
    For any $B$, at the start of each phase $k$, i.e. each time we enter Line \ref{alg-line:start-of-new-phase} in Algorithm \ref{alg:SEE-with-Bracket}, we have $\text{LCB}_a^{\text{ee}}(\mathcal{H}^{\text{ee}},\delta_k) < \mu_0$ holds for all $a\in B$.
\end{lemma}
\begin{proof}[Proof of Lemma \ref{lemma:LCB-below-mu0-at-start-of-phase}]

    We prove the Lemma by induction on $k$. For $k=1$, this is obviously true, as we denote the $\text{LCB}_a^{\text{ee}}$ as $-\infty$ if the pulling times of arm $a$ is 0.

    If the conclusion holds for the phase $k-1$, we prove the case of $k$ by discussing all the possible values of $\hat{a}_{k-1}$. 
    Since the pulling process enter phase $k$, we can conclude phase $k-1$ may end with taking $\hat{a}_{k-1}$ at the Line \ref{alg-line:hata_k-value1}, \ref{alg-line:hata_k-value3} or \ref{alg-line:hata_k-value2}.

    If $\hat{a}_{k-1}$ takes $\textsf{Not Completed}$ at the Line \ref{alg-line:hata_k-value1}, we have we $\text{LCB}^{\text{ee}}_{a}(\mathcal{H}^{\text{ee}},\delta_k)<\text{LCB}^{\text{ee}}_{a}(\mathcal{H}^{\text{ee}},\delta_{k-1})< \mu_0$, forall $a\in[K]$.

    If $\hat{a}_{k-1}$ takes $\textsf{Not Completed}$ at the Line \ref{alg-line:hata_k-value3}, We can also assert $\text{LCB}_{a}^{\text{ee}}(\mathcal{H}^{\text{ee}},\delta_{k-1})< \mu_0,\forall a\in [K]$. We argue by contradiction. Assume there exists $a\in [K]$ such that $\text{LCB}_{a}^{\text{ee}}(\mathcal{H}^{\text{ee}},\delta_{k-1})\geq \mu_0$ at the end of phase $k-1$. From the induction, since $\text{LCB}_a^{\text{ee}}(\mathcal{H}^{\text{ee}},\delta_{k-1}) < \mu_0$ at the start of phase $k-1$, the algorithm would enter Line \ref{alg-line:hata_k-value2} and $\hat{a}_{k-1}$ cannot take value $\textsf{Not Completed}$. Thus, we can assert $\text{LCB}_{a}^{\text{ee}}(\mathcal{H}^{\text{ee}},\delta_{k-1})< \mu_0,\forall a\in [K]$ holds at the end of phase $k-1$. Further, at the start of phase $k$, we have $\text{LCB}_{a}^{\text{ee}}(\mathcal{H}^{\text{ee}},\delta_{k})<\text{LCB}_{a}^{\text{ee}}(\mathcal{H}^{\text{ee}},\delta_{k-1})< \mu_0$. 

    The remaining work is to prove the conclusion in the case that $\hat{a}_{k-1}\in [K]$ at the Line \ref{alg-line:hata_k-value2}. We can first assert $\text{LCB}_{a}^{\text{ee}}(\mathcal{H}^{\text{ee}},\delta_{k-1})< \mu_0,\forall a\neq \hat{a}_{k-1}$. The reason is from the induction, $\text{LCB}_{a}^{\text{ee}}(\mathcal{H}^{\text{ee}},\delta_{k-1}) < \mu_0$ holds for all $a$ at the start of phase $k-1$. Once there exists an arm $a$ whose $\text{LCB}$ is above $\mu_0$ after an arm pull, the algorithm will enter Line \ref{alg-line:hata_k-value2}, suggesting that there is at most one arm whose $\text{LCB}$ is above $\mu_0$. This observation guarantees that $\text{LCB}_{a}^{\text{ee}}(\mathcal{H}^{\text{ee}},\delta_{k})< \mu_0$.
    

    Then, we turn to analyze $\text{LCB}_{\hat{a}_{k-1}}^{\text{ee}}(\mathcal{H}^{\text{ee}},\delta_{k})$ at the start of phase $k$. Before the execution of Line \ref{alg-line:move-latest-samples-into-Q}, the following inequalities must hold
    \begin{align*}
        \frac{X + \sum_{s=1}^{ N_{\hat{a}_{k-1}}^{\text{ee}}-1 } X_{\hat{a}_{k-1},s} }{N_{\hat{a}_{k-1}}^{\text{ee}}}-C\cdot U\left(N_{\hat{a}_{k-1}}^{\text{ee}}, \frac{\delta_{k-1}}{K}\right) > \mu_0\\
        \frac{\sum_{s=1}^{ N_{\hat{a}_{k-1}}^{\text{ee}}-1 } X_{\hat{a}_{k-1},s} }{ N_{\hat{a}_{k-1}}^{\text{ee}}-1 }-C\cdot U\left( N_{\hat{a}_{k-1}}^{\text{ee}}-1 , \frac{\delta_{k-1}}{K}\right) \leq \mu_0
    \end{align*}
    Since we put the last collected sample $X$ into $Q$ in Line \ref{alg-line:move-latest-samples-into-Q}, we can assert $\hat{\mu}_{\hat{a}_{k-1}}$ before the start of phase $k$ must be $\frac{\sum_{s=1}^{ N_{\hat{a}_{k-1}}^{\text{ee}}-1 } X_{\hat{a}_{k-1},s} }{ N_{\hat{a}_{k-1}}^{\text{ee}}-1 }$. Since $U\left( N_{\hat{a}_{k-1}}^{\text{ee}}-1 , \frac{\delta_{k}}{K}\right)> U\left( N_{\hat{a}_{k-1}}^{\text{ee}}-1 , \frac{\delta_{k-1}}{K}\right)$, we know before the start of phase $k$, $\text{LCB}_{\hat{a}_{k-1}}^{\text{ee}}(\mathcal{H}^{\text{ee}},\delta_{k})< \mu_0$ must hold.
\end{proof}

Following lemma suggests that the exploration period is able to output correct answer, given some concentration event.
\begin{lemma}
    \label{lemma:output-corretness-of-positive-instance}
    For $B$ such that $\max_{a\in B} \mu_{a}\geq \mu_j$ for some $\mu_j>\mu_0$, and phase index $k\geq \kappa^{\text{ee}}_B$, Algorithm \ref{alg:SEE-with-Bracket} either takes $\hat{a}_k\in [K]$ such that $\mu_{\hat{a}_k} \geq \omega \mu_j + (1-\omega)\mu_0$, $\omega=\frac{C-1}{C+3}$ or takes $\hat{a}_k=\textsf{Not Completed}$
\end{lemma}
\begin{proof}[Proof of Lemma \ref{lemma:output-corretness-of-positive-instance}]

    From the assumption $k\geq \kappa^{\text{ee}}_B$, by the Lemma \ref{lemma:order-of-empirical-mean}, we know
    \begin{align}
        \label{eqn:concentration-event-k_geq_kappa}
        \hat{\mu}_a(\mathcal{H}^{\text{ee}}) - U(N_a^{\text{ee}}(\mathcal{H}^{ee}), \frac{\delta_k}{K}) \leq \mu_a \leq \hat{\mu}_a(\mathcal{H}^{\text{ee}}) + U(N_a^{\text{ee}}(\mathcal{H}^{ee}), \frac{\delta_k}{K})
    \end{align}
    holds for $a\in B$ through out the phase $k$.

    We first show that Algorithm \ref{alg:SEE-with-Bracket} never enters Line \ref{alg-line:Negative-branch-output}. Denote $a^*(B) = \arg\max_{a\in B}\mu_a$. From the assumption, we know $\mu_{a^*(B)} > \mu_0$. Together with the definition $\text{UCB}^{\text{ee}}_a(\mathcal{H}^{ee},\delta_k) = \hat{\mu}_a^{\text{ee}}(\mathcal{H}^{ee}) + U(N_a^{\text{ee}}(\mathcal{H}^{ee}), \frac{\delta_k}{K})$, we know $\text{UCB}^{\text{ee}}_{a^*(B)}(\mathcal{H}^{ee},\delta_k) > \mu_0$ holds through out the phase $k$, suggesting that the condition in Line \ref{alg-line:Negative-Output-Condition} never holds. Meanwhile, to quit the While loop in the Line \ref{alg-line:while-loop-start}, Algorithm \ref{alg:SEE-with-Bracket} must enter one of the Lines in \ref{alg-line:Negative-branch-output}, \ref{alg-line:hata_k-value2}, \ref{alg-line:hata_k-value3}. Then, we can conclude $\hat{a}_k$ only takes value in $[K]\cup \{\textsf{Not Completed}\}$.
    
    The remaining work is to prove if $\hat{a}_k\in [K]$, we have $\mu_{\hat{a}_k} \geq \omega \mu_j + (1-\omega)\mu_0$, $\omega=\frac{C-1}{C+3}$. For simplicity, we denote $\tau_k^{\text{ee}}$ as $\tau_k^{\text{ee}}(B)$. 
    From the outputting rule, we have 
    \begin{align}
        \label{eqn:events-when-output-hatak}
        \begin{split}
            & A_{\tau^{\text{ee}}_k}^{\text{ee}} = \hat{a}_k, \hat{a}_k=\arg\max_{i\in B}\hat{\mu}_{i, N_i^{\text{ee}}(\tau^{\text{ee}}_k-1)} +  U(N_i^{\text{ee}}(\tau^{\text{ee}}_k-1),\frac{\delta_k}{K})\\
            & \hat{\mu}_{\hat{a}_k, N_{\hat{a}}^{\text{ee}}(\tau^{\text{ee}}_k)} - C\cdot U(N_{\hat{a}_k}^{\text{ee}}(\tau^{\text{ee}}_k),\frac{\delta_k}{K}) > \mu_0\\
            & \hat{\mu}_{\hat{a}_k, N_{\hat{a}_k}^{\text{ee}}(\tau^{\text{ee}}_k)-1} - C\cdot U(N_{\hat{a}_k}^{\text{ee}}(\tau^{\text{ee}}_k)-1,\frac{\delta_k}{K})  < \mu_0.
        \end{split}
    \end{align}
    where $\hat{\mu}_{i, N_i^{\text{ee}}(\tau^{\text{ee}}_k-1)}$ denotes the empirical mean value of arm i of the first $N_i^{\text{ee}}(\tau^{\text{ee}}_k-1)$ samples. The third line is by the Lemma \ref{lemma:LCB-below-mu0-at-start-of-phase}, suggesting $\text{LCB}_{\hat{a}_k}^{\text{ee}} < \mu_0$ must hold before the latest collection of the sample. 

    We consider two cases, $N_{\hat{a}}^{\text{ee}}(\tau^{\text{ee}}_k)=1$ or $N_{\hat{a}}^{\text{ee}}(\tau^{\text{ee}}_k)\geq 2$. If $N_{\hat{a}}^{\text{ee}}(\tau^{\text{ee}}_k)=1$, we have
    \begin{align*}
        & \mu_0\leq \hat{\mu}_{\hat{a}_k, 1} - C\cdot U(1,\frac{\delta_k}{K})\\
        \Rightarrow & \mu_0\leq \mu_{\hat{a}_k} - (C-1)\cdot U(1,\frac{\delta_k}{K})\\
        \Rightarrow & \mu_0\leq \mu_{\hat{a}_k} - 2(C-1)\\
        \Leftrightarrow & \mu_0+2(C-1)\leq \mu_{\hat{a}_k}.
    \end{align*}
    The second last line is from the definition of (\ref{eqn:definition-U_t_delta}), suggesting that $U(1,\frac{\delta_k}{K})\geq 2$. From the assumption $\{\mu_a\}_{a=1}^K\in [0, 1]^K$ and $C+1$, we know
    $\omega \mu_j + (1-\omega)\mu_0 = \frac{C-1}{C+3}\mu_j + (1-\frac{C-1}{C+3})\mu_0 < \frac{C-1}{4}+\mu_0 < \mu_0+ 2(C-1)$. We complete the proof.

    If $N_{\hat{a}}^{\text{ee}}(\tau^{\text{ee}}_k)\geq 2$, we argue by contradiction. Assume $\mu_{\hat{a}_k} < \omega\mu_j+(1-\omega)\mu_0$. As we take $U(t, \frac{\delta}{K}) = \frac{\sqrt{2\cdot2^{\max\{\lceil\log_2 t\rceil, 1 \}}\log\frac{2 K(\lceil\log_2t\rceil)^2}{\delta}}}{t}$, we get
    \begin{align*}
        & \hat{\mu}_{\hat{a}_k, N_{\hat{a}}^{\text{ee}}(\tau^{\text{ee}}_k)} - C\cdot U(N_{\hat{a}_k}^{\text{ee}}(\tau^{\text{ee}}_k),\frac{\delta_k}{K}) > \mu_0\\
        \Leftrightarrow & \hat{\mu}_{\hat{a}_k, N_{\hat{a}_k}^{\text{ee}}(\tau^{\text{ee}}_k)} - C\frac{\sqrt{2\cdot 2^{\max\{\lceil\log_2 N_{\hat{a}_k}^{\text{ee}}(\tau^{\text{ee}}_k)\rceil, 1 \}}\log\frac{2 K(\lceil\log_2 N_{\hat{a}_k}^{\text{ee}}(\tau^{\text{ee}}_k)\rceil)^2}{\delta}}}{N_{\hat{a}_k}^{\text{ee}}(\tau^{\text{ee}}_k)}> \mu_0\\
        \stackrel{(\ref{eqn:concentration-event-k_geq_kappa})}{\Rightarrow} & \mu_{\hat{a}_k} - (C-1)\frac{\sqrt{2\cdot 2^{\max\{\lceil\log_2 N_{\hat{a}_k}^{\text{ee}}(\tau^{\text{ee}}_k)\rceil, 1 \}}\log\frac{2 K(\lceil\log_2 N_{\hat{a}_k}^{\text{ee}}(\tau^{\text{ee}}_k)\rceil)^2}{\delta}}}{N_{\hat{a}_k}^{\text{ee}}(\tau^{\text{ee}}_k)}> \mu_0\\
        \Rightarrow & \omega\mu_{a^*(B)}+(1-\omega)\mu_0 - (C-1)\frac{\sqrt{2\cdot 2^{\max\{\lceil\log_2 N_{\hat{a}_k}^{\text{ee}}(\tau^{\text{ee}}_k)\rceil, 1 \}}\log\frac{2 K(\lceil\log_2 N_{\hat{a}_k}^{\text{ee}}(\tau^{\text{ee}}_k)\rceil)^2}{\delta}}}{N_{\hat{a}_k}^{\text{ee}}(\tau^{\text{ee}}_k)} > \mu_0\\
        \Leftrightarrow & \omega(\mu_{a^*(B)}-\mu_0) > (C-1)\frac{\sqrt{2\cdot 2^{\max\{\lceil\log_2 N_{\hat{a}_k}^{\text{ee}}(\tau^{\text{ee}}_k)\rceil, 1 \}}\log\frac{2 K(\lceil\log_2 N_{\hat{a}_k}^{\text{ee}}(\tau^{\text{ee}}_k)\rceil)^2}{\delta}}}{N_{\hat{a}_k}^{\text{ee}}(\tau^{\text{ee}}_k)}\\
        \Leftrightarrow & \frac{2\omega(\mu_{a^*(B)}-\mu_0)}{C-1} >\frac{2\sqrt{2\cdot 2^{\max\{\lceil\log_2 N_{\hat{a}_k}^{\text{ee}}(\tau^{\text{ee}}_k)\rceil, 1 \}}\log\frac{2 K(\lceil\log_2 N_{\hat{a}_k}^{\text{ee}}(\tau^{\text{ee}}_k)\rceil)^2}{\delta}}}{N_{\hat{a}_k}^{\text{ee}}(\tau^{\text{ee}}_k)}.
    \end{align*}
    By $\omega = \frac{C-1}{C+3}$, we have $\frac{(1-\omega)}{2} = \frac{1-\frac{C-1}{C+3}}{2}=\frac{4}{2(C+3)} = \frac{2\omega}{C-1}$, we have
    \begin{align}
        \frac{(1-\omega)(\mu_{a^*(B)}-\mu_0)}{2}>\frac{2\sqrt{2\cdot 2^{\max\{\lceil\log_2 N_{\hat{a}_k}^{\text{ee}}(\tau^{\text{ee}}_k)\rceil, 1 \}}\log\frac{2 K(\lceil\log_2 N_{\hat{a}_k}^{\text{ee}}(\tau^{\text{ee}}_k)\rceil)^2}{\delta}}}{N_{\hat{a}_k}^{\text{ee}}(\tau^{\text{ee}}_k)}.\label{eqn:1-omega-ineqn-sqrtNa-divide-Na}
    \end{align}
    Notice that
    \begin{align*}
        & \frac{\sqrt{2\cdot 2^{\max\{\lceil\log_2 N_{\hat{a}_k}^{\text{ee}}(\tau^{\text{ee}}_k)-1\rceil, 1 \}}\log\frac{2 K(\lceil\log_2 N_{\hat{a}_k}^{\text{ee}}(\tau^{\text{ee}}_k)-1\rceil)^2}{\delta}}}{N_{\hat{a}_k}^{\text{ee}}(\tau^{\text{ee}}_k)-1} \\
        \leq & \frac{N_{\hat{a}_k}^{\text{ee}}(\tau^{\text{ee}}_k)}{N_{\hat{a}_k}^{\text{ee}}(\tau^{\text{ee}}_k)-1}\frac{\sqrt{2\cdot 2^{\max\{\lceil\log_2 N_{\hat{a}_k}^{\text{ee}}(\tau^{\text{ee}}_k)\rceil, 1 \}}\log\frac{2 K(\lceil\log_2 N_{\hat{a}_k}^{\text{ee}}(\tau^{\text{ee}}_k)\rceil)^2}{\delta}}}{N_{\hat{a}_k}^{\text{ee}}(\tau^{\text{ee}}_k)} \\
        \leq & \frac{2\sqrt{2\cdot 2^{\max\{\lceil\log_2 N_{\hat{a}_k}^{\text{ee}}(\tau^{\text{ee}}_k)\rceil, 1 \}}\log\frac{2 K(\lceil\log_2 N_{\hat{a}_k}^{\text{ee}}(\tau^{\text{ee}}_k)\rceil)^2}{\delta}}}{N_{\hat{a}_k}^{\text{ee}}(\tau^{\text{ee}}_k)},
    \end{align*}
    Together with (\ref{eqn:1-omega-ineqn-sqrtNa-divide-Na}), we can conclude
    \begin{align}
        & \hat{\mu}_{\hat{a}_k, N_{\hat{a}}^{\text{ee}}(\tau^{\text{ee}}_k)} - C\cdot U(N_{\hat{a}_k}^{\text{ee}}(\tau^{\text{ee}}_k),\frac{\delta_k}{K}) > \mu_0\notag\\
        \Rightarrow & \frac{(1-\omega)(\mu_{a^*(B)}-\mu_0)}{2} \geq \frac{\sqrt{2\cdot 2^{\max\{\lceil\log_2 N_{\hat{a}_k}^{\text{ee}}(\tau^{\text{ee}}_k)-1\rceil, 1 \}}\log\frac{2 K(\lceil\log_2 N_{\hat{a}_k}^{\text{ee}}(\tau^{\text{ee}}_k)-1\rceil)^2}{\delta}}}{N_{\hat{a}_k}^{\text{ee}}(\tau^{\text{ee}}_k)-1}.\label{eqn:output-hata-mean-N_hata-large-enough}
    \end{align}
    On the other hands,
    \begin{align}
        &\hat{\mu}_{\hat{a}_k, N_{\hat{a}_k}^{\text{ee}}(\tau^{\text{ee}}_k)} +  U(N_{\hat{a}_k}^{\text{ee}}(\tau^{\text{ee}}_k-1),\frac{\delta_k}{K}) \leq \mu_{a^*(B)}\notag\\
        \stackrel{(\ref{eqn:concentration-event-k_geq_kappa})}{\Leftarrow} & \mu_{\hat{a}_k} +  2U(N_{\hat{a}_k}^{\text{ee}}(\tau^{\text{ee}}_k-1),\frac{\delta_k}{K}) \leq \mu_{a^*(B)}\notag\\
        \Leftarrow & \omega\mu_{a^*(B)}+(1-\omega)\mu_0+2U(N_{\hat{a}_k}^{\text{ee}}(\tau^{\text{ee}}_k-1),\frac{\delta_k}{K}) \leq \mu_{a^*(B)}\notag\\
        \Leftrightarrow & \frac{2\sqrt{2\cdot 2^{\max\{\lceil\log_2 N_{\hat{a}_k}^{\text{ee}}(\tau^{\text{ee}}_k)-1\rceil, 1 \}}\log\frac{2 K(\lceil\log_2 N_{\hat{a}_k}^{\text{ee}}(\tau^{\text{ee}}_k)-1\rceil)^2}{\delta}}}{N_{\hat{a}_k}^{\text{ee}}(\tau^{\text{ee}}_k)-1} \leq (1-\omega)(\mu_{a^*(B)}-\mu_0)\notag\\
        \Leftrightarrow & \frac{\sqrt{2\cdot 2^{\max\{\lceil\log_2 N_{\hat{a}_k}^{\text{ee}}(\tau^{\text{ee}}_k)-1\rceil, 1 \}}\log\frac{2 K(\lceil\log_2 N_{\hat{a}_k}^{\text{ee}}(\tau^{\text{ee}}_k)-1\rceil)^2}{\delta}}}{N_{\hat{a}_k}^{\text{ee}}(\tau^{\text{ee}}_k)-1} \leq \frac{(1-\omega)(\mu_{a^*(B)}-\mu_0)}{2}.\label{eqn:hat-a_k-cannot-pull-since-N_hata-is-large}
    \end{align}
    
    Combining (\ref{eqn:output-hata-mean-N_hata-large-enough}) and (\ref{eqn:hat-a_k-cannot-pull-since-N_hata-is-large}), we can conclude
    \begin{align*}
        & \hat{\mu}_{\hat{a}_k, N_{\hat{a}_k}^{\text{ee}}(\tau)} - C\cdot U(N_{\hat{a}_k}^{\text{ee}}(\tau^{\text{ee}}_k), \frac{\delta_k}{K}) > \mu_0\\
        \Rightarrow & \hat{\mu}_{\hat{a}_k, N_{\hat{a}_k}(\tau^{\text{ee}}-1)} +  U(N_{\hat{a}_k}^{\text{ee}}(\tau^{\text{ee}}_k-1), \frac{\delta_k}{K}) \leq \mu_{a^*(B)}\\
        \stackrel{(\ref{eqn:concentration-event-k_geq_kappa})}{\Rightarrow} & \hat{\mu}_{\hat{a}_k, N_{\hat{a}_k}^{\text{ee}}(\tau^{\text{ee}}-1)} +  U(N_{\hat{a}_k}^{\text{ee}}(\tau^{\text{ee}}_k-1), \frac{\delta_k}{K})< \hat{\mu}_{a^*(B), N_{a^*(B)}^{\text{ee}}(\tau^{\text{ee}}-1)} +  U(N_{a^*(B)}^{\text{ee}}(\tau^{\text{ee}}_k-1), \frac{\delta_k}{K}) \\
        \Rightarrow & A_{\tau^{\text{ee}}_k}^{\text{ee}}\neq \hat{a}_k.
    \end{align*}
    The last line suggests that under the assumption $\mu_{\hat{a}_k} < \omega\mu_j+(1-\omega)\mu_0$, we can conclude $A_{\tau^{\text{ee}}_k}^{\text{ee}}\neq \hat{a}_k$, which contradicts with (\ref{eqn:events-when-output-hatak}). And we have proved$\mu_{\hat{a}} \geq \omega\mu_j+(1-\omega)\mu_0$.
\end{proof}

\begin{lemma}
    \label{lemma:pulling-times-of-a-single-arm-positive-instance}
    For $B$ such that $\max_{a'\in B} \mu_{a'}\geq \mu_j$ for some $\mu_j>\mu_0$, phase index $k\geq \kappa^{\text{ee}}_B$, for each arm $a\in B$, we have
    \begin{align*}
        & N_a^{\text{ee}, B}\Big(\tau_k^{\text{ee}}(B)\Big)\\
        \leq & \max\left\{N_a^{\text{ee}, B}\Big(\tau_{\kappa^{\text{ee}}_B-1}^{\text{ee}}(B)\Big), \frac{29(C+1)^2\left(\log\frac{2K}{\delta_k} + \log\log \frac{24(C+1)^2}{\max\{\Delta_{a,j}^2, \Delta_{a,0}^2\}}\right)}{\max\{\Delta_{a,j}^2, \Delta_{a,0}^2\}}\right\}
    \end{align*}
\end{lemma}
\begin{proof}[Proof of Lemma \ref{lemma:pulling-times-of-a-single-arm-positive-instance}]

    We first prove an intermediate conclusion, which is for each arm $a\in B$, we have
    \begin{equation}
        \label{eqn:recursive-pulling-k_k-1}
        \begin{aligned}
            & N_a^{\text{ee}, B}\Big(\tau_k^{\text{ee}}(B)\Big)\\
            \leq & \max\left\{N_a^{\text{ee}, B}\Big(\tau_{k-1}^{\text{ee}}(B)\Big), \frac{29(C+1)^2\left(\log\frac{2K}{\delta_k} + \log\log \frac{24(C+1)^2}{\max\{\Delta_{a,j}^2, \Delta_{a,0}^2\}}\right)}{\max\{\Delta_{a,j}^2, \Delta_{a,0}^2\}}\right\}
        \end{aligned}
    \end{equation}
    holds for for any phase index $k\geq \kappa^{\text{ee}}_B$. If (\ref{eqn:recursive-pulling-k_k-1}) holds, we can use induction to conclude
    \begin{align*}
        & N_a^{\text{ee}, B}\Big(\tau_k^{\text{ee}}(B)\Big)\\
        \leq & \max\Bigg\{N_a^{\text{ee}, B}\Big(\tau_{\kappa^{\text{ee}}_B-1}^{\text{ee}}(B)\Big), \frac{29(C+1)^2\left(\log\frac{2K}{\delta_{k-1}} + \log\log \frac{24(C+1)^2}{\max\{\Delta_{a,j}^2, \Delta_{a,0}^2\}}\right)}{\max\{\Delta_{a,j}^2, \Delta_{a,0}^2\}} \\
        & \frac{29(C+1)^2\left(\log\frac{2K}{\delta_k} + \log\log \frac{24(C+1)^2}{\max\{\Delta_{a,j}^2, \Delta_{a,0}^2\}}\right)}{\max\{\Delta_{a,j}^2, \Delta_{a,0}^2\}}\Bigg\}\\
        = & \max\left\{N_a^{\text{ee}, B}\Big(\tau_{\kappa^{\text{ee}}_B-1}^{\text{ee}}(B)\Big), \frac{29(C+1)^2\left(\log\frac{2K}{\delta_k} + \log\log \frac{24(C+1)^2}{\max\{\Delta_{a,j}^2, \Delta_{a,0}^2\}}\right)}{\max\{\Delta_{a,j}^2, \Delta_{a,0}^2\}}\right\},
    \end{align*}
    which proves the Lemma. Thus, we suffice to prove the intermediate conclusion (\ref{eqn:recursive-pulling-k_k-1}).

    From the assumption $k\geq \kappa^{\text{ee}}_B$, we know (\ref{eqn:concentration-event-k_geq_kappa}) still holds. Fix the arm $a\in B$, we prove the claim by discussing the value of $\mu_a$. We consider two cases, $\mu_a\geq \frac{\mu_j+\mu_0}{2}$ or $\mu_a < \frac{\mu_j+\mu_0}{2}$.

    For arm $a$ such that $\mu_a\geq \frac{\mu_j+\mu_0}{2}$ holds, we suffice to prove
    \begin{align*}
        & N_a^{\text{ee}, B}\Big(\tau_k^{\text{ee}}(B)\Big)\\
        \leq & \max\left\{N_a^{\text{ee}, B}\Big(\tau_{k-1}^{\text{ee}}(B)\Big), \frac{29(C+1)^2\left(\log\frac{2K}{\delta_k} + \log\log \frac{1}{\Delta_{a,0}^2}\right)}{\Delta_{a,0}^2}\right\}.
    \end{align*}
    We can conduct direct calculation, as
    \begin{align*}
        & t \geq \frac{28(C+1)^2\left(\log\frac{2K}{\delta_k} + \log\log \frac{24(C+1)^2}{\Delta_{a,0}^2}\right)}{\Delta_{a,0}^2}\\
        \stackrel{\text{Lemma }\ref{lemma:inequality-application-t-loglogt} }{\Rightarrow} & (C+1)\sqrt{\frac{4\log\frac{2K(\log_2 (2t))^2}{\delta_k}}{t}} \leq \Delta_{a,0}\\
        \Leftrightarrow & \mu_a - (C+1)\sqrt{\frac{4\log\frac{2K(\log_2 (2t))^2}{\delta_k}}{t}} > \mu_0\\
        \Rightarrow & \mu_a - (C+1)\frac{\sqrt{2 \cdot 2^{\lceil\log_2 t\rceil^+}\log\frac{2 K(\lceil\log_2t\rceil^+)^2}{\delta_k}}}{t} > \mu_0\\
        \stackrel{(\ref{eqn:concentration-event-k_geq_kappa})}{\Rightarrow} & \hat{\mu}_{a, t} - C\cdot\frac{\sqrt{2 \cdot 2^{\lceil\log_2 t\rceil^+}\log\frac{2 K(\lceil\log_2t\rceil^+)^2}{\delta_k}}}{t} > \mu_0.
    \end{align*}
    By the Lemma \ref{lemma:LCB-below-mu0-at-start-of-phase}, we know at the start of phase $k$, $\text{LCB}_a^{\text{ee}} < \mu_0$, suggesting that $N_a^{\text{ee}, B}\Big(\tau_{k-1}^{\text{ee}}(B)\Big)\leq \frac{28(C+1)^2\left(\log\frac{2K}{\delta_k} + \log\log \frac{24(C+1)^2}{\Delta_{a,0}^2}\right)}{\Delta_{a,0}^2}$. Then, we can assert 
    \begin{align*}
        N_a^{\text{ee}, B}\Big(\tau_{k}^{\text{ee}}(B)\Big)<\frac{28(C+1)^2\left(\log\frac{2K}{\delta_k} + \log\log \frac{24(C+1)^2}{\Delta_{a,0}^2}\right)}{\Delta_{a,0}^2}+1.
    \end{align*}
    If not, the algorithm will output $\hat{a}_k=a$ before $N_a^{\text{ee}, B}\Big(\tau_{k}^{\text{ee}}(B)\Big)\geq \frac{28(C+1)^2\left(\log\frac{2K}{\delta_k} + \log\log \frac{24(C+1)^2}{\Delta_{a,0}^2}\right)}{\Delta_{a,0}^2}+1$ holds.

    For arm $a$ such that $\mu_a< \frac{\mu_j+\mu_0}{2}$ holds, we suffice to prove
    \begin{align*}
        & N_a^{\text{ee}, B}\Big(\tau_k^{\text{ee}}(B)\Big)\\
        \leq & \max\left\{N_a^{\text{ee}, B}\Big(\tau_{k-1}^{\text{ee}}(B)\Big), \frac{29(C+1)^2\left(\log\frac{4K}{\delta_k} + \log\log \frac{24(C+1)^2}{\Delta_{a,j}^2}\right)}{\Delta_{a,j}^2}\right\}.
    \end{align*}
    We can also conduct direct calculation, with
    \begin{align*}
        & t \geq \frac{28(C+1)^2\left(\log\frac{2K}{\delta_k} + \log\log \frac{24(C+1)^2}{\Delta_{a,j}^2}\right)}{\Delta_{a,j}^2}\\
        \stackrel{\text{Lemma }\ref{lemma:inequality-application-t-loglogt} }{\Rightarrow} & (C+1)\sqrt{\frac{4\log\frac{2K(\log_2 (2t))^2}{\delta_k}}{t}} \leq \Delta_{a,j}\\
        \stackrel{C>1}{\Rightarrow} & \mu_a + 2\sqrt{\frac{4\log\frac{2K(\log_2 (2t))^2}{\delta_k}}{t}} \leq \mu_j\\
        \Rightarrow & \mu_a + 2\frac{\sqrt{2 \cdot 2^{\lceil\log_2 t\rceil^+}\log\frac{2 K(\lceil\log_2t\rceil^+)^2}{\delta_k}}}{t} \leq \mu_j\\
        \stackrel{(\ref{eqn:concentration-event-k_geq_kappa})}{\Rightarrow} & \hat{\mu}_{a, t} +\frac{\sqrt{2 \cdot 2^{\lceil\log_2 t\rceil^+}\log\frac{2 K(\lceil\log_2t\rceil^+)^2}{\delta_k}}}{t} \leq \mu_j
    \end{align*}
    If $N_a^{\text{ee}, B}\Big(\tau_{k-1}^{\text{ee}}(B)\Big)\geq \frac{28(C+1)^2\left(\log\frac{4K}{\delta_k} + \log\log \frac{24(C+1)^2}{\Delta_{a,j}^2}\right)}{\Delta_{a,j}^2}+1$, we know $\text{UCB}_a^{\text{ee}}(\mathcal{H}^{\text{ee}}, \frac{\delta_k}{K}) < \mu_j \leq \mu_{a^*(B)}\leq \text{UCB}_{a^*(B)}^{\text{ee}}(\mathcal{H}^{\text{ee}}, \frac{\delta_k}{K})$ holds through out the phase $k$, and arm $a$ will never get pulled. We can assert $N_a^{\text{ee}, B}\Big(\tau_{k}^{\text{ee}}(B)\Big)=N_a^{\text{ee}, B}\Big(\tau_{k-1}^{\text{ee}}(B)\Big)$.

    If $N_a^{\text{ee}, B}\Big(\tau_{k-1}^{\text{ee}}(B)\Big)< \frac{28(C+1)^2\left(\log\frac{4K}{\delta_k} + \log\log \frac{24(C+1)^2}{\Delta_{a,j}^2}\right)}{\Delta_{a,j}^2}+1$, we can assert $N_a^{\text{ee}, B}\Big(\tau_{k}^{\text{ee}}(B)\Big)\leq \frac{29(C+1)^2\left(\log\frac{4K}{\delta_k} + \log\log \frac{24(C+1)^2}{\Delta_{a,j}^2}\right)}{\Delta_{a,j}^2}$ as the arm $a$ never gets pulled once $\text{UCB}_a^{\text{ee}} \leq \mu_{a^*(B)}\leq \text{UCB}_{a^*(B)}^{\text{ee}}$ holds. We complete the proof.
\end{proof}

\begin{lemma}
    \label{lemma:correctness-of-rough-guess}
    For $B$ such that $\max_{a'\in B} \mu_{a'}\geq \mu_j$, and phase index $k\geq \max\{\kappa^{\text{ee}}_B,  \lceil\log_2 \frac{112(C+1)^2}{\Delta_{j,0}^2}\rceil \}$, Algorithm \ref{alg:SEE-with-Bracket} will take $\hat{a}_k\in [K]$ such that $\mu_{\hat{a}_k} \geq \omega \mu_j + (1-\omega)\mu_0$, $\omega=\frac{C-1}{C+3}$.
\end{lemma}
\begin{proof}[Proof of Lemma \ref{lemma:correctness-of-rough-guess}]

    From the proof of Lemma \ref{lemma:output-corretness-of-positive-instance}, we know Algorithm \ref{alg:SEE-with-Bracket} will never meet the condition in Line \ref{alg-line:Negative-Output-Condition}.

    Meanwhile, from the calculation in the proof of Lemma \ref{lemma:pulling-times-of-a-single-arm-positive-instance}, we know
    \begin{align*}
        & t \geq \frac{28(C+1)^2\left(\log\frac{2K}{\delta_k} + \log\log \frac{24(C+1)^2}{\Delta_{a^*(B),0}^2}\right)}{\Delta_{a^*(B),0}^2}\\
        \Rightarrow & \hat{\mu}_{a^*(B), t} -C\cdot \frac{\sqrt{2 \cdot 2^{\lceil\log_2 t\rceil^+}\log\frac{2 K(\lceil\log_2t\rceil^+)^2}{\delta_k}}}{t} > \mu_0
    \end{align*}
    and 
    \begin{align*}
        & k\geq \lceil\log_2 \frac{112(C+1)^2}{\Delta_{j,0}^2}\rceil\\
        \Rightarrow & \beta_k \geq \frac{112(C+1)^2}{\Delta_{j,0}^2}, \log\frac{1}{\delta_k}\geq \log 3 \log_2\frac{112(C+1)^2}{\Delta_{j,0}^2}\\
        \Rightarrow & \frac{1}{2}\beta_k \log\frac{4K}{\delta_k} \geq \frac{56\log\frac{2K}{\delta_k}}{\Delta_{j,0}^2}, \frac{1}{2}\beta_k \log\frac{4K}{\delta_k}\geq \frac{56\log\log \frac{24(C+1)^2}{\Delta_{j,0}^2}}{\Delta_{j,0}^2}\\
        \Rightarrow & \beta_k \log\frac{4K}{\delta_k} \geq \frac{56\left(\log\frac{2K}{\delta_k} + \log\log \frac{24(C+1)^2}{\Delta_{j,0}^2}\right)}{\Delta_{j,0}^2}\\
        \Rightarrow & (C+1)^2\beta_k \log\frac{4K}{\delta_k}-1 \geq \frac{28(C+1)^2\left(\log\frac{2K}{\delta_k} + \log\log \frac{24(C+1)^2}{\Delta_{j,0}^2}\right)}{\Delta_{j,0}^2}\\
        \Rightarrow & (C+1)^2\beta_k \log\frac{4K}{\delta_k}-1 \geq \frac{28(C+1)^2\left(\log\frac{2K}{\delta_k} + \log\log \frac{24(C+1)^2}{\Delta_{a^*(B),0}^2}\right)}{\Delta_{a^*(B),0}^2},
    \end{align*}
    which means the Algorithm will first fulfill the condition in Line \ref{alg-line:rough-guess-of-hat-a_k} before fulfilling the condition in Line \ref{alg-line:run-out-of-budget}. This observation concludes $\hat{a}_k\in [K]$. By the Lemma \ref{lemma:output-corretness-of-positive-instance}, we can further conclude $\mu_{\hat{a}_k} \geq \omega \mu_j + (1-\omega)\mu_0$.
\end{proof}

\begin{lemma}
    \label{lemma:required-budget-for-exploitation}
    For $B$ such that $\max_{a'\in B} \mu_{a'}\geq \mu_j$, and phase index $k\geq \max\{\kappa^{\text{et}}_B,  \lceil\log_2 \frac{448(C+3)^2}{(C-1)^2\Delta_{j,0}^2}\rceil \}$. If Algorithm \ref{alg:SEE-with-Bracket} enters Line \ref{alg-line:hata_k-value2} with $\mu_{\hat{a}_k} \geq \omega \mu_j + (1-\omega)\mu_0$, $\omega=\frac{C-1}{C+3}$, Algorithm \ref{alg:SEE-with-Bracket} will terminate and output $\hat{a}_k$.
\end{lemma}
\begin{proof}[Proof of Lemma \ref{lemma:required-budget-for-exploitation}]

    From the condition $k\geq \kappa^{\text{et}}_B$, by the Lemma \ref{lemma:size-of-Q-and-empirical-mean-exploit}, we know
    \begin{align*}
        \hat{\mu}_a^{\text{et}}(\mathcal{H}^{\text{et}}) - U(N_a^{\text{et}}(\mathcal{H}^{et}), \frac{\delta}{K\alpha_k}) \leq \mu_a \leq \hat{\mu}_a^{\text{et}}(\mathcal{H}^{\text{et}}) + U(N_a^{\text{et}}(\mathcal{H}^{et}), \frac{\delta}{K\alpha_k})
    \end{align*}
    holds for $a\in B$ through out the phase $k$.

    By similar calculation in the proof of Lemma \ref{lemma:pulling-times-of-a-single-arm-positive-instance}, we know
    \begin{align*}
        & t\geq \frac{112\left(\log\frac{2K\alpha_k}{\delta} + \log\log \frac{96}{\omega^2\Delta_{j,0}^2}\right)}{\omega^2\Delta_{j,0}^2}\\
        \Rightarrow & t \geq \frac{112\left(\log\frac{2K\alpha_k}{\delta} + \log\log \frac{96}{\Delta_{\hat{a}_k,0}^2}\right)}{\Delta_{\hat{a}_k,0}^2}\\
        \Rightarrow & \hat{\mu}_{\hat{a}_k, t}^{\text{et}} -\frac{\sqrt{2 \cdot 2^{\lceil\log_2 t\rceil^+}\log\frac{2 K\alpha_k(\lceil\log_2t\rceil^+)^2}{\delta}}}{t} > \mu_0. 
    \end{align*}
    Here $\hat{\mu}_{\hat{a}_k, t}^{\text{et}}$ denotes the empirical mean of first $t$ samples collected in the exploitation period. Notice that
    \begin{align*}
        & k\geq \lceil\log_2 \frac{448(C+3)^2}{(C-1)^2\Delta_{j,0}^2}\rceil\\
        \Rightarrow & \beta_k \geq \frac{448}{\Delta_{j,0}^2}, \log \alpha_k \geq \log\frac{96(C+3)^2}{(C-1)^2\Delta_{j,0}^2}\log 5\\
        \Rightarrow & \beta_k\log\frac{4K\alpha_k}{\delta}\geq \frac{448\log\frac{2K\alpha_k}{\delta}}{\Delta_{j,0}^2}, \beta_k\log\frac{4K\alpha_k}{\delta}\geq \frac{448\log\log \frac{96}{\omega^2\Delta_{j,0}^2}}{\Delta_{j,0}^2}\\
        \Rightarrow & \beta_k\log\frac{4K\alpha_k}{\delta}\geq \frac{224\left(\log\frac{2K\alpha_k}{\delta} + \log\log \frac{96}{\omega^2\Delta_{j,0}^2}\right)}{\Delta_{j,0}^2}\\
        \Rightarrow & \frac{(C+3)^2}{(C-1)^2}\beta_k\log\frac{4K\alpha_k}{\delta}-1 \geq \frac{112\left(\log\frac{2K\alpha_k}{\delta} + \log\log \frac{96}{\omega^2\Delta_{j,0}^2}\right)}{\omega^2\Delta_{j,0}^2}.
    \end{align*}
    The last line suggests that Algorithm \ref{alg:SEE-with-Bracket} would output $\hat{a}_k$ before meeting the stop condition $N_{\hat{a}}^{\text{et}} > \frac{(C+3)^2}{(C-1)^2}\beta_k\log\frac{4K\alpha_k}{\delta}-1$ in Line \ref{alg-line:stop-condition-in-exploitation}.
\end{proof}

\begin{lemma}
    \label{lemma:termination-phase-index-positive-instance}
    For $B$ such that $\max_{a'\in B} \mu_{a'}\geq \mu_j$, Algorithm \ref{alg:SEE-with-Bracket} must terminate no later than the end of phase $\max\{\kappa^{\text{ee}}_B, \kappa^{\text{et}}_B, \lceil\log_2 \frac{448(C+3)^2}{(C-1)^2\Delta_{j,0}^2}\rceil\}$.
\end{lemma}
\begin{proof}[Proof of Lemma \ref{lemma:termination-phase-index-positive-instance}]
    This is a direct conclusion from Lemma \ref{lemma:correctness-of-rough-guess}, \ref{lemma:required-budget-for-exploitation}.
\end{proof}

\begin{lemma}
    \label{lemma:output-corretness-of-negative-instance}
    Consider the largest bracket $B$, i.e. $B=[K]$. If $\max_{a'\in B} \mu_{a'}=\max_{a'\in [K]} \mu_{a'}< \mu_0$, then for phase index $k\geq \kappa^{\text{ee}}_B$, Algorithm \ref{alg:SEE-with-Bracket} either takes $\hat{a}_k=\textsf{None}$ or takes $\hat{a}_k=\textsf{Not Completed}$
\end{lemma}
\begin{proof}[Proof of Lemma \ref{lemma:output-corretness-of-negative-instance}]

    From the assumption $k\geq \kappa^{\text{ee}}_B$, by the Lemma \ref{lemma:order-of-empirical-mean}, we know
    \begin{align}
        \label{eqn:concentration-event-k_geq_kappa_negative}
        \hat{\mu}_a^{\text{ee}}(\mathcal{H}^{\text{ee}}) - U(N_a^{\text{ee}}(\mathcal{H}^{ee}), \frac{\delta_k}{K}) \leq \mu_a \leq \hat{\mu}_a^{\text{ee}}(\mathcal{H}^{\text{ee}}) + U(N_a^{\text{ee}}(\mathcal{H}^{ee}), \frac{\delta_k}{K})
    \end{align}
    holds for $a\in B$ through out the phase $k$.

    The analysis is similar to the Lemma \ref{lemma:output-corretness-of-positive-instance}. To quit the While loop in the Line \ref{alg-line:while-loop-start}, Algorithm \ref{alg:SEE-with-Bracket} must enter one of the Lines in \ref{alg-line:Negative-branch-output}, \ref{alg-line:hata_k-value2}, \ref{alg-line:hata_k-value3}. By the assumption $\mu_a<\mu_0$ holds for all $a\in [K]$ and (\ref{lemma:output-corretness-of-negative-instance}), we know
    \begin{align*}
        & \text{LCB}^{\text{ee}}_{a}(\mathcal{H}^{\text{ee}},\delta_k)\\
        = & \hat{\mu}_a^{\text{ee}}(\mathcal{H}^{\text{ee}}) - C\cdot U(N_a^{\text{ee}}(\mathcal{H}^{ee}), \frac{\delta_k}{K})\\
        \stackrel{C>1}{\leq} & \hat{\mu}_a^{\text{ee}}(\mathcal{H}^{\text{ee}}) -  U(N_a^{\text{ee}}(\mathcal{H}^{ee}),\frac{\delta_k}{K})\\
        \leq & \mu_a\\
        < & \mu_0,
    \end{align*}
    holds through out the phase $k$. The last line also suggests that the condition in Line \ref{alg-line:rough-guess-of-hat-a_k} will never be satisfied in phase $k$. Thus, we can conclude $\hat{a}_k=\textsf{None}\text{ or }\textsf{Not Completed} $.
\end{proof}

\begin{lemma}
    \label{lemma:pulling-times-of-a-single-arm-negative-instance}
    For $B$ such that $|B|=K$ and $\max_{a'\in B} \mu_{a'}< \mu_0$, phase index $k\geq \kappa^{\text{ee}}_B$, for each arm $a\in B$, we have
    \begin{align*}
        & N_a^{\text{ee}, B}\Big(\tau_k^{\text{ee}}(B)\Big)\\
        \leq & \max\left\{N_a^{\text{ee}, B}\Big(\tau_{\kappa^{\text{ee}}_B-1}^{\text{ee}}(B)\Big), \cdot\frac{113\left(\log\frac{4K}{\delta_k} + \log\log \frac{96}{\Delta_{a,0}^2}\right)}{\Delta_{a,0}^2}\right\}
    \end{align*}
\end{lemma}
\begin{proof}[Proof of Lemma \ref{lemma:pulling-times-of-a-single-arm-negative-instance}]

    Similar to the idea of Lemma \ref{lemma:pulling-times-of-a-single-arm-positive-instance}, we first turn to prove an intermediate conclusion
    \begin{equation}
        \label{eqn:recursive-pulling-k_k-1-negative}
        \begin{aligned}
            & N_a^{\text{ee}, B}\Big(\tau_k^{\text{ee}}(B)\Big)\\
        \leq & \max\left\{N_a^{\text{ee}, B}\Big(\tau_{k-1}^{\text{ee}}(B)\Big), \frac{113\left(\log\frac{4K}{\delta_k} + \log\log \frac{96}{\Delta_{a,0}^2}\right)}{\Delta_{a,0}^2}\right\}
        \end{aligned}
    \end{equation}
    holds for all $k\geq\kappa^{\text{ee}}_B $ and $a\in B$. If the above conclusion holds, we can use induction to conclude
    \begin{align*}
        & N_a^{\text{ee}, B}\Big(\tau_k^{\text{ee}}(B)\Big)\\
        \leq & \max\Bigg\{N_a^{\text{ee}, B}\Big(\tau_{\kappa^{\text{ee}}_B-1}^{\text{ee}}(B)\Big), \frac{113\left(\log\frac{4K}{\delta_{k-1}} + \log\log \frac{96}{\Delta_{a,0}^2}\right)}{\Delta_{a,0}^2}, \\
        & \frac{113\left(\log\frac{4K}{\delta_k} + \log\log \frac{96}{\Delta_{a,0}^2}\right)}{\Delta_{a,0}^2}\Bigg\}\\
        = & \max\left\{N_a^{\text{ee}, B}\Big(\tau_{\kappa^{\text{ee}}_B-1}^{\text{ee}}(B)\Big), \frac{113\left(\log\frac{4K}{\delta_k} + \log\log \frac{96}{\Delta_{a,0}^2}\right)}{\Delta_{a,0}^2}\right\}.
    \end{align*}
    which proves the Lemma. Thus, we suffice to prove the intermediate conclusion (\ref{eqn:recursive-pulling-k_k-1-negative}).

    We can conduct direct calculation for each arm $a$.
    \begin{align*}
        & t \geq \frac{28\cdot 2^2\left(\log\frac{2K}{\delta_k} + \log\log \frac{96}{\Delta_{a,0}^2}\right)}{\Delta_{a,0}^2}\\
        \stackrel{\text{Lemma }\ref{lemma:inequality-application-t-loglogt} }{\Rightarrow} & 2\sqrt{\frac{4\log\frac{2K(\log_2 (2t))^2}{\delta_k}}{t}} \leq \Delta_{a,0}\\
        \Leftrightarrow & \mu_a + 2\sqrt{\frac{4\log\frac{2K(\log_2 (2t))^2}{\delta_k}}{t}} \leq \mu_0\\
        \Rightarrow & \mu_a + 2\cdot\frac{\sqrt{2 \cdot 2^{\lceil\log_2 t\rceil^+}\log\frac{2 K(\lceil\log_2t\rceil^+)^2}{\delta_k}}}{t} > \mu_0\\
        \stackrel{(\ref{eqn:concentration-event-k_geq_kappa})}{\Rightarrow} & \hat{\mu}_{a, t} + \frac{\sqrt{2 \cdot 2^{\lceil\log_2 t\rceil^+}\log\frac{2 K(\lceil\log_2t\rceil^+)^2}{\delta_k}}}{t} > \mu_0.
    \end{align*}
    The last line is by the assumption $k\geq \kappa^{\text{ee}}_B$, which means (\ref{eqn:concentration-event-k_geq_kappa}) still holds. 

    If $N_a^{\text{ee}, B}\Big(\tau_{k-1}^{\text{ee}}(B)\Big)\geq \frac{112\left(\log\frac{4K}{\delta_k} + \log\log \frac{96}{\Delta_{a,0}^2}\right)}{\Delta_{a,0}^2}+1$, we know $\text{UCB}_a^{\text{ee}}(\mathcal{H}^{\text{ee}}, \frac{\delta_k}{K}) < \mu_0$ holds at the start of phase $k$. If $a=\arg\max_{a'\in B} \text{UCB}_{a'}^{\text{ee}}(\mathcal{H}^{\text{ee}}, \frac{\delta_k}{K})$ holds during the phase $k$, Algorithm \ref{alg:SEE-with-Bracket} will turn to phase $k+1$ as the condition in the Line \ref{alg-line:Negative-Output-Condition} holds. Thus, we can conclude Algorithm \ref{alg:SEE-with-Bracket} will never pull arm $a$ during the phase $k$. That means if $N_a^{\text{ee}, B}\Big(\tau_{k-1}^{\text{ee}}(B)\Big)\geq \frac{112\left(\log\frac{4K}{\delta_k} + \log\log \frac{96}{\Delta_{a,0}^2}\right)}{\Delta_{a,0}^2}+1$, 
    $N_a^{\text{ee}, B}\Big(\tau_{k}^{\text{ee}}(B)\Big)=N_a^{\text{ee}, B}\Big(\tau_{k-1}^{\text{ee}}(B)\Big)$.

    If $N_a^{\text{ee}, B}\Big(\tau_{k-1}^{\text{ee}}(B)\Big)< \frac{112\left(\log\frac{4K}{\delta_k} + \log\log \frac{96}{\Delta_{a,0}^2}\right)}{\Delta_{a,0}^2}+1$. We can also assert $\text{UCB}_a^{\text{ee}}(\mathcal{H}^{\text{ee}}, \frac{\delta_k}{K}) \leq \mu_0$ will hold if $N_a^{\text{ee}, B} \geq \frac{112\left(\log\frac{4K}{\delta_k} + \log\log \frac{96}{\Delta_{a,0}^2}\right)}{\Delta_{a,0}^2}+1$ holds in the futuer rounds. Similar to the above argument, if $\text{UCB}_a^{\text{ee}}(\mathcal{H}^{\text{ee}}, \frac{\delta_k}{K}) \leq \mu_0$, $a=\arg\max_{a'}\text{UCB}_{a'}^{\text{ee}}(\mathcal{H}^{\text{ee}}, \frac{\delta_k}{K})$ both holds, the algorithm would turn to phase $k+1$ or terminate because of meeting the condition in the Line \ref{alg-line:Negative-Output-Condition}. That means once $N_a^{\text{ee}, B} \geq \frac{112\left(\log\frac{4K}{\delta_k} + \log\log \frac{96}{\Delta_{a,0}^2}\right)}{\Delta_{a,0}^2}+1$ holds, Algorithm \ref{alg:SEE-with-Bracket} will never pull arm $a$ in the remaining phase $k$. We can assert $N_a^{\text{ee}, B}\Big(\tau_{k}^{\text{ee}}(B)\Big)\leq \frac{113\left(\log\frac{4K}{\delta_k} + \log\log \frac{96}{\Delta_{a,0}^2}\right)}{\Delta_{a,0}^2}$.
\end{proof}

\begin{lemma}
    \label{lemma:termination-phase-index-negative-instance}
    For $B$ such that $|B|=K$ and $\max_{a'\in B} \mu_{a'}< \mu_0$, Algorithm \ref{alg:SEE-with-Bracket} must terminate no later than the end of phase $\max\{\kappa^{\text{ee}}_B, \lceil \log_2 \frac{226}{\Delta_{1,0}^2}\rceil, \lceil\log_3 \frac{3}{\delta}\rceil\}$.
\end{lemma}
\begin{proof}[Proof of Lemma \ref{lemma:termination-phase-index-negative-instance}]

    From the calculation in the proof of Lemma \ref{lemma:pulling-times-of-a-single-arm-negative-instance}, we know
    \begin{align*}
        & t \geq \frac{112\left(\log\frac{2K}{\delta_k} + \log\log \frac{96}{\Delta_{a,0}^2}\right)}{\Delta_{a,0}^2}\\
        \Rightarrow & \hat{\mu}_{a, t} +\frac{\sqrt{2 \cdot 2^{\lceil\log_2 t\rceil^+}\log\frac{2 K(\lceil\log_2t\rceil^+)^2}{\delta_k}}}{t} > \mu_0
    \end{align*}
    for arm $a\in B=[K]$. Notice that for any $a\in [K]$, we have
    \begin{align*}
        & k\geq \lceil \log_2 \frac{226}{\Delta_{1,0}^2}\rceil\\
        \stackrel{\mu_0>\mu_1\geq\mu_a}{\Rightarrow} & \beta_k\geq \frac{226}{\Delta_{a,0}^2},\log\frac{1}{\delta_k} \geq \frac{96}{\Delta_{a,0}^2}\log 3\\
        \Rightarrow & \beta_k \log\frac{4K}{\delta_k}\geq \frac{226\log\frac{2K}{\delta_k} }{\Delta_{a,0}^2}, \beta_k \log\frac{4K}{\delta_k}\geq \frac{226\log\log \frac{96}{\Delta_{a,0}^2} }{\Delta_{a,0}^2}\\
        \Rightarrow & \beta_k \log\frac{4K}{\delta_k}\geq \frac{113\left(\log\frac{2K}{\delta_k} + \log\log \frac{96}{\Delta_{a,0}^2}\right)}{\Delta_{a,0}^2}\\
        \stackrel{C>1}{\Rightarrow} & (C+1)^2\beta_k \log\frac{4K}{\delta_k}-1 \geq \frac{113\left(\log\frac{2K}{\delta_k} + \log\log \frac{96}{\Delta_{a,0}^2}\right)}{\Delta_{a,0}^2}.
    \end{align*}
    That means at the end of phase $\max\{\kappa^{\text{ee}}_B,\lceil \log_2 \frac{226}{\Delta_{1,0}^2}\rceil, \lceil\log_3 \frac{3}{\delta}\rceil\}$, we must have 
    \begin{align*}
        \text{UCB}_a^{\text{ee}}(\mathcal{H}^{\text{ee}}, \frac{\delta_{\max\{\kappa^{\text{ee}}_B,\lceil \log_2 \frac{226}{\Delta_{1,0}^2}\rceil, \log_3 \frac{3}{\delta}\}}}{K})\leq \mu_0
    \end{align*}
    holds for all $a\in [K]$. Also, not hard to see $\delta_{\max\{\kappa^{\text{ee}}_B,\lceil \log_2 \frac{226}{\Delta_{1,0}^2}\rceil, \lceil\log_3 \frac{3}{\delta}\rceil\}}\leq \delta_{\lceil\log_3 \frac{3}{\delta}\rceil} < \frac{\delta}{3}$. We can conclude Algorithm \ref{alg:SEE-with-Bracket} must terminate by the condition in Line \ref{alg-line:Negative-branch-output}, \ref{alg-line:Negative-branch-output-None}.
\end{proof}

\subsection{Sketch Proof of Theorem \ref{theorem:delta-pac}, \ref{theorem:Parallel-SEE-Etau-upper-bound}}
\label{sec:Sketch-Proof-of-upper-bound}

The full proof of Theorem \ref{theorem:delta-pac}, \ref{theorem:Parallel-SEE-Etau-upper-bound} requires multiple lemmas to support.



\begingroup
\renewcommand{\proofname}{Sketch Proof of Theorem \ref{theorem:delta-pac}}
\begin{proof}
    The first step is prove $\Pr(\tau(B_{ \lceil \log_2 K\rceil+1}) = +\infty)=0$, whose main idea is very similar to the proof of Theorem 5.2 in \cite{pmlr-v267-li25f}. As in \cite{pmlr-v267-li25f}, denote 
    $\kappa^{\text{ee}}_{b}, \kappa^{\text{et}}_{b}$ 
    as the minimum phase indices after which the exploration and exploitation concentration events of copy $b$ hold, 
    we can prove $\Pr(\tau(B_{ \lceil \log_2 K\rceil+1}) = +\infty)\leq\Pr(\max\{\kappa^{\text{ee}}_{\lceil \log_2 K\rceil+1}, \kappa^{\text{et}}_{\lceil \log_2 K\rceil+1}\} = +\infty)=0$.
    
    
    Given the first step, we are able to conclude $\Pr(\tau < +\infty)=1$, by the fact that $\tau \leq (\lceil\log_2 K\rceil +1) \tau(B_b)$ holds for all $b\in [\lceil\log_2 K\rceil+1]$. Then, to prove Algorithm \ref{alg:Parallel-SEE-on-Bracket} is $\delta$-PAC, we suffice to show $\Pr_{\nu}\left(\hat{a}=\textsf{None} \text{ or }\{\hat{a}\in [K],\mu_{\hat{a}}<\mu_0\}\right) < \delta$ for a positive instance $\nu$, and $\Pr_{\nu}\left(\hat{a}\in [K]\right) < \delta$ for a negative instance $\nu$. The two cases are proved similarly. 
    In Line \ref{alg-line:create-alg-copy} of Algorithm~\ref{alg:Parallel-SEE-on-Bracket}, we assign  $\delta/(\lceil\log_2 K\rceil+1)$ to each copy.
    We are able to prove $\Pr(\text{alg}_{b}\text{ outputs incorrect }\hat{a})<\frac{O(1)\delta}{\lceil\log_2 K\rceil+1}$ holds for all $b\in [\lceil\log_2 K\rceil + 1]$. A union bound over all copies completes the proof.
\end{proof}
\endgroup

\begingroup
\renewcommand{\proofname}{Sketch Proof of Theorem \ref{theorem:Parallel-SEE-Etau-upper-bound}}
\begin{proof}

    We consider negative and positive instances separately.

    If the instance $\nu$ is negative, we focus on analyzing $\tau(B_{\lceil\log_2 K\rceil+1})$, which is very similar to the negative case of Theorem 5.3 in \cite{pmlr-v267-li25f}. By the fact that $\mathbb{E}\tau \leq [\lceil\log_2 K\rceil+1] \mathbb{E}\tau(B_{\lceil\log_2 K\rceil+1})$, we complete the proof.

    If the instance $\nu$ is positive, denote $\mu_1\geq \cdots \geq \mu_m >\mu_0 \geq \mu_{m+1}\geq\cdots \geq\mu_K$. For $j\in[m]$ and $a\in[K]$, define $\tilde{T}_j(a)=\frac{\log(4K) + \log_2\frac{1}{\Delta_{j,0}^2}}{\max\{\Delta_{a,j}^2, \Delta_{a,0}^2\}}$, $\tilde{T}_j(a;k)= \frac{\log\frac{4K}{\delta_k} + \log_2\frac{1}{\Delta_{j,0}^2}}{\max\{\Delta_{a,j}^2, \Delta_{a,0}^2\}}$.
    We show that for every $j\in[m]$, $\mathbb{E}\tau \leq O(\log K)\cdot O\left(\frac{\log\frac{1}{\delta}}{\Delta_{j,0}^2} + \frac{\log^3 K}{j}\sum_{a=1}^K\tilde{T}_j(a)\right)$. In addition, the constant hidden by the Big O notation is independent of $j$. For any given $j\in [m]$, we introduce $b_j=\min\{b\in [\lceil \log_2 K\rceil + 1 ]: B_b\cap [j]\neq \emptyset\}$, the smallest bracket index that contains at least one arm among the top-$j$ arms. As we adopt a random permutation to create brackets, $b_j$ is a random variable, and we have $\mathbb{E}\tau = \sum_{b=1}^{\lceil \log_2 K\rceil + 1 }\mathbb{E}\tau\mathds{1}(b_j=b)$. Using $\tau \le (\lceil\log_2 K\rceil+1)\tau(B_b)$ for each $b\in [\lceil\log_2 K\rceil+1]$, it suffices to bound $O(\log K)\,\mathbb{E}_{\nu}[\tau(B_b)\mathds{1}(b_j=b)]$.
    
    Let $\tau^{\text{ee}}(B_{b}),\tau^{\text{et}}(B_{b})$ denote the total length of exploration and exploitation periods of algorithm copy $\text{alg}_{b}$ respectively. For any arm $a\in B_{b}$, denote $N_a^{\text{ee}, B_{b}}$ as the pulling times of arm $a$ in the exploration period. It is evident that $\tau(B_{b})=\tau^{\text{ee}}(B_{b}) + \tau^{\text{et}}(B_{b})$. We assert that
    \begin{align}
        & \mathbb{E}\tau^{\text{et}}(B_b)\mathds{1}(b_j=b)\leq O\left(\frac{(C+3)^2}{(C-1)^2}\frac{\log\frac{4K}{\delta} + \log_2\frac{(C+3)^2}{(C-1)^2\Delta_{j,0}^2}}{\Delta_{j,0}^2}\right)\Pr(b_j=b),\label{eqn:tau-et-b-bound}\\
        & \mathbb{E}\tau^{\text{ee}}(B_b)\mathds{1}(b_j=b)
        \leq \sum_{a=1}^K O(1)\tilde{T}_j(a)\min\{\frac{|B_b|}{K}, \Pr(b_j=b)\} + O(\log K)|B_b|\Pr(b_j=b).\label{eqn:tau-ee-b-bound}
    \end{align}
    Given (\ref{eqn:tau-et-b-bound})(\ref{eqn:tau-ee-b-bound}), we can sum up (\ref{eqn:tau-et-b-bound})(\ref{eqn:tau-ee-b-bound}) over $b=1,2,\cdots,\lceil \log_2 K\rceil + 1$. Lemma \ref{lemma:minimum-bracket-index-that-contains-[j]-enhanced}(section \ref{sec:performance-alg-parallel-SEE}), \ref{lemma:sum-bracket-size-prob-1}, \ref{lemma:sum-bracket-size-prob-2}( section \ref{sec:technical-lemma}) provides upper bounds for $\sum\limits_{b=1}^{\lceil \log_2 K\rceil + 1} |B_b|\Pr( b_j=b)$ and $\sum\limits_{b=1}^{\lceil \log_2 K\rceil + 1}\min\{\frac{|B_b|}{K}, \Pr(b_j=b)\}$, which complete the proof of Theorem \ref{theorem:Parallel-SEE-Etau-upper-bound}. The remaining sketch proof reduces to (\ref{eqn:tau-et-b-bound})(\ref{eqn:tau-ee-b-bound}) for each $b\in [\lceil \log_2 K\rceil + 1 ]$.

    To prove (\ref{eqn:tau-et-b-bound}), we first show the algorithm must terminate no later than the end of phase $L_b^{\text{pos}}=\max\{\kappa^{\text{ee}}_{b}, \kappa^{\text{et}}_{b}, \lceil\log_2\frac{O(1)(C+3)^2}{(C-1)^2\Delta_{j,0}^2}\rceil\}$, if $b_j=b$. Notice that $\beta_{k} \log\frac{4K(\lceil \log_2 K\rceil + 1)\alpha_{k}}{\delta}|_{k=\lceil\log_2\frac{O(1)(C+3)^2}{(C-1)^2\Delta_{j,0}^2}\rceil}=O\left(\frac{(C+3)^2}{(C-1)^2}\frac{\log\frac{4K}{\delta} + \log_2\frac{(C+3)^2}{(C-1)^2\Delta_{j,0}^2}}{\Delta_{j,0}^2}\right)$, and Lemma \ref{lemma:prob-upper-bound-for-kappa} shows $\Pr(\kappa^{\text{ee}}_b\geq k)\leq \frac{\pi^2\delta_{k-1}}{6}$, $\Pr(\kappa^{\text{et}}_b\geq k)\leq \frac{\pi^2\delta}{6\alpha_{k-1}}$, we can take expectation on both sides of $\tau^{\text{et}}(B_b)\mathds{1}(b_j=b)\leq \frac{(C+3)^2}{(C-1)^2}\beta_k\log\frac{4K\alpha_k}{\delta}|_{k=L_b^{\text{pos}}}\mathds{1}(b_j=b)$ to complete the proof of (\ref{eqn:tau-et-b-bound}).

    To prove (\ref{eqn:tau-ee-b-bound}), we first show for each $a\in B_b$, we have
    \begin{align}
        N_a^{\text{ee}, B_{b}}\mathds{1}(b_j=b)\leq O(\tilde{T}_j(a;L_b^{\text{pos}}))(C+1)^2 \mathds{1}(b_j=b).\label{eqn:N_a-ee-upper-bound}
    \end{align}
    The intuition is as follows. Given the fact that $j\in B_b$, once the pulling times of arm $a$ is of order $O(1)\tilde{T}_j(a;L_b^{\text{pos}})(C+1)^2$, copy $\text{alg}_b$ can assert(with some probability, see Lemma \ref{lemma:prob-upper-bound-for-kappa}, \ref{lemma:pulling-times-of-a-single-arm-positive-instance}) either $\text{UCB}_a < \mu_j$ or $\text{LCB}_a > \mu_0$. In both cases, $\text{alg}_b$ stops pulling arm $a$ in the exploration period. By the facts that elements in $B_{b_j}$ are independent of $\kappa^{\text{ee}}_{b_j}, \kappa^{\text{et}}_{b_j}$, we can sum up (\ref{eqn:N_a-ee-upper-bound}) over $a\in B_b$ and derive
    \begin{align*}
        \mathbb{E}\tau^{\text{ee}}(B_b)\mathds{1}(b_j=b)
        \leq \sum_{a=1}^K O(1)\tilde{T}_j(a) \Pr(a\in B_b, b_j=b) + O(\log K)|B_b|\Pr(b_j=b).
    \end{align*}
    Since $\Pr(a\in B_b, b_j=b)\leq \min\{\Pr(a\in B_b), \Pr(b_j=b)\}$, and $\Pr(a\in B_b)=\frac{|B_b|}{K}$, (\ref{eqn:tau-ee-b-bound}) is shown.
\end{proof}
\endgroup

\subsection{Proof of Theorem  \ref{theorem:delta-pac} and \ref{theorem:Parallel-SEE-Etau-upper-bound}}
\label{sec:performance-alg-parallel-SEE}

In this subsection, We present the rigorous proof of Theorem \ref{theorem:delta-pac} and \ref{theorem:Parallel-SEE-Etau-upper-bound} and the required lemmas. Before illustrating the proof of Theorem  \ref{theorem:delta-pac} and \ref{theorem:Parallel-SEE-Etau-upper-bound}, we first prove Lemma \ref{lemma:minimum-bracket-index-that-contains-[j]-enhanced} for the random permutation, Line \ref{alg-line:random-permutation} of Algorithm \ref{alg:Parallel-SEE-on-Bracket}.
\begin{restatelemma}[restatement of Lemma \ref{lemma:minimum-bracket-index-that-contains-[j]-enhanced}]
   Given a qualified arm index $j\in [m]$, denote bracket index $b_j=\min\{b\in [\lceil \log_2 K\rceil + 1 ]: B_b\cap [j]\neq \emptyset\}$. We have
    \begin{align*}
        \Pr(b_j\geq \tilde{b})\leq \frac{1}{\exp\left(\frac{j|B_{\tilde{b}-1}|}{K}\right)}, \forall \tilde{b} \in [\lceil \log_2 K\rceil + 1]
    \end{align*}
    Here we denote $B_0=\emptyset$ and $|B_0|=0$. Since we take $|B_{\tilde{b}}|=\min\{2^{\tilde{b}-1}, K\}$ for $\tilde{b} \in [\lceil \log_2 K\rceil + 1]$, we can rewrite the conclusion as
    \begin{align*}
        \Pr(b_j\geq \tilde{b})\leq \frac{1}{\exp\left(\frac{j\cdot \lfloor 2^{\tilde{b}-2}\rfloor}{K}\right)}, \forall \tilde{b} \in [\lceil \log_2 K\rceil + 1].
    \end{align*}
\end{restatelemma}
\begin{proof}[Proof of Lemma \ref{lemma:minimum-bracket-index-that-contains-[j]-enhanced}]

    If $\tilde{b}=1$, the conclusion is $\Pr(b_j\geq 1)\leq 1$, which is obviously true. We focus on the case that $2\leq \tilde{b}\leq \lceil \log_2 K\rceil + 1$. We have
    \begin{align*}
        \Pr(b_j\geq \tilde{b}) = & \Pr\Big( \{\sigma(i)\}_{i=1}^{\min\{ 2^{\tilde{b}-1-1}, K\} }  \cap [j]=\emptyset  \Big)\\
        = & \Pr\Big( \{\sigma(i)\}_{i=1}^{2^{\tilde{b}-2}}  \cap [j]=\emptyset  \Big).
    \end{align*}
    The second line is from the fact that $\tilde{b}\in [\lceil \log_2 K\rceil + 1]$, which concludes $2^{\tilde{b}-2}\leq 2^{\lceil \log_2 K\rceil - 1}\leq 2^{\log_2 K}=K$. We can conclude
    \begin{align*}
        & \Pr\Big( \{\sigma(i)\}_{i=1}^{2^{\tilde{b}-2}}  \cap [j]=\emptyset  \Big)\\
        = & \prod_{i=1}^{2^{\tilde{b}-2}} \frac{K-j-i+1}{K-i+1}\\
        \leq & (\frac{K-j}{K})^{2^{\tilde{b}-2}}\\
        = & \left(\left(1-\frac{j}{K}\right)^{\frac{K}{j}}\right)^{\frac{j\cdot2^{\tilde{b}-2}}{K}}\\
        \leq & \exp\left(-\frac{j\cdot2^{\tilde{b}-2}}{K}\right).
    \end{align*}
\end{proof}

Given the above preparation, we are ready to prove the Theorem \ref{theorem:delta-pac} and \ref{theorem:Parallel-SEE-Etau-upper-bound}. We first prove Theorem \ref{theorem:delta-pac}.
\begin{proof}[Proof of Theorem \ref{theorem:delta-pac}]

    We first show $\Pr_{\nu}(\tau < +\infty)=1$, by applying Algorithm \ref{alg:Parallel-SEE-on-Bracket}. This conclusion holds no matter $\nu$ is positive or negative. Consider the bracket $\lceil \log_2 K\rceil+1$, which is indeed $[K]$. From Lemma \ref{lemma:termination-phase-index-negative-instance} and Lemma \ref{lemma:termination-phase-index-positive-instance}, we know
    \begin{align*}
        & \Pr_{\nu}(\tau = +\infty)\\
        \leq & \Pr_{\nu}(\tau(B_{\lceil \log_2 K\rceil+1}) = +\infty)\\
        \leq & \Pr_{\nu}(\kappa^{\text{ee}}_{\lceil \log_2 K\rceil+1} = +\infty) + \Pr_{\nu}(\kappa^{\text{et}}_{\lceil \log_2 K\rceil+1} = +\infty)
    \end{align*}
    holds for any instance $\nu$. By the Lemma \ref{lemma:prob-upper-bound-for-kappa}, we know $\Pr_{\nu}(\kappa^{\text{ee}}_{\lceil \log_2 K\rceil+1} = +\infty)= \Pr_{\nu}(\kappa^{\text{et}}_{\lceil \log_2 K\rceil+1} = +\infty)=0$. That means $\Pr_{\nu}(\tau = +\infty)=0$, further $\Pr_{\nu}(\tau < +\infty)=1$. This suggest that $\hat{a}$ is well defined with probability 1, and we can assert $\Pr_{\nu}(\hat{a}\in [K]\cup\{\textsf{None}\})=1$ if $\nu$ is positive or negative.

    Then, we suffice to show $\Pr_{\nu}(\mu_{\hat{a}} \leq \mu_0\text{ or }\mu_{\hat{a}}=\textsf{None})<\delta$ for a positive instance $\nu$ and $\Pr_{\nu}(\hat{a}\in [K])<\delta$ for a negative instance $\nu$. If $\nu$ is positive, we get
    \begin{align*}
        & \Pr_{\nu}(\mu_{\hat{a}} \leq \mu_0\text{ or }\hat{a}=\textsf{None})\\
        \leq & \Pr_{\nu}(\mu_{\hat{a}} \leq \mu_0) + \Pr_{\nu}(\hat{a}=\textsf{None})\\
        \leq &  \Pr_{\nu}(\text{alg copy }\lceil \log_2 K\rceil+1\text{ outputs }\hat{a}=\textsf{None}) +\\ & \sum_{b=1}^{\lceil \log_2 K\rceil+1}\Pr_{\nu}(\text{alg copy }b\text{ outputs }\hat{a},s.t. \mu_{\hat{a}} \leq \mu_0)\\
        \leq & \frac{\pi^2}{6}\frac{\delta}{3}+ \sum_{b=1}^{\lceil \log_2 K\rceil+1}\frac{\pi^2}{6\cdot 5}\frac{\delta}{\lceil \log_2 K\rceil+1}\\
        \leq & \delta.
    \end{align*}
    The second last line is by Lemma \ref{lemma:correctnesss-output-unqualified-arm} and \ref{lemma:correctnesss-output-None-while-qualified-exists}.

    If $\nu$ is negative, we get
    \begin{align*}
        & \Pr_{\nu}(\hat{a}\in [K])\\
        \leq & \sum_{b=1}^{\lceil \log_2 K\rceil+1}\Pr_{\nu}(\text{alg copy }b\text{ outputs }\hat{a},s.t. \mu_{\hat{a}} \leq \mu_0)\\
        \leq & \sum_{b=1}^{\lceil \log_2 K\rceil+1}\frac{\pi^2}{6\cdot 5}\frac{\delta}{\lceil \log_2 K\rceil+1}\\
        < & \delta.
    \end{align*}
    The second last line is by Lemma \ref{lemma:correctnesss-output-unqualified-arm}.
\end{proof}

To prove Theorem \ref{theorem:Parallel-SEE-Etau-upper-bound}, we split the proof into two cases, $\nu$ is positive or negative.
\begin{proof}[Proof of Theorem \ref{theorem:Parallel-SEE-Etau-upper-bound}, when $\nu$ is positive]

    Fix arbitrary $j\in [m]$ and follow the notation in Lemma \ref{lemma:minimum-bracket-index-that-contains-[j]-enhanced}. Denote $b_j=\min\{b\in [\lceil \log_2 K\rceil + 1 ]: B_b\cap [j]\neq \emptyset\}$. We have $\tau = \sum_{b=1}^{\lceil \log_2 K\rceil + 1}\tau\mathds{1}(b_j=b)$. Then, we can turn to find the upper bound for each $\tau\mathds{1}(b_j=b)$.

    From the algorithm design, we know 
    \begin{align*}
        \tau\mathds{1}(b_j=b)\leq & (\lceil \log_2 K\rceil + 1)\tau(B_b)\mathds{1}(b_j=b)\\
        \leq & (\lceil \log_2 K\rceil + 1)\Big(\tau^{\text{ee}}(B_b)\mathds{1}(b_j=b) + \tau^{\text{et}}(B_b)\mathds{1}(b_j=b)\Big).
    \end{align*}
    Then, we are going to derive an upper bound for $\tau^{\text{ee}}(B_b)\mathds{1}(b_j=b)$ and $\tau^{\text{et}}(B_b)\mathds{1}(b_j=b)\Big)$ separately.

    For $\tau^{\text{ee}}(B_b)\mathds{1}(b_j=b)$, we have
    \begin{align*}
        & \tau^{\text{ee}}(B_b)\mathds{1}(b_j=b)\\
        \stackrel{\text{Lemma }\ref{lemma:size-tau}}{\leq} & \mathds{1}(b_j=b)\Big(|\mathcal{H}^{\text{ee}}| + |Q|\Big)\\
        \stackrel{\text{Lemma }\ref{lemma:size-of-Q}, \ref{lemma:pulling-times-of-a-single-arm-positive-instance},\ref{lemma:termination-phase-index-positive-instance}}{\leq} & \mathds{1}(b_j=b) \Bigg(|B_b|+\sum_{a\in B_b} N_a^{\text{ee}, B_b}\Big(\tau_{\kappa^{\text{ee}}_{b}-1}^{\text{ee}}(B_b)\Big)+\\
        & \frac{29(C+1)^2\left(\log\frac{2K}{\delta_{\max\{\kappa^{\text{ee}}_b, \kappa^{\text{et}}_b, \lceil\log_2\frac{448(C+3)^2}{(C-1)^2\Delta_{j,0}^2}\rceil\}} } + \log\log \frac{24(C+1)^2}{\max\{\Delta_{a,j}^2, \Delta_{a,0}^2\}}\right)}{\max\{\Delta_{a,j}^2, \Delta_{a,0}^2\}}\Bigg)\\
        \stackrel{\text{Lemma }\ref{lemma:budget-upper-bound-for-N_a-tau_ee-tau_et}}{\leq} & \underbrace{\mathds{1}(b_j=b)(C+1)^2|B_b|\left(\beta_{\kappa^{\text{ee}}_{b}-1}\log\frac{4K}{\delta_{\kappa^{\text{ee}}_{b}-1} }+1\right) }_{\dagger} + \\
        & \underbrace{\sum_{a=1}^K\frac{29(C+1)^2\left(\log\frac{4K\log \frac{24(C+1)^2}{\max\{\Delta_{a,j}^2, \Delta_{a,0}^2\}}}{\delta_{\max\{\kappa^{\text{ee}}_{b}, \kappa^{\text{et}}_{b}, \lceil\log_2\frac{448(C+3)^2}{(C-1)^2\Delta_{j,0}^2}\rceil\}} }\right)}{\max\{\Delta_{a,j}^2, \Delta_{a,0}^2\}}\mathds{1}(b_j=b, a\in B_b)}_{\ddagger}.
    \end{align*}
    For $\tau^{\text{et}}(B_b)\mathds{1}(b_j=b)$, we have
    \begin{align*}
        & \tau^{\text{et}}(B_b)\mathds{1}(b_j=b)\\
        \stackrel{\text{Lemma }\ref{lemma:budget-upper-bound-for-tau_ee-tau_et}}{\leq} & \mathds{1}(b_j=b)\sum_{k=1}^{\max\{\kappa^{\text{ee}}_{b}, \kappa^{\text{et}}_{b}, \lceil\log_2\frac{448(C+3)^2}{(C-1)^2\Delta_{j,0}^2}\rceil\}}\frac{(C+3)^2}{(C-1)^2}\beta_k\log\frac{4K(\lceil \log_2 K\rceil + 1)\alpha_k}{\delta}\\
        \leq & \underbrace{\frac{2(C+3)^2}{(C-1)^2}\mathds{1}(b_j=b) \beta_{k} \log\frac{4K(\lceil \log_2 K\rceil + 1)\alpha_{k}}{\delta}|_{k=\max\{\kappa^{\text{ee}}_{b}, \kappa^{\text{et}}_{b}, \lceil\log_2\frac{448(C+3)^2}{(C-1)^2\Delta_{j,0}^2}\rceil\}}}_{\diamondsuit}.
    \end{align*}
    The last line is by the fact that $\frac{ \beta_{k} \log\frac{4K(\lceil \log_2 K\rceil + 1)\alpha_{k}}{\delta}}{ \beta_{k-1} \log\frac{4K(\lceil \log_2 K\rceil + 1)\alpha_{k-1}}{\delta}} \geq 2$.

    Then, to bound $\mathbb{E}\tau^{\text{ee}}(B_b)\mathds{1}(b_j=b)$ and $\mathbb{E}\tau^{\text{et}}(B_b)\mathds{1}(b_j=b)\Big)$, we suffice to consider the upper bound of the expectation of $\dagger,\ddagger,\diamondsuit$. For $\dagger$, we have
    \begin{align*}
        & \mathbb{E}\left[\mathds{1}(b_j=b)|B_b|\left(\beta_{\kappa^{\text{ee}}_{b}-1}\log\frac{4K}{\delta_{\kappa^{\text{ee}}_{b}-1} }+1\right)\right]\\
        \leq & \Pr(b_j=b)|B_b|\mathbb{E}\left[\left(\beta_{\kappa^{\text{ee}}_{b}-1}\log\frac{4K}{\delta_{\kappa^{\text{ee}}_{b}-1} }+1\right)\right]\\
        \leq & \frac{1}{\exp\left(\frac{j\cdot \lfloor 2^{b-2}\rfloor}{K}\right)}\min\{2^{b-1},K\} \left(1+\sum_{k=1}^{+\infty}2^{k-1}\log(4K\cdot 3^{k-1})\frac{\pi^2}{6}\frac{1}{3^{k-1}}\right)\\
        \leq & O(\log K) \frac{1}{\exp\left(\frac{j\cdot \lfloor 2^{b-2}\rfloor}{K}\right)}\min\{2^{b-1},K\}.
    \end{align*}
    The second line is by the fact that $\kappa^{\text{ee}}_{b}$ is independent with $b_j$. The third line is by the Lemma \ref{lemma:prob-upper-bound-for-kappa}, \ref{lemma:minimum-bracket-index-that-contains-[j]-enhanced}. For the term $\ddagger$, we have
    \begin{align*}
        & \mathbb{E}\left[\sum_{a=1}^K\frac{\log\frac{4K}{\delta_{\max\{\kappa^{\text{ee}}_{b}, \kappa^{\text{et}}_{b}, \lceil\log_2\frac{448(C+3)^2}{(C-1)^2\Delta_{j,0}^2}\rceil\}} } + \log\log \frac{1}{\max\{\Delta_{a,j}^2, \Delta_{a,0}^2\}}}{\max\{\Delta_{a,j}^2, \Delta_{a,0}^2\}}\mathds{1}(b_j=b)\mathds{1}(a\in B_b)\right]\\
        \leq & \sum_{a=1}^K \frac{\log\log \frac{1}{\max\{\Delta_{a,j}^2, \Delta_{a,0}^2\}}}{\max\{\Delta_{a,j}^2, \Delta_{a,0}^2\}}\Pr(a\in B_b, b_j=b) + \\
        & \sum_{a=1}^K\mathbb{E}\left[\frac{\log(4K) + (\kappa^{\text{ee}}_{b} + \kappa^{\text{et}}_{b}+\lceil\log_2\frac{448(C+3)^2}{(C-1)^2\Delta_{j,0}^2}\rceil)\log 3 }{\max\{\Delta_{a,j}^2, \Delta_{a,0}^2\}}\right]\Pr(a\in B_b, b_j=b)\\
        \leq  & \sum_{a=1}^K\frac{O\left( (\log(4K) + \log_2\frac{(C+3)^2}{(C-1)^2\Delta_{j,0}^2})\min\{\Pr(a\in B_b), \Pr( b_j=b)\}\right)}{\max\{\Delta_{a,j}^2, \Delta_{a,0}^2\}} + \\
        & \sum_{a=1}^K\frac{\min\{\Pr(a\in B_b), \Pr( b_j=b)\} }{\max\{\Delta_{a,j}^2, \Delta_{a,0}^2\}} \cdot (\sum_{k=1}^{+\infty} k\frac{\pi^2}{6}\frac{1}{3^{k-1}} + \sum_{k=1}^{+\infty} k\frac{\pi^2}{6}\frac{\delta}{5^{k-1}}).
    \end{align*}
    The second step is by the fact that $\frac{1}{\max\{\Delta_{a,j}^2, \Delta_{a,0}^2\}}\leq \frac{4}{\Delta_{j,0}^2}$ holds for all $a,j\in [K]$ and the third step is by the Lemma \ref{lemma:prob-upper-bound-for-kappa}.
    
    From Lemma \ref{lemma:minimum-bracket-index-that-contains-[j]-enhanced}, we know $\Pr( b_j=b)\leq \frac{1}{\exp\left(\frac{j\cdot \lfloor 2^{b-2}\rfloor}{K}\right)}$. What's more, we know $\Pr(a\in B_b)=\frac{|B_b|}{K}=\frac{\min\{2^{b-1}, K\}}{K}$ holds for all $a\in [K]$. Thus, we can conclude
    \begin{align*}
        & \mathbb{E}\left[\sum_{a=1}^K\frac{\log\frac{4K}{\delta_{\max\{\kappa^{\text{ee}}_{b}, \kappa^{\text{et}}_{b}, \log_2\frac{\Theta(1)}{\Delta_{j,0}^2}\}} } + \log\log \frac{1}{\max\{\Delta_{a,j}^2, \Delta_{a,0}^2\}}}{\max\{\Delta_{a,j}^2, \Delta_{a,0}^2\}}\mathds{1}(b_j=b)\mathds{1}(a\in B_b)\right]\\
        \leq  & \min\{\frac{1}{\exp\left(\frac{j\cdot \lfloor 2^{b-2}\rfloor}{K}\right)},\frac{\min\{2^{b-1}, K\}}{K}\}\sum_{a=1}^K\frac{O\left( (\log(4K) + \log_2\frac{(C+3)^2}{(C-1)^2\Delta_{j,0}^2})\right)}{\max\{\Delta_{a,j}^2, \Delta_{a,0}^2\}} + \\
        & \min\{\frac{1}{\exp\left(\frac{j\cdot \lfloor 2^{b-2}\rfloor}{K}\right)},\frac{\min\{2^{b-1}, K\}}{K}\}\sum_{a=1}^K\frac{ O(1) }{\max\{\Delta_{a,j}^2, \Delta_{a,0}^2\}}\\
        \leq 
        & \min\{\frac{1}{\exp\left(\frac{j\cdot \lfloor 2^{b-2}\rfloor}{K}\right)},\frac{\min\{2^{b-1}, K\}}{K}\}\sum_{a=1}^K\frac{O\left( (\log(4K) + \log_2\frac{(C+3)^2}{(C-1)^2\Delta_{j,0}^2})\right)}{\max\{\Delta_{a,j}^2, \Delta_{a,0}^2\}}.
    \end{align*}
    
    For the expectation of $\diamondsuit$, we have
    \begin{align*}
        & \mathbb{E}\mathds{1}(b_j=b) \beta_{\max\{\kappa^{\text{ee}}_{b}, \kappa^{\text{et}}_{b}, \lceil\log_2\frac{448(C+3)^2}{(C-1)^2\Delta_{j,0}^2}\rceil\}} \log\frac{4K(\lceil \log_2 K\rceil + 1)\alpha_{\max\{\kappa^{\text{ee}}_{b}, \kappa^{\text{et}}_{b},\lceil\log_2\frac{448(C+3)^2}{(C-1)^2\Delta_{j,0}^2}\rceil\}}}{\delta})\\
        \leq & \Pr(b_j=b) O\left(\mathbb{E}\beta_{\kappa^{\text{ee}}_{b}} \log\frac{4K\alpha_{\kappa^{\text{ee}}_{b}}}{\delta} +\mathbb{E} \beta_{\kappa^{\text{et}}_{b}} \log\frac{4K\alpha_{\kappa^{\text{et}}_{b}}}{\delta} + \frac{(C+3)^2}{(C-1)^2}\frac{\log\frac{4K}{\delta} + \log_2\frac{(C+3)^2}{(C-1)^2\Delta_{j,0}^2}}{\Delta_{j,0}^2}\right)\\
        \leq & \Pr(b_j=b)O\left(\sum_{k=1}^{+\infty}\frac{\pi^2}{6}\frac{2^k}{3^k}\log\frac{4K\cdot 5^k}{\delta} + \sum_{k=1}^{+\infty}\frac{\pi^2}{6}\frac{2^k\delta}{5^k}\log\frac{4K\cdot 5^k}{\delta} + \frac{\log\frac{4K}{\delta} + \log_2\frac{\Theta(1)}{\Delta_{j,0}^2}\log 5}{\Delta_{j,0}^2} \right)\\
        \leq & \Pr(b_j=b)O\left(\frac{(C+3)^2}{(C-1)^2}\frac{\log\frac{4K}{\delta} + \log_2\frac{(C+3)^2}{(C-1)^2\Delta_{j,0}^2}}{\Delta_{j,0}^2}\right)
    \end{align*}
    Sum up all the upper bounds, we have
    \begin{align*}
        & \mathbb{E}\tau\\
        \leq & O(\log K) \sum_{b=1}^{\lceil \log_2 K\rceil + 1}O(\log K) \frac{1}{\exp\left(\frac{j\cdot \lfloor 2^{b-2}\rfloor}{K}\right)}\min\{2^{b-1},K\} + \\
        & O(\log K) \sum_{a=1}^K\sum_{b=1}^{\lceil \log_2 K\rceil + 1}\frac{O\left( (\log(4K) + \log_2\frac{(C+3)^2}{(C-1)^2\Delta_{j,0}^2})\right)}{\max\{\Delta_{a,j}^2, \Delta_{a,0}^2\}}\min\{\frac{1}{\exp\left(\frac{j\cdot \lfloor 2^{b-2}\rfloor}{K}\right)},\frac{\min\{2^{b-1}, K\}}{K}\} + \\
        & O(\log K) \sum_{b=1}^{\lceil \log_2 K\rceil + 1}\frac{(C+3)^2}{(C-1)^2}\Pr(b_j=b)O\left(\frac{\log\frac{4K}{\delta} + \log_2\frac{(C+3)^2}{(C-1)^2\Delta_{j,0}^2}}{\Delta_{j,0}^2}\right)
    \end{align*}
    By the Lemma \ref{lemma:sum-bracket-size-prob-1}, \ref{lemma:sum-bracket-size-prob-2}, we can conclude
    \begin{align*}
        & \sum_{b=1}^{\lceil \log_2 K\rceil + 1}\frac{1}{\exp\left(\frac{j\cdot \lfloor 2^{b-2}\rfloor}{K}\right)}\min\{2^{b-1},K\}
        \leq \frac{K}{j}\cdot O\Big((\log_2\frac{K}{j}) +1\Big)\\
        & \sum_{b=1}^{\lceil \log_2 K\rceil + 1}\min\{\frac{1}{\exp\left(\frac{j\cdot \lfloor 2^{b-2}\rfloor}{K}\right)},\frac{\min\{2^{b-1}, K\}}{K}\} \leq O\Big(\frac{\log j}{j}\Big).
    \end{align*}
    Together with the fact $\sum_{b=1}^{\lceil \log_2 K\rceil + 1}\Pr(b_j=b)=1$ and plug in default value $C=1.01$ specified by Algorithm \ref{alg:Parallel-SEE-on-Bracket}, we have
    \begin{align*}
        & \mathbb{E}\tau\\
        \leq & O(\log K)\cdot O\left(\frac{\log\frac{1}{\delta}}{\Delta_{j,0}^2} + \frac{K\log K\Big((\log_2\frac{K}{j}) +1\Big)}{j} + \frac{\log j}{j}\sum_{a=1}^K\frac{ \log(4K) + \log_2\frac{1}{\Delta_{j,0}^2}}{\max\{\Delta_{a,j}^2, \Delta_{a,0}^2\}}\right)\\
        \leq & O(\log K)\cdot O\left(\frac{\log\frac{1}{\delta}}{\Delta_{j,0}^2} + \frac{\log j\log K\Big((\log_2\frac{K}{j}) +1\Big)}{j}\sum_{a=1}^K\frac{ \log(4K) + \log_2\frac{1}{\Delta_{j,0}^2}}{\max\{\Delta_{a,j}^2, \Delta_{a,0}^2\}}\right)\\
        \leq & O(\log K)\cdot O\left(\frac{\log\frac{1}{\delta}}{\Delta_{j,0}^2} + \frac{\log^3 K}{j}\sum_{a=1}^K\frac{ \log(4K) + \log_2\frac{1}{\Delta_{j,0}^2}}{\max\{\Delta_{a,j}^2, \Delta_{a,0}^2\}}\right).
    \end{align*}
    The second last line is by the assumption $\Delta_{i,j}<1$ for all $i,j\in [K]\cup \{0\}$. Notice the above upper bound holds for all $j\in [m]$, we can conclude
    \begin{align*}
        \mathbb{E}\tau \leq O(\log K)\cdot O\left(\min_{1\leq j\leq m}\frac{\log\frac{1}{\delta}}{\Delta_{j,0}^2} + \frac{\log^3 K}{j}\sum_{a=1}^K\frac{ \log(4K) + \log_2\frac{1}{\Delta_{j,0}^2}}{\max\{\Delta_{a,j}^2, \Delta_{a,0}^2\}}\right).
    \end{align*}
\end{proof} 

\begin{proof}[Proof of Theorem \ref{theorem:Parallel-SEE-Etau-upper-bound}, when $\nu$ is negative]

Consider the bracket with size $K$, denoting this bracket as $B$. Similar to the proof of positive case, we split $\tau(B) = \tau^{\text{ee}}(B) + \tau^{\text{et}}(B)$ and prove an upper bound for $\mathbb{E}\tau^{\text{ee}}(B)$ and $\mathbb{E}\tau^{\text{et}}(B)$ separately. From Lemma \ref{lemma:size-tau}, \ref{lemma:pulling-times-of-a-single-arm-negative-instance} and Lemma \ref{lemma:termination-phase-index-negative-instance}, we know
\begin{align*}
    \tau^{\text{ee}}(B) \leq & |Q|+\sum_{a=1}^K N_a^{\text{ee}, B}\Big(\tau_{\kappa^{\text{ee}}_B-1}^{\text{ee}}(B)\Big) + \frac{O\left(\log\frac{4K}{\delta_{\max\{\kappa^{\text{ee}}_B, \kappa^{\text{et}}_B, \lceil \log_2 \frac{226}{\Delta_{1,0}^2}\rceil, \lceil\log_3 \frac{3}{\delta}\rceil\}}} + \log\log \frac{1}{\Delta_{a,0}^2}\right)}{\Delta_{a,0}^2}\\
    = & K + \tau_{\kappa^{\text{ee}}_B-1}^{\text{ee}}(B) + \sum_{a=1}^K \frac{O\left(\log\frac{4K}{\delta_{\max\{\kappa^{\text{ee}}_B, \kappa^{\text{et}}_B,\lceil \log_2 \frac{226}{\Delta_{1,0}^2}\rceil, \lceil\log_3 \frac{3}{\delta} \rceil\}}} + \log\log \frac{1}{\Delta_{a,0}^2}\right)}{\Delta_{a,0}^2}.
\end{align*}
Taking expectation on both sides, we have
\begin{align*}
    & \mathbb{E}\tau_{\kappa^{\text{ee}}_B-1}^{\text{ee}}(B)\\
    \leq & \sum_{k=1}^{+\infty} |K|\beta_{k-1} \log\frac{4K}{\delta_{k-1}} \Pr(\kappa^{\text{ee}}_B=k)\\
    \leq & \sum_{k=1}^{+\infty} |K| 2^{k-1} \log(4K \cdot 3^{k-1}) \frac{\pi^2}{6\cdot 3^{k-1}}\\
    \leq & O(K\log K).
\end{align*}
The second last line is by the Lemma \ref{lemma:prob-upper-bound-for-kappa}. In addition, we can also derive 
\begin{align*}
    & \mathbb{E} \log \frac{1}{\delta_{\max\{\kappa^{\text{ee}}_B, \kappa^{\text{et}}_B,\lceil \log_2 \frac{226}{\Delta_{1,0}^2}\rceil, \lceil\log_3 \frac{3}{\delta} \rceil\}}}\\
    \leq & \mathbb{E} \log \frac{1}{\delta_{\kappa^{\text{ee}}_B}} + \mathbb{E} \log \frac{1}{\delta_{\kappa^{\text{et}}_B}} + \mathbb{E} \log \frac{1}{\delta_{\lceil \log_2 \frac{226}{\Delta_{1,0}^2}\rceil}} + \mathbb{E} \log \frac{1}{\delta_{\lceil\log_3 \frac{3}{\delta} \rceil}}\\
    \leq & \sum_{k=1}^{+\infty} k\log 3 \Pr(\kappa^{\text{ee}}_B=k) + \sum_{k=1}^{+\infty} k\log 3 \Pr(\kappa^{\text{et}}_B=k) + O\left(\log_2 \frac{1}{\Delta_{1,0}^2}\right)+\log \frac{9}{\delta}\\
    \leq & \sum_{k=1}^{+\infty} \frac{\pi^2}{6}\frac{k}{3^{k-1}}\log 3 \Pr(\kappa^{\text{ee}}_B=k) + \sum_{k=1}^{+\infty} \frac{\pi^2}{6} \frac{k\delta}{5^{k-1}}\log 3  + O\left(\log_2 \frac{1}{\Delta_{1,0}^2}\right)+\log \frac{3}{\delta}\\
    \leq & O\left(\log_2 \frac{1}{\Delta_{1,0}^2}+\log \frac{1}{\delta}\right) .
\end{align*}
The second last line is still by the Lemma \ref{lemma:prob-upper-bound-for-kappa}. The above two inequalities suggests that
\begin{align*}
    \mathbb{E}\tau^{\text{ee}}(B) \leq & O(K\log K) + \sum_{a=1}^K O\left(\frac{\log\frac{12K(\lceil \log_2 K\rceil + 1)}{\delta} + \log\frac{1}{\Delta_{0,1}^2}}{\Delta_{a,0}^2}\right)\\
    \leq & \sum_{a=1}^K O\left(\frac{\log\frac{12K(\lceil \log_2 K\rceil + 1)}{\delta} + \log\frac{1}{\Delta_{0,1}^2}}{\Delta_{a,0}^2}\right).
\end{align*}
Meanwhile, for $\mathbb{E}\tau^{\text{et}}(B)$, we can utilize Lemma \ref{lemma:output-corretness-of-negative-instance} to conclude
\begin{align*}
    \mathbb{E}\tau^{\text{et}}(B) \leq & \mathbb{E}\sum_{k=1}^{\kappa^{\text{ee}}_B-1} O(1)\beta_k\log\frac{4K(\lceil \log_2 K\rceil + 1)\alpha_k}{\delta}\\
    \leq & \mathbb{E}O(1)\beta_{\kappa^{\text{ee}}_B-1}\log\frac{4K(\lceil \log_2 K\rceil + 1)\alpha_{\kappa^{\text{ee}}_B-1}}{\delta}\\
    \leq & O(1)\sum_{k=1}^{+\infty}\beta_{k-1} \frac{\pi^2}{6}\delta_{k-1}\log\frac{4K(\lceil \log_2 K\rceil + 1)\alpha_{k-1}}{\delta}\\
    \leq & O\left(\log\frac{4K}{\delta}\right).
\end{align*}
Combining the upper bound of $\mathbb{E}\tau^{\text{ee}}(B)$ and $\mathbb{E}\tau^{\text{et}}(B)$, we have
\begin{align*}
    \mathbb{E}\tau\leq O(\log K)\mathbb{E}\tau(B) \leq O(\log K)O\left(\sum_{a=1}^K\frac{\log\frac{K}{\delta} + \log\frac{1}{\Delta_{0,1}^2}}{\Delta_{a,0}^2}\right)
\end{align*}

\end{proof}

\section{Technical Lemma}
\label{sec:technical-lemma}
This section is for the proof of inequalities.
\begin{lemma}
    \label{lemma:Gaussian-summation-submartingale}
    Denote $\{Z_s\}_{s=1}^{+\infty}$ is a group of random variable fulfilling that $Z_s|Z_1,Z_2,\cdots, Z_{s-1}\sim N(0, \Delta^2)$ for some $\Delta > 0$. Then, we have
    \begin{align*}
        \Pr(\max_{1\leq t\leq n} \sum_{s=1}^t Z_s > z) \leq \exp\left(-\frac{z^2}{2n\Delta^2}\right).
    \end{align*}
\end{lemma}
\begin{proof}[Proof of Lemma \ref{lemma:Gaussian-summation-submartingale}]
    Notice that for $\lambda = \frac{z}{n\Delta^2}$, we have
    \begin{align*}
        & \Pr\big(\max_{1\leq t\leq n}\sum_{s=1}^{t}Z_s\geq z\big)\\
        = & \Pr\big(\max_{1\leq t\leq n}\lambda\big(\sum_{s=1}^{t}Z_s\big)\geq \lambda z\big)\\
        = & \Pr\left(\max_{1\leq t\leq n} \prod_{s=1}^t\exp\left(\lambda\big(Z_s\big)\right)\geq \exp\left(\lambda z\right)\right).
    \end{align*}
    Easy to see $\{\prod_{s=1}^t\exp\left(\lambda\big(Z_s\big)\right)\}_{t=1}^{+\infty}$ is a submartingale, as 
    \begin{align*}
        & \mathbb{E}[\exp\left(\lambda\big(Z_t\big)\right) | z_1,\cdots, z_{t-1}]\\
        \stackrel{\text{Jensen's Inequality}}{\geq} & \exp\left(\lambda \mathbb{E}[\big(Z_t\big) | Z_1,\cdots, Z_{t-1}]\right)\\
        = & \exp(0)=1.
    \end{align*}
    Then, we can conclude
    \begin{align*}
        & \Pr(\max_{1\leq t\leq n} \sum_{s=1}^t Z_s > z)\\
        \leq & \frac{\mathbb{E}_{\nu'}\prod_{t=1}^n\exp(\lambda(Z_t))}{\exp(\lambda z)}\\
        = & \frac{\exp(\frac{n\lambda^2\Delta^2}{2})}{\exp(\lambda z)}\\
        = & \exp(-\frac{z^2}{2n\Delta^2}).
    \end{align*}
\end{proof}

\begin{lemma}
    \label{lemma:sum-bracket-size-prob-1}
    For any $K\in \mathbb{N}$, $K\geq 2$, $j\in [K]$ we have
    \begin{align*}
        \sum_{b=1}^{\lceil \log_2 K\rceil + 1}\frac{1}{\exp\left(\frac{j\cdot \lfloor 2^{b-2}\rfloor}{K}\right)}\min\{2^{b-1},K\}
        \leq \frac{K}{j}\cdot O\Big((\log_2\frac{K}{j}) +1\Big)
    \end{align*}
\end{lemma}
\begin{proof}
    We can conduct direct calculation.
    \begin{align*}
        & \sum_{b=1}^{\lceil \log_2 K\rceil + 1}\frac{1}{\exp\left(\frac{j\cdot \lfloor 2^{b-2}\rfloor}{K}\right)}\min\{2^{b-1},K\}\\
        \stackrel{K\geq 2}{=} & 1 + \frac{2}{\exp(\frac{j}{K})} + \sum_{b=3}^{\lceil \log_2 K\rceil + 1}\frac{1}{\exp\left(\frac{j\cdot \lfloor 2^{b-2}\rfloor}{K}\right)}\min\{2^{b-1},K\}\\
        \leq & 1 + \frac{2}{\exp(\frac{j}{K})} + \sum_{b=3}^{\lceil \log_2 K\rceil + 1}\frac{2^{b-1}}{\exp\left(\frac{j\cdot  2^{b-2}}{K}\right)}\\
        = & 1 + \frac{2}{\exp(\frac{j}{K})} + \frac{2K}{j}\sum_{b=3}^{\lceil \log_2 K\rceil + 1}\frac{\frac{j\cdot2^{b-2}}{K}}{\exp\left(\frac{j\cdot  2^{b-2}}{K}\right)}\\
        \leq & 1 + \frac{2K}{j} + \frac{2K}{j}\sum_{b=3}^{\lceil \log_2 K\rceil + 1}\frac{\frac{j\cdot2^{b-2}}{K}}{\exp\left(\frac{j\cdot  2^{b-2}}{K}\right)}.
    \end{align*}
    Then, we suffice to prove $\sum_{b=3}^{\lceil \log_2 K\rceil + 1}\frac{\frac{j\cdot2^{b-2}}{K}}{\exp\left(\frac{j\cdot  2^{b-2}}{K}\right)} \leq O\left((\log_2 \frac{K}{j}) +1\right)$.
    
    From the fact that $\frac{d \frac{x}{\exp(x)}}{dx} = e^{-x}-xe^{-x}=(1-x)e^{-x}$, we know $\frac{x}{\exp(x)}$ monotonically increases in the interval $(0, 1)$ and decreases in the interval $(1, +\infty)$. We have
    \begin{align*}
        & \sum_{b=3}^{\lceil \log_2 K\rceil + 1}\frac{\frac{j\cdot2^{b-2}}{K}}{\exp\left(\frac{j\cdot  2^{b-2}}{K}\right)}\\
        \leq & \sum_{b=3}^{\lfloor \log_2\frac{K}{j} \rfloor +3 } \frac{\frac{j\cdot2^{b-2}}{K}}{\exp\left(\frac{j\cdot  2^{b-2}}{K}\right)} + \sum_{b=\lceil \log_2\frac{K}{j} \rceil +3}^{\lceil \log_2 K\rceil + 1} \frac{\frac{j\cdot2^{b-2}}{K}}{\exp\left(\frac{j\cdot  2^{b-2}}{K}\right)}\\
        \leq & (\lfloor \log_2\frac{K}{j} \rfloor +1)\frac{x}{e^x}|_{x=1} + \sum_{b=\lceil \log_2\frac{K}{j} \rceil +3}^{\lceil \log_2 K\rceil + 1} \frac{\frac{j\cdot2^{b-2}}{K}}{\exp\left(\frac{j\cdot  2^{b-2}}{K}\right)}\\
        \leq & (\lfloor \log_2\frac{K}{j} \rfloor +1)\frac{1}{e} + \sum_{b=\lceil \log_2\frac{K}{j} \rceil +3}^{\lceil \log_2 K\rceil + 1} \frac{\frac{j\cdot2^{b-2}}{K}}{\exp\left(\frac{j\cdot  2^{b-2}}{K}\right)}.
    \end{align*}
    The second last line is by that function $\frac{x}{e^x}$ achieves its maximum at $x=1$. The remaining work is to prove $\sum_{b=\lceil \log_2\frac{K}{j} \rceil +3}^{\lceil \log_2 K\rceil + 1} \frac{\frac{j\cdot2^{b-2}}{K}}{\exp\left(\frac{j\cdot  2^{b-2}}{K}\right)}\leq O(1)$. 

    Notice that for $b\geq \lceil \log_2\frac{K}{j} \rceil +3$, we have $\frac{j2^{b-2}}{K} \geq \frac{j2^{ \log_2\frac{K}{j}  +1}}{K} \geq 2$. Since we can conclude $\frac{x}{e^x} \geq 2\cdot\frac{2x}{e^{2x}}$ holds for all $x\geq 2$, we can conclude
    \begin{align*}
        \frac{\frac{j\cdot2^{b-1}}{K}}{\exp\left(\frac{j\cdot  2^{b-1}}{K}\right)} / \frac{\frac{j\cdot2^{b-2}}{K}}{\exp\left(\frac{j\cdot  2^{b-2}}{K}\right)} \leq \frac{1}{2}
    \end{align*}
    holds for all $b\geq \lceil \log_2\frac{K}{j} \rceil +3$. Thus, we have
    \begin{align*}
        & \sum_{b=\lceil \log_2\frac{K}{j} \rceil +3}^{\lceil \log_2 K\rceil + 1} \frac{\frac{j\cdot2^{b-2}}{K}}{\exp\left(\frac{j\cdot  2^{b-2}}{K}\right)}\\
        \leq & \sum_{b=\lceil \log_2\frac{K}{j} \rceil +3}^{+\infty} 2^{-(b-\lceil \log_2\frac{K}{j} \rceil -3)} \frac{\frac{j\cdot2^{\lceil \log_2\frac{K}{j} \rceil +3-2}}{K}}{\exp\left(\frac{j\cdot  2^{\lceil \log_2\frac{K}{j} \rceil +3-2}}{K}\right)}\\
        \leq & \frac{1}{e} \sum_{b=0}^{+\infty} 2^{-b}\\
        \leq & O(1).
    \end{align*}
    That means we complete the proof.
\end{proof}

\begin{lemma}
    \label{lemma:sum-bracket-size-prob-2}
    For any $K\in \mathbb{N}$, $K\geq 2$, $j\in [K]$ we have
    \begin{align*}
        \sum_{b=1}^{\lceil \log_2 K\rceil + 1}\min\{\frac{1}{\exp\left(\frac{j\cdot \lfloor 2^{b-2}\rfloor}{K}\right)},\frac{\min\{2^{b-1}, K\}}{K}\} \leq O\Big(\frac{\log j}{j}\Big).
    \end{align*}
\end{lemma}
\begin{proof}
    We can conduct direct calculation.
    \begin{align*}
        & \sum_{b=1}^{\lceil \log_2 K\rceil + 1}\min\{\frac{1}{\exp\left(\frac{j\cdot \lfloor 2^{b-2}\rfloor}{K}\right)},\frac{\min\{2^{b-1}, K\}}{K}\}\\
        \leq & \frac{1}{K} + \sum_{b=2}^{\lceil \log_2 K\rceil + 1}\min\{\frac{1}{\exp\left(\frac{j\cdot  2^{b-2}}{K}\right)},\frac{2}{j}\frac{j\cdot2^{b-2}}{K}\}\\
        \leq & \frac{1}{K} + \sum_{b=2}^{\lceil \log_2 \frac{K\log j}{j}\rceil + 2}\frac{2}{j}\frac{j\cdot2^{b-2}}{K} + \sum_{b=\lceil \log_2 \frac{K\log j}{j}\rceil + 3}^{\lceil \log_2 K\rceil + 1}\frac{1}{\exp\left(\frac{j\cdot  2^{b-2}}{K}\right)}.
    \end{align*}
    We prove an upper bound for $\sum_{b=2}^{\lceil \log_2 \frac{K\log j}{j}\rceil + 2}\frac{2}{j}\frac{j\cdot2^{b-2}}{K}$ and $\sum_{b=\lceil \log_2 \frac{K\log j}{j}\rceil + 3}^{\lceil \log_2 K\rceil + 1}\frac{1}{\exp\left(\frac{j\cdot  2^{b-2}}{K}\right)}$ separately. 

    First, we have
    \begin{align*}
        & \sum_{b=2}^{\lceil \log_2 \frac{K\log j}{j}\rceil + 2}\frac{2}{j}\frac{j\cdot2^{b-2}}{K}\\
        \leq & \frac{2}{j}\cdot 2 \cdot \frac{j2^{b-2}}{K} |_{b=\lceil \log_2 \frac{K\log j}{j}\rceil + 2}\\
        \leq & \frac{2}{j}\cdot 2 \cdot \frac{j2^{1+\log_2 \frac{K\log j}{j}}}{K}\\
        = & \frac{8\log j}{j}.
    \end{align*}
    For $\sum_{b=\lceil \log_2 \frac{K\log j}{j}\rceil + 3}^{\lceil \log_2 K\rceil + 1}\frac{1}{\exp\left(\frac{j\cdot  2^{b-2}}{K}\right)}$, we first notice that when $b\geq \lceil \log_2 \frac{K\log j}{j}\rceil + 3$, we have
    \begin{align*}
        \frac{j 2^{b-2}}{K} \geq \frac{j 2^{\lceil \log_2 \frac{K\log j}{j}\rceil + 3-2}}{K} \geq \frac{j 2^{\log_2 \frac{K\log j}{j}+ 1}}{K} \geq 2.
    \end{align*}
    Further, we can conclude for $b\geq \lceil \log_2 \frac{K\log j}{j}\rceil + 3$, we have $\frac{j 2^{b-2}}{K} +1 \leq \frac{j 2^{b-2}}{K} \cdot 2 = \frac{j 2^{b-1}}{K}$. Thus, we have
    \begin{align*}
        & \sum_{b=\lceil \log_2 \frac{K\log j}{j}\rceil + 3}^{\lceil \log_2 K\rceil + 1}\frac{1}{\exp\left(\frac{j\cdot  2^{b-2}}{K}\right)}\\
        \leq & \sum_{b=\lceil \log_2 \frac{K\log j}{j}\rceil + 3}^{+\infty} \frac{1}{\exp(b-\lceil \log_2 \frac{K\log j}{j}\rceil - 3)}\frac{1}{\exp\left(\frac{j\cdot  2^{\lceil \log_2 \frac{K\log j}{j}\rceil + 3-2}}{K}\right)}\\
        \leq & \frac{2}{\exp\left(\frac{j\cdot  2^{\lceil \log_2 \frac{K\log j}{j}\rceil + 3-2}}{K}\right)}\\
        \leq & \frac{2}{\exp\left(\frac{j\cdot  2^{ \log_2 \frac{K\log j}{j} + 1}}{K}\right)}\\
        = & \frac{2}{\exp\left(\frac{j\cdot 2\cdot  \frac{K\log j}{j}}{K}\right)}\\
        = & \frac{2}{\exp\left(2\log j\right)}\\
        < & \frac{2}{j}.
    \end{align*}
    In summary, we have proved
    \begin{align*}
        & \sum_{b=1}^{\lceil \log_2 K\rceil + 1}\min\{\frac{1}{\exp\left(\frac{j\cdot \lfloor 2^{b-2}\rfloor}{K}\right)},\frac{\min\{2^{b-1}, K\}}{K}\}\\
        \leq & \frac{1}{K} + \frac{8\log j}{j} + \frac{2}{j}\\
        \leq & O\Big(\frac{\log j}{j}\Big).
    \end{align*}
\end{proof}

\begin{lemma}[Adapted Lemma 3 in \cite{jamieson2014lil}, Lemma D.3 in \cite{pmlr-v267-li25f}]
    \label{lemma:Adapted-lil-UCB}
    Denote $\{X_i\}_{i=1}^{+\infty}$ as i.i.d $\sigma^2$-subgaussian random variable with true mean reward $\mu=0$. For any $\delta \in (0, 1)$, we have
    \begin{align*}
        \Pr\left(\exists t, |\sum_{s=1}^t X_s| \geq \sqrt{2\sigma^2 2^{\lceil\log_2 t\rceil^+ }\log\frac{2(\log_2 2^{\lceil\log_2 t\rceil^+})^2}{\delta}}\right) < \frac{\pi^2}{6}\delta.
    \end{align*}
    Or equivalently, 
    \begin{align*}
        \Pr\left(\exists t, |\sum_{s=1}^t X_s| \geq \sqrt{2\sigma^2 2^{\lceil\log_2 t\rceil^+}\log\frac{2(\lceil\log_2 t\rceil^+)^2}{\delta}}\right) < \frac{\pi^2}{6}\delta.
    \end{align*}
\end{lemma}
\begin{proof}[Proof of Lemma \ref{lemma:Adapted-lil-UCB}]
    Define $u_k=2^k$, $k\geq 1$. Define $x=\sqrt{2\sigma^2 u_k\log\frac{2(\log_2 u_k)^2}{\delta}}$, $S_t=\sum_{i=1}^t X_i$ and the event 
    \begin{align*}
        E_k = \left\{\max_{1\leq t\leq u_k} S_t > \sqrt{2\sigma^2 u_k\log\frac{2(\log_2 u_k)^2}{\delta}}\right\} \cup \left\{\min_{1\leq t\leq u_k} S_t < -\sqrt{2\sigma^2 u_k\log\frac{2(\log_2 u_k)^2}{\delta}}\right\},
    \end{align*}
    
    For $\lambda> 0 $,  notice that
    \begin{align*}
        & \mathbb{E}\left[\exp(\lambda(\sum_{s=1}^{t}X_s))|X_1,\cdots,X_{t-1}\right]\\
        =& \exp(\lambda(\sum_{s=1}^{t-1}X_s))\mathbb{E}\exp(\lambda X_t)\\
        \geq & \exp(\lambda(\sum_{s=1}^{t-1}X_s))\exp(\mathbb{E}\lambda X_t)\\
        = & \exp(\lambda(\sum_{s=1}^{t-1}X_s)).
    \end{align*}
    Take $\lambda=\frac{x}{u_k\sigma^2}$, we can conclude $\{\exp(\lambda S_t)\}$ is a submartingale. Then,
    \begin{align*}
        & \Pr\left(\max_{1\leq t\leq u_k} S_t \geq x\right)\\
        = & \Pr\left(\max_{1\leq t\leq u_k} \exp(\lambda S_t) > \exp\left(\lambda x\right)\right)\\
        \stackrel{*}{\leq}& \frac{\mathbb{E}\exp(\lambda S_{u_k}) }{\exp\left(\lambda x\right)}\\
        \leq & \frac{\exp(\frac{u_k\lambda^2\sigma^2}{2}) }{\exp\left(\lambda x\right)}\\
        \stackrel{\lambda=\frac{x}{u_k\sigma^2}}{=} & \exp(-\frac{x^2}{2u_k \sigma^2}).
    \end{align*}
    Step * is by the maximal inequality for the submartingale. Take $x=\sqrt{2\sigma^2 u_k\log\frac{2(\log_2 u_k)^2}{\delta}}$, we have
    \begin{align*}
        \Pr\left(\max_{1\leq t\leq u_k} S_t \geq \sqrt{2\sigma^2 u_k\log\frac{2(\log_2 u_k)^2}{\delta}}\right)\leq \exp\left(-\log\frac{2(\log_2 u_k)^2}{\delta}\right)=\frac{\delta}{2(\log_2 u_k)^2} = \frac{\delta}{2k^2}
    \end{align*}
    For the part of $\Pr\left(\min_{1\leq t\leq u_k} S_t < -x\right)$, the proof is similar. We can conclude $\Pr(E_k) \leq \frac{\delta}{k^2}$ and further $\Pr(\cup_{k=1}^{+\infty }E_k) \leq \frac{\pi^2 \delta}{6}$.

    Thus, 
    \begin{align*}
        & \Pr\left(\exists t, |\sum_{s=1}^t X_s| \geq \sqrt{2\sigma^2 \max\{2^{\lceil\log_2t\rceil},2\}\log\frac{2(\log_2 \max\{2^{\lceil\log_2t\rceil},2\})^2}{\delta}}\right)\\
        \leq & \Pr\left(\exists k, \max_{1\leq t'\leq u_{k}}|\sum_{s=1}^{t'} X_s| \geq \sqrt{2\sigma^2 u_{k}\log\frac{2(\log_2 u_{k})^2}{\delta}}\right)\\
        \leq & \Pr(\cup_{k=1}^{+\infty }E_k) \leq \frac{\pi^2 \delta}{6}.
    \end{align*}
\end{proof}
Some Comments are as follows.
\begin{itemize}
    \item We can similarly prove $\Pr\left(\exists t, |\sum_{s=1}^t X_s| \geq \sqrt{2\sigma^2 2^{\lceil\log_2 t\rceil^+}\log\frac{2\pi^2(\log_2 2^{\lceil\log_2t\rceil^+})^2}{6\delta}}\right) < \delta$ holds for all $\delta\in(0, 1)$.
    \item Since $\lceil\log_2 t\rceil^+\leq 1 + \log_2 t$, we have
    \begin{align*}
        \frac{\pi^2 \delta}{6}\geq & \Pr\left(\exists t, |\sum_{s=1}^t X_s| \geq \sqrt{2\sigma^2 2^{\max\{\lceil\log_2t\rceil, 1 \}}\log\frac{2(\log_2 2^{\lceil\log_2t\rceil})^2}{\delta}}\right)\\
        \geq & \Pr\left(\exists t, |\sum_{s=1}^t X_s| \geq \sqrt{4\sigma^2 t\log\frac{2(\log_2 2t)^2}{\delta}}\right)
    \end{align*}
\end{itemize}

\begin{lemma}[Lemma D.2 in \cite{pmlr-v267-li25f}]
    \label{lemma:inequality-application-t-loglogt}
    For any $\Delta \in (0, 1], K\geq 2, \delta\in (0, \frac{1}{2}], C\geq 1$, we can conclude
    \begin{align*}
        & t > \frac{28C^2\log\frac{2K}{\delta}}{\Delta^2} + \frac{16 C^2\log\left(\log\left(\frac{24C^2}{\Delta^2}\right) \right)}{\Delta^2}\\
        \Rightarrow & C\sqrt{\frac{4\log\frac{2K (\log_2 2t)^2}{\delta}}{t}} < \Delta
    \end{align*}
\end{lemma}
\begin{proof}
    By simple calculation, we can derive
    \begin{align*}
        & t > \frac{28C^2\log\frac{2K}{\delta}}{\Delta^2} + \frac{16 C^2\log\left(\log\left(\frac{24C^2}{\Delta^2}\right) \right)}{\Delta^2}\\
        \Leftrightarrow & 2t > \frac{56C^2\log\frac{2K}{\delta}}{\Delta^2} + \frac{32 C^2\log\left(\log\left(\frac{24C^2}{\Delta^2}\right) \right)}{\Delta^2}\\
        \Leftrightarrow & 2t > \frac{24C^2\log\frac{2K}{\delta}}{\Delta^2} + \frac{32 C^2\log\left(\log\left(\frac{24C^2}{\Delta^2}\right) \right)+32C^2\log(\frac{2K}{\delta})}{\Delta^2}\\
        \stackrel{\log (x+y) \leq \log x + \log y, \forall x,y\geq 2}{\Rightarrow} & 2t > \frac{24C^2\log\frac{2K}{\delta}}{\Delta^2} + \frac{32 C^2\log\left(\log\left(\frac{24C^2}{\Delta^2}\right)+\frac{2K}{\delta} \right)}{\Delta^2}\\
        \Rightarrow & 2t > \frac{24C^2\log\frac{2K}{\delta}}{\Delta^2} + \frac{32 C^2\log\left(\log\left(\frac{24C^2}{\Delta^2}\right)+\log\left(\log\frac{2K}{\delta}\right) \right)}{\Delta^2}\\
        \Leftrightarrow & 2t > \frac{24C^2\log\frac{2K}{\delta}}{\Delta^2} + \frac{32 C^2\log\log\left(\frac{24C^2\log\frac{2K}{\delta}}{\Delta^2}\right)}{\Delta^2}\\
        \stackrel{\text{Lemma } \ref{lemma:inequality-t-loglogt}, \text{ as } 24C^2\log\frac{2K}{\delta}> e^2}{\Rightarrow} & 2t > \frac{24C^2\log\frac{2K}{\delta}}{\Delta^2} + \frac{16 C^2\log\log\left(2t\right)}{\Delta^2}\\
        \Rightarrow & 2t > \frac{8C^2\log\frac{2K}{\delta} + 16C^2\log \log_2 e + 16C^2\log\log (2t) }{\Delta^2}\\
        \Leftrightarrow & t > \frac{4C^2\log\frac{2K}{\delta} + 8C^2\log \frac{\log (2t)}{\log 2}}{\Delta^2}\\
        \Leftrightarrow & t > \frac{4C^2\log\frac{2K}{\delta} + 8C^2\log (\log_2 2t)}{\Delta^2}\\
        \Leftrightarrow & C\sqrt{\frac{4\log\frac{2K (\log_2 2t)^2}{\delta}}{t}} < \Delta.
    \end{align*}
\end{proof}

The last lemma is to solve inequality.
\begin{lemma}
    \label{lemma:inequality-t-loglogt}
    For any $b\geq a \geq 0$, 
    \begin{itemize}
        \item If $b\geq e^2$, we have $x \geq b+2a\log\log b\Rightarrow x\geq a\log\log(x)+b$.
        \item If $b,a\geq e$, we have $e\leq x\leq b+a\log\log b\Rightarrow x < a\log\log(x)+b$.
    \end{itemize}
\end{lemma}
The second inequality also implies $x \geq a\log\log(x)+b\Rightarrow x\geq b+a\log\log b$.
\begin{proof}[Proof of Lemma \ref{lemma:inequality-t-loglogt}]
    We prove the first claim. Easy to see $\frac{d(x-a\log\log x - b)}{dx} = 1-\frac{a}{x\log x} = \frac{x\log x-a}{x\log x}$. Take $x_0 = b+2a\log\log b$, then $x_0\log x_0 > x_0\log (e^2+2a\log\log e^2) > x_0\log e^2=2x_0 > a$. Thus $x- a\log\log(x)-b$ increases in the interval $(x_0, +\infty)$. On the other hand, easy to check
    \begin{align*}
        &x_0- a\log\log x_0-b\\
        =&b+2a\log\log b - a\log\log(b+2a\log\log b) - b\\
        =&2a\log\log b - a\log\log(b+2a\log\log b)\\
        =&a\left(2\log\log b - \log\log(b+2a\log\log b)\right)\\
        =&a\left(\log(\log b)^2 - \log\log(b+2a\log\log b)\right)\\
        =&a\log \frac{(\log b)^2}{\log(b+2a\log\log b)}\\
        \geq &a\log \frac{(\log b)^2}{\log(b+2b\log\log b)}\\
        = &a\log \frac{(\log b)^2}{\log b + \log(1+2\log\log b)},
    \end{align*}
    Notice that
    \begin{align*}
        &(\log b)^2 - \log b - \log(1+2\log\log b)\\
        =&\log b(\log b-1)-\log(1+2\log\log b)\\
        \stackrel{b\geq e^2}{\geq} & \log b-\log(1+2\log\log b)\\
        \geq & \log b - 2\log \log b\\
        \stackrel{x>2\log x,\forall x>0}{>} & 0,
    \end{align*}
    we can conclude $x_0- a\log\log x_0-b > 0$ holds for all $x\geq b+2a\log\log b$.

    Then we turn to prove the second claim. As $\frac{d(x-a\log\log x - b)}{dx} = 1-\frac{a}{x\log x} = \frac{x\log x-a}{x\log x}$, we know there is at most 1 zero point of $\frac{x\log x-a}{x\log x}$ in the interval $(e, b+a\log\log b)$. Thus,
    \begin{align*}
         \max_{e\leq x \leq b+a\log\log b}x-a\log\log x - b = \max\{x-a\log\log x - b|_{x=e}, x-a\log\log x - b|_{x=b+a\log\log b}\}.
    \end{align*}
    Easy to see
    \begin{align*}
         e-a\log\log e - b = e-b<0,
    \end{align*}
    and
    \begin{align*}
         &b+a\log\log b - a\log\log(b+a\log\log b) - b\\
         =&a\log\log b - a\log\log(b+a\log\log b)\\
         <&a\log\log b - a\log\log(b)\\
         = & 0.
    \end{align*}
    That means $\max\limits_{e\leq x \leq b+a\log\log b}x-a\log\log x - b < 0$. The second conclusion is done.
\end{proof}



\end{document}